\documentclass[11pt]{article}
\usepackage[letterpaper, left=1.2truein, right=1.2truein, top = 1.2truein,bottom = 1.2truein]{geometry}
\usepackage{filecontents}
\RequirePackage{amsmath, amsfonts, amssymb, amsthm, bm, graphicx, mathtools, enumerate,multirow}
\RequirePackage[numbers, sort]{natbib}
\RequirePackage{graphicx}%
\usepackage{url}
\usepackage[utf8]{inputenc} 
\usepackage[T1]{fontenc}    
\usepackage{hyperref}       
\usepackage{url}            
\usepackage{booktabs}       
\usepackage{amsfonts}       
\usepackage{nicefrac}       
\usepackage{microtype}      

\usepackage{amsthm,amsmath,amssymb,bbm,bm}
\usepackage{cancel}
\usepackage{natbib}
\usepackage{longtable}
\usepackage{multirow}
\usepackage{setspace}
\usepackage{centernot}
\usepackage{array}
\usepackage{mathrsfs}
\usepackage{dsfont}
\usepackage{relsize}
\usepackage{rotating}
\usepackage{enumitem}
\usepackage{float}
\floatstyle{plaintop}
\restylefloat{table}
\usepackage{subcaption}
\usepackage{multirow}
\usepackage{wrapfig}
\usepackage{graphicx}
\usepackage{comment}
\usepackage{subcaption}

\DeclareMathOperator*{\argmax}{arg\,max}
\DeclareMathOperator*{\argmin}{arg\,min}

\theoremstyle{definition}
\newtheorem{remark}{Remark}
\newtheorem{assumption}{Assumption}
\newtheorem{definition}{Definition}[section]

\newtheorem{theorem}{Theorem}[section]
\newtheorem{lemma}{Lemma}[section] 

\newtheorem{corollary}{Corollary}[section]

\usepackage{algorithm}
\usepackage{algpseudocode}
\usepackage{tikz}

\usepackage{graphicx}

\usepackage{caption}
\usepackage[labelfont=sf,justification=centering]{subcaption}
\usepackage{hyperref}       
\hypersetup{
 colorlinks=true,
 linkcolor=red,
 urlcolor=blue,
 citecolor=blue
}
\usepackage{microtype}      
\usepackage{xcolor}         
\usepackage{multirow}
\usepackage{parskip}
\usepackage{tcolorbox}

\newcommand{\KL}{\mathrm{KL}}
\newcommand{\ML}{\mathrm{ML}}
\newcommand{\pess}{\mathfrak{p}}

\newcommand{\cD}{\mathcal{D}}
\newcommand{\cR}{\mathcal{R}}
\newcommand{\cV}{\mathcal{V}}

\newcommand{\cE}{\mathcal{E}}
\newcommand{\cF}{\mathcal{F}}
\newcommand{\cN}{\mathcal{N}}

\newcommand{\cP}{\mathcal{P}}   
\newcommand{\cY}{\mathcal{Y}} 
\newcommand{\cA}{\mathcal{A}} 
\newcommand{\cB}{\mathcal{B}} 
\newcommand{\cL}{\mathcal{L}} 
\newcommand{\cI}{\mathcal{I}}

\newcommand{\bbE}{\mathbb{E}}
\newcommand{\bbR}{\mathbb{R}}
\newcommand{\bbS}{\mathbb{S}}
\newcommand{\bbP}{\mathbb{P}}

\newcommand{\bbN}{\mathbb{N}}

\newcommand{\rd}{\mathrm{d}}

\usepackage{bbm}

\begin{document}

\title{PASTA: A Unified Framework for Offline Assortment Learning\thanks{This is an extended version of 
\href{https://arxiv.org/abs/2302.03821}{arXiv:2302.03821}.}}
\author{Juncheng Dong$^{\S 1}$ \and Weibin Mo$^{\S 2}$ \and Zhengling Qi$^3$ \and Cong Shi$^4$ \and Ethan X. Fang$^{1}$ \and Vahid Tarokh$^{1}$}
\date{%
    $^\S$Equal Contribution\\%
    $^1$Duke University\\%
    $^2$Purdue University\\%
    $^3$George Washington University\\%
    $^4$University of Miami
}

\maketitle

\begin{abstract}
We study a broad class of assortment optimization problems in an offline and data-driven setting. In such problems, a firm lacks prior knowledge of the underlying choice model, and aims to determine an optimal assortment based on historical customer choice data. The combinatorial nature of assortment optimization often results in insufficient data coverage, posing a significant challenge in designing provably effective solutions. To address this, we introduce a novel Pessimistic Assortment Optimization (PASTA) framework that leverages the principle of pessimism to achieve optimal expected revenue under general choice models. Notably, PASTA requires only that the offline data distribution contains an optimal assortment, rather than providing the full coverage of all feasible assortments. Theoretically, we establish the first finite-sample regret bounds for offline assortment optimization across several widely used choice models, including the multinomial logit and nested logit models. Additionally, we derive a minimax regret lower bound, proving that PASTA is minimax optimal in terms of sample and model complexity. Numerical experiments further demonstrate that our method outperforms existing baseline approaches.
\end{abstract}

\section{Introduction}\label{sec:Intro}
In business operations, a key challenge is selecting a product assortment that maximizes the seller's expected revenue, a problem commonly known as assortment optimization. Solving this problem requires understanding customer choice behavior for a given assortment, often modeled using established frameworks such as the widely studied multinomial logit (MNL) model~\citep{mcfadden1973conditional}. 
This work tackles the offline assortment optimization problem by leveraging historical observational data. We propose a novel algorithm that efficiently identifies the optimal assortment under a broad class of choice models, without requiring prior knowledge of the model parameters.

Offline assortment optimization is of particular interest in the era of Big Data, as many companies have access to extensive customer data, and it is in their best interest to leverage this existing asset to inform better decisions. 
In contrast, online learning for assortment optimization can be costly, time-consuming, and sometimes even infeasible due to various constraints~\citep{ditchfield2021ethical}.  
However, existing works on data-driven assortment optimization mainly study the dynamic/online setting~\citep{bertsimas2015data}, where decision-makers adaptively learn the underlying choice model and optimize assortments in a trial-and-error fashion.
Online algorithms often build upon the principle of optimism, prioritizing exploration of assortments with the potential to generate large revenues~\citep{chen2021optimal_dynamic}. 
Such an exploration strategy ensures sufficient coverage of assortments, preventing the suboptimal ones from being mistakenly selected due to the estimation error of choice models. 
In comparison, the offline data distribution can hardly cover all potential assortments. For example, some assortments with obvious suboptimal performance are unlikely to be offered to customers and, therefore, are never observed in the offline data. 
Then, the revenues of these assortments are often poorly estimated, leading to likely suboptimal decisions.
One cannot address this challenge through active assortment exploration because it is infeasible in the offline setting. 
This challenge, unique to the offline setting, goes by the name of \emph{insufficient data coverage}. 
Due to this challenge, offline assortment optimization requires a different treatment compared with the optimism-based online algorithms, whose success hinges on sufficient assortment exploration.

To this end, we adopt the principle of pessimism, an approach that finds great success in offline learning problems~\citep{jin2021pessimism}. 
Building on this principle, we propose a novel \textit{Pessimistic ASsorTment leArning} (PASTA for short) framework for offline assortment optimization under \emph{general} choice models. To identify an optimal assortment under insufficient data coverage, PASTA 
first constructs an uncertainty set for the underlying choice model based on the likelihood ratio test~\citep{owen1990empirical}.  
It then evaluates each assortment based on its worst-case revenue across all models within the uncertainty set 
and selects the assortment that maximizes the worst-case revenue by solving a max-min problem.
To shed some light, we note that insufficient data coverage can potentially result in large estimation errors of the underlying choice model, where some suboptimal assortments' corresponding revenues can be over-estimated.
PASTA addresses this challenge by leveraging robust optimization to prevent the selection of suboptimal assortments.
In particular, when the uncertainty set contains the true choice model, the worst-case optimization objective serves as a valid lower bound for the true revenue function. 
If a suboptimal assortment is under-explored, the potential estimation error can cause significant underestimation of its revenue. 
Conversely, when an optimal assortment is sufficiently explored, the risk of underestimation is negligible.
As a result, PASTA guarantees to rule out under-explored suboptimal assortments and recover an optimal assortment,  without requiring sufficient exploration of all feasible assortments but only an optimal one. 
We note that assuming sufficient exploration of an optimal assortment is reasonable, as many historical assortments in the offline data are from sales experts with the goal of maximizing revenue. 

From a technical perspective, the success of PASTA depends on the precise control of the size of its uncertainty set, governed by a scalar parameter $\alpha > 0$. Note that we construct the uncertainty set from a likelihood ratio test, where $\alpha$ serves as the test's critical value, guiding the selection of choice models likely to include the true underlying one.
By selecting $\alpha$ appropriately, we show that PASTA is  effective across a broad class of choice models. In particular, we demonstrate that when $\alpha$ is sufficiently large relative to the choice model of interest (e.g., the MNL model), it provides an upper bound on PASTA's regret,  which we aim to minimize. 
We characterize $\alpha$ following empirical process theory for a general model class \citep{van2023weak}, encompassing the celebrated MNL~\citep{mcfadden1981econometric}, latent class logit~\citep{latent-class-logit}, and nested logit~\citep{nested-logit} models as special cases. Specifically, we derive an efficient choice of $\alpha$ that scales with $D/n$, where $n$ denotes the sample size of the offline data, and $D$ is the effective dimension of the choice model. Under this selection, PASTA achieves a regret upper bound of the order $\sqrt{\alpha} = \sqrt{D/n}$.  
We further show that this rate is minimax optimal, as it matches the lower bound under the MNL model. 

We highlight that our theoretical guarantees of PASTA hold even in the challenging setting of insufficient data coverage. Rather than assuming the offline data includes all feasible assortments, PASTA only requires that an optimal assortment appears in the data with positive probability. For a model-agnostic framework with general regret guarantees, this is a necessary condition to identify an optimal assortment. The positivity assumption is indeed realistic, as logged assortments often include expert-selected, high-revenue sets. We note that this requirement can be weakened under additional model structure, but that would narrow PASTA’s generality. 

\noindent{\bf Major Contributions.} Our contribution is threefold: (1) We introduce PASTA, a unified framework for offline assortment optimization under a broad class of choice models. To illustrate its generality, we apply PASTA to several widely used choice models, including the MNL, latent class logit, and nested logit (NL) models. For these models, we establish finite-sample regret bounds in the offline setting, which, to the best of our knowledge, are novel to the literature. 
(2) Under the MNL model, we derive the first minimax regret lower bound for assortment optimization in the offline setting, which justifies the minimax optimality of the proposed framework.
(3) We thoroughly validate our framework with simulations that allow exact computation of the optimal assortment and regret. 

\noindent{\bf Paper Organization.} We first briefly discuss related works in Section~\ref{sec:related_work}. In Section~\ref{sec:pre}, we formalize the offline assortment optimization problem and discuss its unique challenges. In Section~\ref{sec:pess}, we elaborate on our PASTA framework and establish its general regret bounds. In Section~\ref{sec:application}, we instantiate PASTA on various choice models and establish regret guarantees for PASTA under these particular models. After presenting simulation studies in Section~\ref{sec:exp}, we conclude the paper. 

\section{Related Work}\label{sec:related_work}
\noindent{\bf Assortment Optimization.} The assortment optimization problem is a fundamental decision-making challenge in business operations, owing to its broad applicability~\citep{assortment_review}. Consequently, extensive works aim to address this problem under various choice models, including the MNL model without constraints~\citep{AO_MNL, wang2012capacitated}, MNL model with unimodular constraints~\citep{AO_LP}, NL model~\citep{davis2014assortment}, ranking-based preference model~\citep{feldman2019assortment}, and consider-then-choose model~\citep{aouad2021assortment}. 
These works assume known choice models and focus on the combinatorial optimization challenges associated with assortment planning~\citep[e.g.,][]{ass_opt_2015,ass_opt_2019,ass_opt_2020,liu2020ass_opt,wang2021consumer}, whereas we consider the settings with unknown model parameters that require estimation based on the offline data. 

\noindent{\bf Data-driven Assortment Optimization.} On the other hand, in the current era of Big Data, we observe an increasing volume of works regarding data-driven assortment optimization~\citep{bertsimas2015data}. The existing literature mainly focuses on dynamic assortment optimization (DAO)~\citep{DAO} where
we repeatedly update the assortment based on the observed choices of customers, with the goal that the offered assortment converges to an optimal one. 
A wealth of research addresses this problem under a range of application scenarios. 
In particular,~\cite{rusmevichientong2010dynamic} study DAO under MNL model with capacity constraints;~\cite{chen2021optimal_dynamic} propose an optimal algorithm for DAO under MNL model without constrains, achieving cumulative regrets independent of the total number of items that matches the minimax regret lower bound in terms of the horizon;~\cite{DAO_nest_MNL} consider DAO under NL model; for more complex scenarios, ~\cite{kallus2020dynamic} and~\cite{chen2020dynamic} study DAO under MNL with high-dimensional context information and changing context information;~\cite{cao2024tiered} investigate the tiered DAO problems.
In comparison with the aforementioned works, our work is the first to focus on the offline setting where the seller selects an assortment relying on pre-collected historical data, and does not actively explore new assortments. 
While we use ``offline" to describe the source of data for data-driven decisions, some existing works in assortment optimization use ``offline" to refer to either physical locations, i.e., brick-and-mortar stores~\citep{dzyabura2018offline} or decision-making problems without data~\citep{wang2024optimizing}. These works are orthogonal to ours. 

\noindent{\bf Pessimism.} 
Our proposed PASTA framework builds upon the principle of pessimism, which has been successfully applied in offline reinforcement learning (RL) to find optimal policies using historical datasets. On the empirical side, it improves the performance of both the model-based and value-based approaches in the offline setting~\citep[e.g.,][]{model_based_RL_pess_2020,model_based_RL_pess_2_2020,value_based_RL_pess_2020}. 
The effectiveness of pessimism has also been analyzed and verified theoretically in the setting of RL~\citep{jin2021pessimism,fu2022offline}. 
Notably, while the principle of pessimism has proven effective for offline learning, its exact implementations for different problems require careful instance-specific designs. 
To this end, the main contribution of our work is to take a pessimistic approach to offline assortment optimization problems and demonstrate its empirical and theoretical values. Moreover, our work differs from the above works by focusing on a decision-making problem with actions of combinatorial structures. 

\noindent{\bf Distributionally Robust Optimization (DRO) and Robust Assortment Optimization (RAO).} 
As detailed in Section~\ref{sec:pess}, PASTA tackles the insufficient data coverage in offline learning by solving max–min programs over data-driven uncertainty sets. While this closely relates to work in DRO and RAO~\citep{ben2013robust,wang2016likelihood,bertsimas2018robust,duchi2021statistics,rusmevichientong2012robust,desir2024robust,sturt2025value}, our aim is fundamentally different: instead of robustness, PASTA seeks to identify the optimal assortment with regret guarantees using large-scale historical data, whereas, to the best of our knowledge, such optimality cannot be guaranteed by existing DRO and RAO approaches.
Although DRO approaches also adopt max-min formulations and require the uncertainty set to contain the true choice model, PASTA further requires the uncertainty set to satisfy a finite-sample concentration condition (see Assumption~\ref{ass: technical assumptions}).  
RAO likewise uses max–min programs, but its main focus is to efficiently solve specific robust optimization problems and find an assortment that achieves robustness. In comparison, our contribution is statistical: we use pessimism to address limited coverage and prove that the worst-case solution recovers the optimal assortment and achieves efficient finite-sample regret.


\section{Offline Assortment Optimization under Insufficient Data Coverage}\label{sec:pre}
In Section~\ref{sec:offline-ass-intro}, we introduce and formally define the offline assortment optimization problem. In Section~\ref{sec:offline-ass-challenges}, we explore the unique challenges rising from insufficient data coverage and illustrate them using a motivating example.  

\subsection{Offline Assortment Optimization}\label{sec:offline-ass-intro}
Let $[N]=\{1,2,..,N\}$ denote the set of $N$ distinct items. 
Denote the collection of assortments under consideration by $\mathbb{S} \subseteq 2^{[N]}\setminus\{\emptyset\}$. In the assortment optimization problem, when presented with an assortment $S \in \bbS$ offered by the seller, a customer makes a purchase $A \in S \cup \{0\}$ with $A = 0$ if no purchase is made. The seller subsequently receives a revenue $R$. 
The ultimate goal of assortment optimization is to find an optimal assortment $s^{\star} \in \bbS$ to maximize the expected revenue. Specifically, let $p_0(a|s)$ denote an unknown choice model which represents the probability of the customer purchasing $a \in s\cup\{0\}$ when $s$ is offered, with $p_0(a|s)=0$ for all $a \notin s$. The assortment optimization problem is to find an optimal assortment $s^\star$ that
\begin{align}
	s^\star \in \argmax_{s \in \bbS}\left\{ \cV_{0}(s) = \sum_{a\in s}p_{0}(a|s)r(s, a) \right\},
	\label{eq:clairvoynant}
\end{align}
where $r(s,a)=\bbE[R|S=s,A=a]$ represents the expected revenue when the customer is presented with the assortment $s \in \bbS$ and purchases the item $a \in s$.
Here, \(\cV_0(s)\) in~\eqref{eq:clairvoynant} represents the true expected revenue of assortment $s$. 
For optimization tractability, it is often assumed that the revenue $r(s,a) = r(a)$ only depends on the purchased item but not on the offered assortment~\citep{AO_MNL,AO_LP,ass_opt_2019}. We assume that the optimization problem in the form of \eqref{eq:clairvoynant} can be tractably solved. Our goal is to develop an offline learning method that further enjoys  performance guarantees under the challenge of insufficient data coverage. 
We also assume that $r(s,a)$ is a known function to the seller, which incorporates the special case that the revenue $R$ is a deterministic and known function of $(s,a)$. 


In the offline data-driven context, we choose assortments based on an offline dataset $\cD=\{S_i,A_i\}^n_{i=1}$, consisting of $n$ pairs of historically offered assortments and corresponding customer choices. We assume that they are independent and identically distributed (i.i.d.) samples of the random tuple $(S, A)$, following the probability mass function $(s,a) \mapsto \pi_S(s)p_0(a|s)$ where $\pi_S(s)$ represents the probability of observing $s$ in the offline data. 
When presented with the offline dataset, a common approach is to first estimate the true choice model $p_0$ with an estimator $\widehat{p}_n$, such as the maximum likelihood estimator (MLE), and then to solve~\eqref{eq:clairvoynant} by replacing $p_0$ with $\widehat{p}_n$. 
However, this \textit{as-if} optimization approach may fail due to the insufficient data coverage problem, which is the most critical challenge for offline learning~\citep{jin2021pessimism,shi2022pessimistic}.
Specifically, insufficient data coverage implies that certain feasible assortments are entirely absent from the offline dataset, meaning $\pi_S(s) = 0$ for some $s \in \mathbb{S}$.
In the next subsection, we present a simple example 
to demonstrate how insufficient data coverage poses a unique challenge for offline assortment optimization.

\subsection{An Example of Insufficient Data Coverage}\label{sec:offline-ass-challenges}
In the following example, we show that if an assortment never appears in the data due to insufficient data coverage, we cannot recover its expected revenue, and the optimal assortment may therefore remain ambiguous.

Consider a simple assortment optimization problem with two items, denoted as $1$ and $2$, and all possible assortments are $\{1\},\{2\},\{1,2\}$. 
Suppose that a customer makes the purchasing choice according to the following model, which is a special case of the latent class logit model. In particular, for $j \in \{1,2\}$, we assume
\begin{align}
    \begin{aligned}
    p_{0}(j|\{j\}) &= {1 \over 2}\left\{ {v_{j} \over 1 + v_{j}} + {\widetilde{v}_{j} \over 1 + \widetilde{v}_{j}} \right\}, \\ 
    p_{0}(0|\{j\}) &= {1 \over 2}\left\{ {1 \over 1 + v_{j}} + {1 \over 1 + \widetilde{v}_{j}} \right\}, \\
    p_{0}(j|\{1,2\}) &= {1 \over 2}\left\{ {v_{j} \over 1 + v_{1} + v_{2}} + {\widetilde{v}_{j} \over 1 + \widetilde{v}_{1} + \widetilde{v}_{2}} \right\}, \\
    p_{0}(0|\{1,2\}) &= {1 \over 2}\left\{ {1 \over 1 + v_{1} + v_{2}} + {1 \over 1 + \widetilde{v}_{1} + \widetilde{v}_{2}} \right\},
\end{aligned}
    \label{eq:eg_prob}
\end{align}
where the choice of $0$ means that the customer does not make a purchase,
and the item-specific parameters $v_{1},v_{2},\widetilde{v}_{1},\widetilde{v}_{2} > 0$ are unknown to the seller. Then given the item-specific revenues $r(1),r(2)>0$, 
the seller aims to choose from assortments $\{1\},\{2\},\{1,2\}$ to maximize the expected revenue by utilizing the offline data. 

Suppose that in the historical dataset, the seller can only observe the samples of singleton assortments $\{1\}$ and $\{2\}$, i.e., $\pi_{S}(\{1\}), \pi_{S}(\{2\}) > 0$, while their union assortment $\{1,2\}$ is never observed, i.e., $\pi_{S}(\{1,2\}) = 0$. 
To better understand the challenge, suppose that we can precisely recover the customer choice probability for the two singleton assortments $\{1\}$ and $\{2\}$ and have 
$p_{0}(1|\{1\}) = p_{0}(0|\{1\}) = p_{0}(2|\{2\}) = p_{0}(0|\{2\}) = 1/2$, i.e., any singleton assortment yields a 50\%  purchase rate.
These equalities imply that the unknown parameters $v_{1},v_{2},\widetilde{v}_{1},\widetilde{v}_{2}$ satisfy  
\[ p(j|\{j\}) = {1 \over 2}\left\{{v_{j} \over 1 + v_{j}} + {\widetilde{v}_{j} \over 1+\widetilde{v}_{j}}\right\} = {1 \over 2} \quad \Leftrightarrow \quad \widetilde{v}_{j} = {1 \over v_{j}}; \quad j = 1,2. \]
In other words, all the parameters in the parameter space $\{(v_1,v_2,\widetilde{v}_1,\widetilde{v}_2 ) = (v_{1},v_{2},v_{1}^{-1},v_{2}^{-1}): v_{1},v_{2} > 0 \}$ satisfy the data equally well. For example, both of the parameters $(1,1,1,1)$ and $(10,1/10,1/10,10)$ are feasible parameters to yield a purchase rate of 50\% for singleton assortments.
But these different parameters give different customer choice probabilities for $\{1,2\}$, leading to different expected revenues. As a consequence of this ambiguity, it is challenging for the seller to determine which assortment is optimal. 
 
For example, consider a scenario where item \( 1 \) has revenue \( r(1) = \frac{1}{5} \), and item \( 2 \) has revenue \( r(2) = 1 \). In this case, the revenues of the singleton assortments are given by \( \cV_{0}(\{1\}) = \frac{1}{10} \) and \( \cV_{0}(\{2\}) = \frac{1}{2} \). To determine the optimal assortment, the seller needs to compare assortment $\{1,2\}$'s expected revenue $\cV_{0}(\{1,2\})$ with \( \cV_{0}(\{1\}) \) and \( \cV_{0}(\{2\})\). However, the seller cannot determine the optimal assortment: if the true parameter is \( (1,1,1,1) \), then \(\cV_{0}(\{1,2\}) = \frac{2}{5} < \cV_{0}(\{2\})\), 
implying that \( \{2\} \) is the optimal assortment; conversely, if the true parameter is \( (10,\frac{1}{10},\frac{1}{10},10) \), then \(\cV_{0}(\{1,2\}) = \frac{303}{555} > \cV_{0}(\{2\})\),
making \( \{1,2\} \) the optimal assortment.

This motivating example illustrates that, when certain assortments are missing from the offline data due to insufficient coverage, their revenues cannot be recovered, and thus it is challenging to determine the optimal assortment.  
This challenge also makes it difficult to establish performance guarantees for the as-if optimization approach. For instance, in the example above, even with an infinite amount of offline data, the MLE is only guaranteed to converge to some point within the parameter set \( \{ (v_{1},v_{2},v_{1}^{-1},v_{2}^{-1}): v_{1},v_{2} > 0 \} \) \citep{redner1981note}, based on which the seller cannot guarantee finding an optimal assortment. It remains unclear how to perform assortment optimization in this case.

\section{Pessimistic Assortment Learning}\label{sec:pess}
Despite the aforementioned challenges of insufficient coverage, our insight is that finding an optimal assortment may not necessarily require $\pi_{S}(s) > 0$ everywhere but only at one optimal assortment~$s^{\star}$. In particular, when computing for an optimal assortment $s^\star$ following~\eqref{eq:clairvoynant},
it is unnecessary to estimate $p_0(a|s)$ well for all $s \neq s^{\star}$, as long as we can rule out those suboptimal assortments. 
Building on this insight, we propose a novel PASTA framework, where we only require the coverage at optimum, i.e., $\pi_S(s^{\star}) > 0$. This is a much weaker and more realistic assumption than that of uniform coverage required by the as-if approach, and much more likely to hold in practice.

\subsection{PASTA Framework}
Let $\cP$ be a set of conditional mass functions $p(a|s)$ for $A|S$, representing a family of models that includes the true choice model $p_0(a|s)$. For example, $\cP$ can be a family of MNL models with different parameters. 
Following the principle of pessimism, PASTA first constructs an uncertainty set for $p_0$, which serves as a set estimate of $p_{0}$. 
Specifically, given an offline dataset $\cD = \{S_i, A_i\}_{i=1}^{n}$ where $n$ is the sample size, we let the likelihood-based loss function $\widehat L_n(p)$ (smaller the better) be
\[ \widehat L_n(p) := -\frac{1}{n}\sum_{i=1}^{n}\log p(A_i|S_i); \quad p \in\cP.\]
Let $\widehat{p}_n \in \argmin_{p \in \cP}\widehat{L}_n(p)$ be the maximum likelihood estimator (MLE). 
For a set estimate of the true choice model $p_{0}$, we consider the set of $\alpha_{n}$-minimizers of the loss function $\widehat{L}_{n}(\cdot)$ defined as
\begin{equation}\label{eqn:uncertainty-set}
\Omega_{n}(\alpha_n) := \bigg\{ p \in \cP: \widehat{L}_{n}(p) - \widehat{L}_{n}(\widehat{p}_n) \leq \alpha_n\bigg\}, 
\end{equation}
where $\alpha_n >0$ is the critical value to control the size of the uncertainty set \(\Omega_{n}(\alpha_n)\). 
Note that $\Omega_{n}(\alpha_{n})$ is essentially a confidence region induced by the likelihood ratio test. 
Without identification assumptions,
\citet{manski2002inference} point out that the consistency of MLE $\widehat{p}_{n}$ is ambiguous, 
while the set estimate $\Omega_{n}(\alpha_{n})$ is consistent with a proper choice of $\alpha_{n}$ that $\bbP\{\Omega_{n}(p_{0}) \in \alpha_n\} \to 1$.
We provide the details of how to determine $\alpha_{n}$ in Section \ref{sec:theory} later.
In the sequel, we  drop $\alpha_n$ in~$\Omega_{n}(\alpha_n)$ when there is no confusion. 

After constructing  uncertainty set $\Omega_{n}$, PASTA  solves the robust optimization problem 
\begin{equation}\label{eqn: pessimistic algorithm}
    \widehat s_{\pess,n} \in \argmax_{s \in \bbS} \min_{p \in \Omega_n} \left\{\mathcal{V}(s;p):=\sum_{a\in s}p(a|s)r(s,a)\right\}.
\end{equation} 
Here, for any $p \in \cP$, we let $\mathcal{V}(s;p)$ be the value function of assortment $s$ under  choice model $p$. 
For a given assortment $s \in \bbS$, the inner layer of minimization computes the worst-case revenue among all possible choice models within the uncertainty set. In particular, when the uncertainty set contains the true model, i.e., $ p_{0}\in \Omega_{n}(\alpha_n)$ , 
the optimal value $\min_{p \in \Omega_n} \mathcal{V}(s;p)$ is a proper lower bound of the true value $\cV_0(s)$ that $\cV_0(s)=\cV(s;p_0) \ge \min_{p \in \Omega_n} \mathcal{V}(s;p)$. 
In this way, the max-min value attained by \eqref{eqn: pessimistic algorithm} is also a valid lower bound of the optimal revenue $\max_{s \in \bbS}\cV_{0}(s)$ attained by the true optimal assortment. To shed more intuition of PASTA, we compare it with the commonly used as-if approach that
\begin{equation}\label{eqn:as-if-opt}
    \widehat s_{\ML,n} \in \argmax_{s \in \bbS}\left\{\cV(s;\widehat p_n)\right\},
\end{equation}
where we replace $p_0$ with an estimate $\widehat p_n$ as if $\widehat p_n$ is the true model. 
Note that  a less observed or never observed assortment $s$ in the data often incurs some larger  estimation error for $\widehat{p}_{n}(\cdot|s)$. Consequently, this potentially leads to some overestimation of its true expected revenue \(\cV(s;p_0)\). That is, the estimated revenue $\cV(s;\widehat p_n)$ can be much larger than its true value $\cV(s;p_0)$. 
Such estimation errors increase the risk of the as-if approach mistakenly selecting a suboptimal assortment.
In contrast, PASTA quantifies the estimation error via the uncertainty set $\Omega_n$, and robustifies assortment optimization against the estimation error. 
Specifically, when the estimated revenue $\mathcal{V}(s; \widehat{p}_n)$ for an assortment $s$ is highly uncertain, the worst-case revenue $\min_{p \in \Omega_{n}} \mathcal{V}(s; p)$ substantially underestimates $\mathcal{V}_0(s)$. Consequently, the outer layer of \eqref{eqn: pessimistic algorithm} selects an alternative assortment with a relatively higher worst-case revenue.
In this way, PASTA prevents the suboptimal assortments from being mistakenly selected due to their highly uncertain estimated revenues. To be successful, PASTA only needs that the underestimation of the revenue at an optimal assortment is mild. This explains why the PASTA algorithm only requires some coverage at the optimum, a much weaker and thus more realistic assumption than the uniform coverage assumption. 

In what follows, we establish the theoretical guarantees for our proposed framework.

\subsection{General Regret Guarantees}\label{sec:general-regret}
We show that the proposed PASTA framework~\eqref{eqn: pessimistic algorithm} achieves a general regret guarantee under the weak assumption of coverage at optimum, where we only require that  $\pi_{S}(s^\star) > 0$ for an optimal assortment~$s^\star$ that solves~\eqref{eq:clairvoynant}.  
We defer the detailed proofs in this subsection to the Appendix\ref{sec:app-proofs}. 

\paragraph{Notations} For any vector $x$, let $x^\intercal$  and $\|x\|_2$ respectively denote the transpose and $\ell_2$-norm of $x$. For any finite set $A$, let $|A|$ denote the cardinal number of $A$. For any two sequences $\{\varpi(n)\}_{n \geq 1}$ and $\{\gamma(n)\}_{n \geq 1}$, we write $\varpi(n) \gtrsim  \gamma(n)$ (respectively $\varpi(N) \lesssim \gamma(n)$) whenever there exist constants $c_1 > 0$
(resp. $c_2 > 0$) such that $\varpi(n) \geq c_1 \gamma(n)$ (respectively $\varpi(n) \leq c_2 \gamma(n)$). Moreover, we write $\varpi(n) \eqsim  \gamma(n)$ whenever $\varpi(n) \gtrsim  \gamma(n)$ and $\varpi(n) \lesssim  \gamma(n)$. 

Our analysis involves the likelihood ratio-based uncertainty quantification, and we adopt the Hellinger distance~\citep{wong1995probability,bilodeau2023minimax} for such quantification. In particular, given an assortment $s$ and two choice models $p_1(a|s),p_2(a|s)$, let 
\[h^2(p_1,p_2,s)={1 \over 2}\sum_{a \in s\cup\{0\}}\left( \sqrt{p_1(a|s)} - \sqrt{p_2(a|s)} \right)^{2}\]
be their squared Hellinger distance. 
Meanwhile, note that the offline assortment $S$ is stochastic. We further consider the
\emph{generalized squared Hellinger distance} between two choice models $p_1$ and $p_2$ by taking the expectation over $S$ with respect to the data distribution $\pi_S$ that
\begin{equation}\label{eq:H}
    \begin{aligned}
          H^2(p_1,p_2) = \bbE_S[h^2(p_1,p_2,S)].
    \end{aligned}
\end{equation}

To evaluate offline assortment optimization algorithms, we adopt the following regret as the performance metric. In particular, for any estimated optimal assortment $\widehat{s}$, we let its regret be
\[ \cR( \widehat{s} ) = \cV_0(s^{\star}) - \cV_0(\widehat{s}),\] 
which represents the gap between the expected revenues of an optimal assortment and the estimated one $\widehat{s}$. Note that by definition~\eqref{eqn: pessimistic algorithm}, $ \cV(s;p_0) = \cV_0(s)$ for all $s \in \bbS$.

We then establish an upper bound  for $\cR(\widehat{s}_{\pess,n})$, where $\widehat{s}_{\pess,n}$ is the solution from PASTA in~\eqref{eqn: pessimistic algorithm}. 
We begin with the following lemma that sheds light on our proof techniques 
via pessimistic optimization.
\begin{lemma}\label{lemma:general_pasta_regret}
If $\pi_{S}(s^\star) > 0$ for an optimal assortment $s^\star$, 
and the uncertainty set \eqref{eqn:uncertainty-set} includes the true choice model, i.e., $p_{0} \in \Omega_{n}$, then
\[ \cR(\widehat{s}_{\pess,n}) \le \max_{p \in \Omega_n} \big\{\cV(s^{\star};p_0) -  \cV(s^{\star};p)\big\} \le \mathcal{C}\max_{p \in \Omega_n}\sqrt{H^2(p,p_0)}, \]
where $\mathcal{C} = r_{s^\star}\sqrt{1/\pi_S(s^\star)}$ is a constant independent of the sample size $n$, and $r_{s^{\star}} = \max_{j \in s^\star}r(s^\star,j)$ is the maximal revenue among the items in $s^\star$.
\end{lemma}
\begin{proof}{Proof of Lemma~\ref{lemma:general_pasta_regret}} To prove the first inequality, we have 
\[\begin{aligned}
    &\cV(s^{\star};p_0)-\cV(\widehat{s}_{\pess,n};p_0)\leq \cV(s^{\star};p_0) - \min_{p \in \Omega_n}\cV(\widehat{s}_{\pess,n};p)\\
&\qquad\leq \cV(s^{\star};p_0) - \min_{p \in \Omega_n}\cV(s^{\star};p)=\max_{p \in \Omega_n} \big\{\cV(s^{\star};p_0) -  \cV(s^{\star};p)\big\},
\end{aligned}\]
where the first inequality above holds since $p_0 \in \Omega_n$, and the second inequality holds because $\widehat{s}_{\pess,n}$ solves~\eqref{eqn: pessimistic algorithm}. 
In the Appendix~\ref{sec:proof-thm-regret}, by direct applications of Lemmas~\ref{lemma:l1_distance_ub} and~\ref{lemma:l1_H_ub}, we prove that for any $p \in \cP$,
\[
\left|\cV(s^{\star};p_0) -  \cV(s^{\star};p) \right| \le \mathcal{C}\sqrt{H^2(p,p_0)},
\]
which concludes the proof.
\end{proof}

We compare Lemma \ref{lemma:general_pasta_regret} with the typical theoretical justifications for the as-if approaches \citep{kitagawa2018should}.
Specifically, in the regret analysis for the as-if optimization, the regret is typically upper bounded by a uniform estimation error $\sup_{s \in \bbS}\big|\cV(s;\widehat{p}_n) - \cV(s;p_0)\big| $ for values across all $ s \in \bbS $~\citep{qian2011performance}. This is to control the error caused by the greedy policy, i.e., $\big|\cV(\widehat{s}_{\ML,n},\widehat{p}_n) - \cV(\widehat{s}_{\ML,n},p_0)\big|$, where \(\widehat{s}_{\ML,n}\) is the solution of the as-if optimization in~\eqref{eqn:as-if-opt}. Such an upper bound typically entails that we estimate the value function $\mathcal{V}(s;p)$ uniformly well for all $s \in \bbS$. It further explains why an estimate-then-optimize approach requires the strong assumption that $\pi_{S}(s) > 0$ for all $s \in \bbS$.
In contrast,  Lemma \ref{lemma:general_pasta_regret} suggests that we only need to control the value function's estimation error at $s^\star$, i.e., \(\max_{p \in \Omega_n} \big\{\cV(s^{\star};p_0)-  \cV(s^{\star};p)\big\}\), as long as the uncertainty set $\Omega_n$ contains the true model $p_0$. 
This explains why PASTA only requires the coverage $\pi_{S}(s^{\star}) > 0$ at the optimal assortment $s^{\star}$.  

To leverage  Lemma~\ref{lemma:general_pasta_regret} for establishing theoretical guarantees of PASTA, we further need
\textit{(i)} the validity of the uncertainty set $\Omega_{n}$  that  $p_0\in\Omega_{n}$ with high probability, and 
\textit{(ii)} an upper bound of the set estimation error $\max_{p \in \Omega_n}H(p,p_0)$, measured by the maximum generalized Hellinger distance in~\eqref{eq:H} between the uncertain models within $\Omega_{n}$ and the true model $p_0$.
We summarize the assumptions and conditions in Assumption~\ref{ass: technical assumptions}, which serve as sufficient conditions for establishing the finite-sample regret guarantees of PASTA.
Without loss of generality, we assume that the revenue function is non-negative, i.e., $r(a) \ge 0 $ for all $a \in [N]$, and $r(0) = 0$ that no purchase incurs zero revenue.
\begin{assumption}\label{ass: technical assumptions}%
\mbox{ }\\
(I) The true choice model is correctly specified by the family $\cP$, i.e., $p_0 \in \cP$. \\
(II) {\sc [Coverage at Optimum]} 
For the offline data generating process, there exists $s^\star$ as an optimal assortment solving~\eqref{eq:clairvoynant}, 
such that $\pi_S(s^\star)>0$.\\
(III) {\sc [Uncertainty Set Validity]} For any $0 < \delta < 1$, with probability at least $1-\delta$, there exists a constant $\alpha_n(\delta)$ depending on $\delta$ such that $p_0 \in \Omega_n(\alpha_n(\delta))$, and $\sup_{p \in \Omega_n(\alpha_n(\delta))}H^{2}(p, p_0) \lesssim {\alpha_n(\delta)}$.
\end{assumption}

In Assumption~\ref{ass: technical assumptions} (I), we focus on a correctly specified and fixed model family $\cP$, although our arguments can be extended to a family growing in $n$ that approximates $p_{0}$ via sieve methods \citep{wong1995probability}.
In Assumption~\ref{ass: technical assumptions} (II),
we highlight that PASTA only requires the coverage at optimum, a much weaker condition compared to the uniform coverage that $\inf_{s \in \bbS}\pi_{S}(s) > 0$. 
As PASTA is a unified framework for general choice models, and our goal is to establish a generic regret guarantee without relying on the specific choice model assumptions, we believe that this is the minimal requirement on the offline dataset to identify an optimal assortment.  

In Assumption~\ref{ass: technical assumptions} (III),
we present the uncertainty quantification validity required by PASTA. In Section~\ref{sec:application},  we show that  Assumption \ref{ass: technical assumptions} (III) holds under different  choice model families.
In particular, our likelihood-based uncertainty set $\Omega_{n}(\alpha_{n})$ depends on a critical value $\alpha_{n}$, which needs to be selected in an \emph{efficient} manner.
On one hand, $\alpha_{n}$ cannot be too small, such that $\Omega_{n}(\alpha_{n})$ contains $p_{0}$ with high probability. In fact, $\alpha_{n}$ needs to be at least as large as the difference $\widehat{L}_{n}(p_{0})-\widehat{L}_{n}(\widehat{p}_{n})$, which is associated with the convergence rate of the likelihood ratio test statistic.
On the other hand, $\alpha_{n}$ cannot be too large, such that the entire set $\Omega_{n}(\alpha_{n})$ remains close to the truth $p_{0}$. 
In Section~\ref{sec:theory}, based on the empirical process theory, we derive an efficient $\alpha_n$ for any choice model family $\cP$. 


Notably,  the distance $H$ defined in \eqref{eq:H} only depends on the distance between $p_{1}(\cdot|s)$ and $p_{2}(\cdot|s)$ at  the assortment $s$ with positive coverage $\pi_{S}(s) > 0$. As a result, the last part of
Assumption~\ref{ass: technical assumptions} (III) only needs that for any $p \in \Omega_{n}(\alpha_{n})$, $p(\cdot|s)$ is close to $p_{0}(\cdot|s)$ for $s$ with $\pi_{S}(s) > 0$.
This eliminates the need to identify or accurately estimate $p_0(\cdot | s)$ for those assortments $s$ with no coverage, thus effectively addressing the issue of insufficient data coverage.


We then establish the generic regret bound for PASTA under Assumption~\ref{ass: technical assumptions}.

\begin{theorem}\label{thm:regret}
Let $r_{s^\star}$ be defined as in Lemma~\ref{lemma:general_pasta_regret}.
Under Assumption~\ref{ass: technical assumptions}, for any $0 < \delta < 1$, with probability at least $1-\delta$, we have $\cV(s^{\star};p_0) -  \cV(s^{\star};p) \lesssim r_{s^{\star}}\sqrt{\alpha_n(\delta)/\pi_S(s^\star)}$ uniformly for all $p \in\Omega_{n}(\alpha_n(\delta))$.
Following Lemma \ref{lemma:general_pasta_regret}, it further implies that, with probability at least $1 - \delta$, the regret of PASTA with the uncertainty set $\Omega_{n}(\alpha_{n}(\delta))$ satisfies
\[\cR(\widehat{s}_{\pess,n}) \lesssim  r_{s^{\star}}\sqrt{\frac{\alpha_n(\delta)}{\pi_S(s^\star)}}.\]
\end{theorem}
We point out that the two constants in the regret bound, $\pi_S(s^\star)$ and $r_{s^\star}$, only depend on the optimal $s^\star$. In particular, the regret decreases as  $\pi_S(s^\star)$ increases, that is, when $s^\star$ is observed more frequently. 
In addition, we note that the choice of  critical value $\alpha_n$ is crucial for obtaining a valid and sharp regret bound. Specifically, the smallest $\alpha_{n}$ such that Assumption \ref{ass: technical assumptions} (III) is fulfilled leads to the tightest regret bound in Theorem \ref{thm:regret}. 
In the next subsection, we derive an efficient choice of the critical value $\alpha_n$, which minimizes the regret while satisfying Assumption \ref{ass: technical assumptions} (III).

\subsection{Validity of Uncertainty Set}\label{sec:theory}
We present an efficient approach to determining the critical value such that  Assumption \ref{ass: technical assumptions} (III) holds for general choice model family $\cP$. 
Building upon the empirical process theory~\citep{van2023weak},
we first characterize the complexity of $\cP$ by calculating its bracketing number.
Next, we derive a balance equation to determine an efficient choice of the critical value $\alpha_{n}$ for the uncertainty set, which is also the learning rate of conditional density estimation of $\cP$ via MLE \citep{bilodeau2023minimax}.
Finally, we show that Assumption \ref{ass: technical assumptions} (III)  holds in Theorem \ref{thm:main-simple}.
Notably, our justifications only rely on the complexity of $\cP$ instead of the particular form of the choice model, making our analysis applicable to a wide class of choice models. 

We first introduce the definition of bracketing number, which is the tool we use to characterize the complexity of $\cP$.
Note that $\cP$ is essentially a class of measurable functions from  domain $\bbS\times [N]$ to $[0,1]$. In particular, $p\in\cP$ is a mapping from $(s,a)\in\bbS\times [N]$ to $p(a|s)\in [0,1]$.
\begin{definition}[Bracketing Number]
    A class of measurable functions $\cB\subseteq\{\bbS\times[N] \to [0,1]\}$ is an $H$-bracket-cover of $\cP$ at scale $\varepsilon > 0$, if for any $p \in \cP$, there exists $b_{1}, b_{2}\in\cB$, such that 
    $H(b_{1},b_{2}) \le \varepsilon$,
    and for any $(s,a) \in \bbS\times [N]$,
    $b_{1}(a|s) \le p(a|s) \le b_{2}(a|s)$.
    The $H$-bracketing number at scale $\varepsilon > 0$ is  the cardinality of the smallest $\cB$ satisfying the above, and is denoted as $\cN_{[]}(\varepsilon,\cP,H)$.
\end{definition}

In the literature, $\log\cN_{[]}(\varepsilon,\cP,H)$ is often referred to as the \textit{bracketing entropy}~\citep{van2023weak}. Its dependency on $\varepsilon$ as $\varepsilon \to 0^{+}$ typically characterizes the minimax learning rate (lower bound) on the function class $\cP$ \citep{yang1999information,bilodeau2023minimax},
and also determines the convergence rate (upper bound) of MLE \citep{wong1995probability}.
One typical example is 
the parametric entropy $\log\cN_{[]}(\varepsilon,\cP,H) \lesssim D\log(1/\varepsilon)$, 
which is achievable by a parametric class $\cP$ with $D$-dimensional unrestricted parameters,
or by a Vapnik–Chervonenkis (VC) function class with VC-dimension $D$.
Another example is the nonparametric entropy $\log\cN_{[]}(\varepsilon,\cP,H) \lesssim \varepsilon^{-\eta}$, which is achievable by an infinite-dimensional class of smooth functions on a compact domain differentiable up to a certain order, where $\eta > 0$ is a constant related to the smoothness of functions. 

For the assortment optimization problem of our interest, the domain of $\cP$ is $\bbS\times [N]$, which is finite. This implies that we can parameterize $\cP$ as a finite-dimensional family with dimension up to $N\cdot2^{N}$.
Thus, we consider parametric entropy in all of our applications in the next section. Meanwhile, to obtain sharp bounds, we devote most of our discussion to characterize the effective dimension $D$ that is potentially much smaller than $N\cdot 2^{N}$.
We note that our theoretical results are extendable to an infinite-dimensional customer choice model family $\cP$.
Such results are of interest if we consider a nonparametric contextual assortment optimization problem, where the customer choice depends on the presented assortment as well as some context-specific covariates. We leave this extension as a future work. 


Using the bracketing entropy $\log\cN_{[]}(\varepsilon,\cP,H)$ to measure the complexity of $\cP$, we now elaborate on our approach for an efficient choice of $\alpha_n > 0$ for any model family $\cP$. Specifically, we determine the critical value $\alpha_{n}$ based on the following \emph{balance equation}:
\begin{align}\label{eqn:main-simple-entropy}
\int_{C_1 \alpha_{n}}^{C_{2}\sqrt{\alpha_{n}}}\sqrt{\log\cN_{[]}\left(u/C_{3},\cP,H\right)}\rd u \le C_{4}\sqrt{n}\alpha_{n},
\end{align}
where $C_{1},C_{2},C_{3},C_{4} \in (0,+\infty)$ are some universal constants. 
In Theorem~\ref{thm:main-simple} below, we prove that any $\alpha_n$ satisfying~\eqref{eqn:main-simple-entropy} leads to an uncertainty set $\Omega_n(\alpha_n)$ satisfying Assumption~\ref{ass: technical assumptions} (III) with high probability. Moreover, Lemma~\ref{lem:monotonicty} shows that the solutions to~\eqref{eqn:main-simple-entropy} guarantee a monotonic property that allows us to find the most efficient choice of $\alpha_n$.

\begin{lemma}\label{lem:monotonicty}
    Suppose that $\alpha_{n} > 0$ satisfies \eqref{eqn:main-simple-entropy}. Then every $\alpha \ge \alpha_{n}$ also satisfies \eqref{eqn:main-simple-entropy}.
\end{lemma}

By Lemma \ref{lem:monotonicty},
it is natural to identify a smallest $\alpha_{n} > 0$
satisfying  \eqref{eqn:main-simple-entropy}, such that the same relationship holds for every $\alpha \ge \alpha_{n}$. This approach results in the smallest critical value for the uncertainty set and thus leads to the tightest regret upper bound for PASTA, as implied by Theorem~\ref{thm:regret}.
To present a concrete example, if we have a parametric entropy $\log\cN_{[]}(\varepsilon,\cP,H) \lesssim D\log(1/\varepsilon)$, then a critical value solving~\eqref{eqn:main-simple-entropy} is $\alpha_{n} \eqsim (D/n)\log(n/D)$, which is also the typical parametric estimation rate.

Note that \eqref{eqn:main-simple-entropy} also characterizes the convergence rate of the MLE for both unconditional \citep{wong1995probability} and conditional \citep{bilodeau2023minimax} density estimation. In particular, \citet[Theorem 5]{bilodeau2023minimax} established that 
\begin{equation}\label{eqn:mle-con}
    H^{2}(\widehat{p}_{n},p_{0}) \lesssim \alpha_{n}
\end{equation}
for MLE $\widehat{p}_{n}$.
However, this is not yet sufficient for any performance guarantee for the as-if optimization, since $\widehat{p}_{n}(\cdot|s)$ and $p_{0}(\cdot|s)$ can exhibit arbitrary relationships at $s$ with no coverage, i.e., when $\pi_{S}(s) = 0$. 
Consequently, although the MLE $\widehat{p}_{n}$ converges to $p_{0}$ in the metric $H$, the challenge of insufficient data coverage, as demonstrated by the motivating example in Section~\ref{sec:offline-ass-challenges}, remains unsolved for the as-if optimization approach.

In comparison, we establish theoretical guarantees of PASTA under the challenge of insufficient data coverage, only requiring the conditions in Assumption \ref{ass: technical assumptions}.
It is important to note that the existing results for the convergence rates of MLE \citep{bilodeau2023minimax} are not sufficient for the statements required by Assumption \ref{ass: technical assumptions} (III). 
Specifically, we note that by the definition of \(\Omega_{n}(\alpha_{n})\), MLE~$\widehat{p}_{n}$ belongs to $\Omega_{n}(\alpha_{n})$.
Then the requirement \(\max_{p \in \Omega_n(\alpha_n)}H^2(p,p_0) \lesssim \alpha_n \) in Assumption~\ref{ass: technical assumptions} (III)
is a stronger condition than~\eqref{eqn:mle-con}. 
Furthermore, the condition $p_{0} \in \Omega_{n}(\alpha_{n})$ in Assumption~\ref{ass: technical assumptions} (III) is equivalent to \(\widehat{L}_n(p_0) - \widehat{L}_n(\widehat{p}_n) \lesssim \alpha_n\), which does not follow from \eqref{eqn:mle-con} immediately.
To this end, we establish these conditions below 
under \eqref{eqn:main-simple-entropy}.
For simplicity, we present below a simplified version of the theorem, while the full statement can be found in Theorem~\ref{thm:main} in Appendix~\ref{sec:app-main}.

\begin{theorem}\label{thm:main-simple}
Consider the set estimate $\Omega_{n}(\alpha)$ in \eqref{eqn:uncertainty-set}.
Suppose that the critical value $\alpha_{n} > 0$ satisfies \eqref{eqn:main-simple-entropy}.
Then for every $\alpha \ge \alpha_{n}$, with probability at least $1 - C_{5}e^{-C_{6}n\alpha}$, we have
\[ p_{0} \in \Omega_{n}(\alpha), \quad  \text{and} \quad \sup_{p\in\Omega_{n}(\alpha)}H^{2}(p,p_{0}) \le C_{7}\alpha. \]
Here, $\{C_{j}\}_{j=1}^{7}$ are universal constants characterized in Theorem \ref{thm:main}.
\end{theorem}


We remark that our critical value determination is tight in terms of satisfying Assumption~\ref{ass: technical assumptions}~(III) for various families $\mathcal{P}$. In particular, \citet[Theorem 2]{bilodeau2023minimax} has established that for $\cP$ with a parametric entropy $\log\cN_{[]}(\epsilon,\cP,H) \eqsim D\log(1/\epsilon)$, any estimate $\widetilde{p}_{n}$ of $p_0$ satisfies $H^{2}(\widetilde{p}_{n},p_{0}) \gtrsim D/(n\log n)$ in the worst-case problem instance. This further implies that for any set estimate $\widetilde{\Omega}_{n}$, we have $\sup_{p\in\widetilde{\Omega}_{n}}H^{2}(p,p_{0}) \gtrsim D/(n\log n)$ in the worst-case. Here, the minimax rate $D/(n\log n)$ matches (up to logarithmic factors) with the upper bound $\sup_{p\in\Omega_{n}(\alpha)}H^{2}(p,p_{0}) \lesssim (D/n)\log(n/D)$ obtained in Theorem~\ref{thm:main-simple} from \eqref{eqn:main-simple-entropy}, demonstrating that our critical value determination through~\eqref{eqn:main-simple-entropy} is tight.
A similar minimax rate optimality also holds for a nonparametric entropy $\log\cN_{[]}(\epsilon,\cP,H) \eqsim \epsilon^{-\eta}$ with $0 < \eta \le 2$, corresponding to a sufficiently smooth family $\cP$.

Based on Theorem~\ref{thm:main-simple}, we conclude the regret guarantee for PASTA based on the critical value $\alpha_{n}$ from \eqref{eqn:main-simple-entropy} as in Corollary~\ref{corol:pasta-regret} below. 
To obtain a high-probability statement as in Assumption~\ref{ass: technical assumptions} (III), for every fixed $\delta > 0$, we further need $\alpha \ge {1 \over C_{6}n}\log{C_{5} \over \delta}$ in Theorem \ref{thm:main-simple}. 
Therefore, we inflate the critical value from \eqref{eqn:main-simple-entropy} by such an amount, which is mild since $\alpha_{n}$ satisfying \eqref{eqn:main-simple-entropy} is at least in the order of $n^{-1}$ even in the parametric entropy case.

\begin{corollary}\label{corol:pasta-regret}
Suppose that the assumptions of Theorem \ref{thm:main-simple} hold,
and let $\alpha_{n}^{\circ} > 0$ satisfy~\eqref{eqn:main-simple-entropy}.  
Fix $0 < \delta < 1$,
and let $\alpha_{n}(\delta) = \alpha_{n}^{\circ} + {1 \over C_{6}n}\log{C_{5} \over \delta}$.
Under Assumption \ref{ass: technical assumptions} (I) $p_0 \in \cP$ and (II) $\pi_S(s^\star) > 0$, with probability at least $1-\delta$, the regret of PASTA with the uncertainty set $\Omega_{n}(\alpha_{n}(\delta))$ satisfies
$$
\cR(\widehat{s}_{\pess,n}) \lesssim r_{s^{\star}}\sqrt{\frac{\alpha_{n}(\delta)}{\pi_S(s^\star)}}.
$$
\end{corollary}

Corollary \ref{corol:pasta-regret} concludes the performance guarantee of PASTA with three required ingredients: the correct specification of the choice model family $\cP$, the coverage at optimum, as well as the critical value determination via \eqref{eqn:main-simple-entropy}. 
For the applications in specific choice models, it remains to determine the critical value $\alpha_{n}$ in \eqref{eqn:main-simple-entropy}, which involves the proper characterization of the complexity of $\cP$ via the bracketing entropy $\log\cN_{[]}\left(\varepsilon,\cP,H\right)$.
To this end, we instantiate on several of the most widely used choice model families $\cP$, and establish the corresponding regret guarantees in the next section. 


\section{Applications}\label{sec:application}

We apply Corollary~\ref{corol:pasta-regret} to several widely used choice models and establish the regret of PASTA under these models. We consider three of the most commonly used choice models -- the MNL~\citep{mcfadden1981econometric}, latent logit~\citep{latent-class-logit}, and nested logit~\citep{nested-logit} models. 
Throughout this section, we  assume that Assumption~\ref{ass: technical assumptions}~(I) $p_0 \in \cP$ and (II) $\pi_S(s^\star) > 0$ hold. 
For each of the choice model family $\cP$ to consider, we first characterize its complexity via the bracketing entropy $\log\cN_{[]}\left(\varepsilon,\cP,H\right)$,
then we determine the critical value from \eqref{eqn:main-simple-entropy},
and finally apply Corollary~\ref{corol:pasta-regret} to establish PASTA's regret guarantee.
All the proofs are provided in the Appendix~\ref{sec:app-application-proofs}. 

For a given family $\cP$, we let $\theta \in \Theta$ denote the underlying parameter, and $\Theta$ is a pre-specified parameter space. 
The form of a customer choice model $p(a|s;\theta)$ is known up to the unknown parameter $\theta$.
Hence, the family of choice models is \(\cP = \{p(a|s;\theta): \theta \in \Theta\}\).
For each model discussed in this section, there exist $\cP$-dependent constants $D_{\cP},C_{\cP} < +\infty$,
where
$D_{\cP}$ is the effective dimension of the family $\cP$,
and $C_{\cP}$ is based on the boundedness of the problem. 
Let $n$ be the sample size of the offline dataset. 
For fixed $0 < \delta < 1$,
with probability at least $1-\delta$, the regret of PASTA with the uncertainty set $\Omega_{n}(\alpha_{n}(\delta))$ takes the form
\begin{equation}\label{eqn:formula-param-regret}
\begin{aligned}
    &\cR(\widehat{s}_{\pess,n}) \lesssim  r_{s^{\star}}\sqrt{{\alpha_{n}(\delta) \over \pi_{S}(s^{\star})}};\\ 
    &\alpha_{n}(\delta) \eqsim \frac{D_{\cP}}{n}\log\frac{C_{\cP}n}{D_{\cP}}+\frac{1}{n}\log\frac{1}{\delta}.
\end{aligned}
\end{equation}
Next, we investigate specific choice models.

\subsection{Multinomial Logit (MNL) Model}
The MNL model is arguably the most widely used discrete choice model in assortment optimization literature \citep{feng2022consumer}. Given the item-specific features $ \{ x_{j} \}_{j \in [N]} $ where $x_j \in \bbR^d$, an MNL model with parameter $\theta \in \Theta \subset \bbR^{d}$ assumes that customer's preference for the $j$-th item is proportional to $\exp(x_{j}^\intercal\theta)$. 
Specifically, the customer choice probability under MNL with parameter $\theta$ is 
\begin{align}\label{eqn:mnl_ass}
    p(a|s;\theta)= 
    \begin{cases}
    \frac{\exp(x_a^\intercal\theta)}{1+\sum_{j \in s}\exp(x_j^\intercal\theta)} & a \in s;\\
    \frac{1}{1+\sum_{j \in s}\exp(x_j^\intercal\theta)} & a = 0.
    \end{cases}
\end{align}
We impose the following mild boundedness assumptions. 
In particular, we assume that the item-specific feature vectors $\{x_{j}\}_{j\in[N]} \subseteq \bbR^{d}$ are deterministic and uniformly bounded, and the parameter space $\Theta$ is also bounded.
\begin{assumption}\label{asm:bdd}
In the MNL model \eqref{eqn:mnl_ass}, there exist constants $C_{x}, C_{\Theta} < +\infty$, such that
$\max_{j\in[N]}\|x_{j}\|_{2} \le C_{x}$ and $\sup_{\theta\in\Theta}\|\theta\|_{2} \le C_{\Theta}$.
\end{assumption}
We are ready to establish the regret of PASTA under the MNL model. 
\begin{theorem}\label{thm:mnl-regret}
Consider the MNL model \eqref{eqn:mnl_ass} satisifying Assumption \ref{asm:bdd}. Let $D_{\cP}=d$ and $C_{\cP}=C_x$. 
Under Assumptions~\ref{ass: technical assumptions} (I) and (II), for $n \gtrsim D_{\cP}\log(n/D_{\cP})$ and every $\delta \in (0,1)$, with probability at least $1 - \delta$, the regret of PASTA with the uncertainty set $\Omega_{n}(\alpha_{n}(\delta))$
satisfies \eqref{eqn:formula-param-regret}. 
\end{theorem}

\subsection{Latent Class Logit (LCL) Model}\label{sec:LCL}
The LCL model is closely related to the MNL model, as it assumes $K$ distinct MNL models \eqref{eqn:mnl_ass} as components, and the LCL model is the mixture of them~\citep{latent-class-logit}. Specifically, let $\Delta(K) = \{(\lambda_{1},\cdots,\lambda_{K}):\lambda_{1},\cdots,\lambda_{K} \ge 0, \lambda_{1}+\cdots+\lambda_{K}=1\}$ represent the probability simplex in $\bbR^K$. 
An LCL model is parametrized by $\theta = (\theta_{1},\cdots,\theta_{K};\lambda_{1},\cdots,\lambda_{K})$, where $\theta_k \in \widetilde{\Theta} \subset \bbR^d$ is the parameter in the $k$-th MNL model for $k \in [K]$, and $(\lambda_1,\dots,\lambda_K)\in \Delta(K)$ consists of the mixture probabilities. 
The parameter space can be written as $\Theta = \widetilde{\Theta}^{K} \times \Delta(K) \subset \bbR^{Kd+K}$. With item features $x_j \in \bbR^d$ for $j \in [N]$, the  custsomer choice probability under LCL with parameter $\theta$ is
\begin{equation}\label{eqn:lcl_ass}
 p(a|s;\theta) = 
 \begin{cases}
     \sum_{k=1}^K \frac{\lambda_k\exp(x_a^\intercal\theta_k)}{1+\sum_{j \in s}\exp(x_j^\intercal\theta_k)} &a \in s;\\
     \sum_{k=1}^K \frac{\lambda_k}{1+\sum_{j \in s}\exp(x_j^\intercal\theta_k)} &a=0.
 \end{cases}
\end{equation}

\begin{theorem}\label{thm:lcl-regret}
Consider the LCL model~\eqref{eqn:lcl_ass} with the component MNL models satisfying Assumption~\ref{asm:bdd}. Let $D_{\cP}=Kd+K-1$ and $C_{\cP}=KN\exp(C_\Theta C_x)$. Under Assumptions~\ref{ass: technical assumptions} (I) and (II), for $n \gtrsim D_{\cP}\log(n/D_{\cP})$ and every $\delta \in (0,1)$, 
with probability at least $1 - \delta$, the regret of PASTA with the uncertainty set $\Omega_{n}(\alpha_{n}(\delta))$ satisfies~\eqref{eqn:formula-param-regret}. 
\end{theorem}

\subsection{Nested Logit (NL) Model}
Despite its widespread use, the MNL model assumes the Independence of Irrelevant Alternatives (IIA)~\citep{iia}, which can limit its applicability in more complex assortment problems. To this end, the NL model~\citep{nested-logit} assumes a partition of all available items into $K$ exclusive baskets $\{B_k\}_{k=1}^K$ such that $B_k \subset [N]$, $B_i \bigcap B_j = \emptyset$ for $i\neq j$, and $\bigcup_{k \in [K]} B_k = [N]$~\citep{gallego2014constrained}. The NL model is parametrized by $\theta = (\widetilde{\theta};\lambda_1,\dots,\lambda_K)$ where $\widetilde{\theta} \in \widetilde{\Theta} \subset \bbR^d$ and $0 < \lambda_1,\cdots,\lambda_{k} \leq 1$. Thus, the parameter space is $\Theta = \widetilde{\Theta}\times(0,1]^K \subset \bbR^{d+K}$.
The seller offers an assortment $s = \{s_k\}_{k=1}^K$ consisting of the sub-assortments $s_k \subset B_k$ selected from the $K$ baskets. For $k \in [K]$, let $v_k = \sum_{i \in s_k}\exp(x_{i}^\intercal \widetilde{\theta}/\lambda_k)$ be the basket-level attraction parameter. 
With item features $x_j \in \bbR^d$ for $j \in [N]$, the  custsomer choice probability under LCL with parameter $\theta$ is
\begin{equation}\label{eqn:nl_ass}
\begin{aligned}
    &p(a|s;\theta) = \begin{cases}
        \frac{v_j^{\lambda_j-1}}{1+\sum_{k=1}^K v_k^{\lambda_k}}\exp\left({x_{a}^{\intercal}\widetilde{\theta}\over\lambda_j}\right) & a \in s_{j}, ~ j = 1,\cdots,K; \\
        \frac{1}{1+\sum_{k=1}^K v_k^{\lambda_k}} & a = 0.
    \end{cases}
\end{aligned}
\end{equation}
In order to establish the regret guarantee, the required boundedness assumption becomes the following. 
\begin{assumption}\label{as:nl}
In the NL model \eqref{eqn:nl_ass},
there exist constants $C_{x}, C_{\Theta},C_{\lambda} < +\infty$, such that
$\max_{j\in[N]}\|x_{j}\|_{2} \le C_{x}$, $\sup_{\theta\in\widetilde{\Theta}}\|\widetilde{\theta}\|_{2} \le C_{\Theta}$, and
$1/\lambda_k \le C_\lambda$ for $k \in [K]$.
\end{assumption}
Note that $\lambda_{1},\cdots,\lambda_{K}$ are used to parametrize the within-basket \emph{dissimilarity} among items~\citep{nested-logit}. The assumption on their boundedness away from zero implies that the within-baskets items should be strictly distinguishable from each other. 
\begin{theorem}\label{thm:nl-regret}
Consider the NL model~\eqref{eqn:nl_ass} satisfying Assumption \ref{as:nl}. Let $D_{\cP}=d+K$ and $C_{\cP}=C_xC_\Theta C_\lambda$. Under Assumptions~\ref{ass: technical assumptions} (I) and (II), for $n \gtrsim D_{\cP}\log(n/D_{\cP})$ and every $\delta \in (0,1)$, 
with probability at least $1-\delta$, the regret of PASTA with the uncertainty set $\Omega_{n}(\alpha_{n}(\delta))$ satisfies~\eqref{eqn:formula-param-regret}.
\end{theorem}

\subsection{Minimax Optimality}
To understand the tightness of the aforementioned regret guarantees, we establish the first minimax regret lower bound for the offline uncapacitated assortment optimization problem. 
Specifically, we consider the scenarios where the choice model is an MNL model \eqref{eqn:mnl_ass} with the true parameter $\theta_0 \in \bbR^d$ and item features $\{x_{j}\}_{j\in[N]} \subset \bbR^d$. 
Here, $N$ is the total number of items.  

\begin{theorem}\label{thm:minimax_d}
Assume the offline data sample size $n \geq \min(d,N)$. There exists a universal constant~$c$ such that, for any measurable function  $\widehat s$ of the offline dataset $\cD_{n}$, there exists a worst-case offline assortment optimization problem instance with bounded item features $\{x_j\}_{j \in [N]}$, an MNL parameter $\theta_0$, an item-only reward function $r:[N]\to [0,+\infty)$, and an offline assortment data distribution $\pi_S$ with $\pi_S(s^\star) > 0$, such that the expected regret of $\widehat s$ is lower bounded by $c\cdot\sqrt{{\min(d,N)}/{n}}$. 
\end{theorem} 
Notably, the minimax regret in Theorem~\ref{thm:minimax_d} matches the regret guarantee of PASTA in Theorem~\ref{thm:mnl-regret} in terms of the sample size $n$ up to logarithmic factors, demonstrating that PASTA is provably efficient in terms of sample complexity for offline assortment optimization. Moreover, in the common scenarios where $d \le N$, PASTA is also provably efficient in terms of the choice model complexity.

\section{Empirical Studies}\label{sec:exp}
In this section, we conduct experiments on simulation datasets to empirically demonstrate the effectiveness of PASTA. 
We consider offline assortment optimization under two of the most popular choice models:  the MNL and NL models. 
To facilitate empirical studies, in Section~\ref{sec:algo}, we first propose an iterative algorithm for heuristically solving the max-min problem given by PASTA in~\eqref{eqn: pessimistic algorithm}. We next discuss the data generating process and baselines in Section~\ref{sec:data-generation}, and subsequently present all the empirical results in Section~\ref{sec:perf-comparison}. 

\subsection{PASTA Algorithm}\label{sec:algo}
Our algorithm PASTA is applicable for scenarios where the choice model is a parametric model $p(a|s;\theta)$ with parameters $\theta \in \Theta \subset \bbR^d$ and $p(a|s;\theta)$ is differentiable with respect to $\theta$ for all $s$ and $a \in s\cup\{0\}$. 
These scenarios include many practical settings, such as the MNL and NL models. 
Notably, even when $\cP$ is nonparametric, it is standard to approximate $\cP$ with a parametric family so that the optimization is tractable. Hence, an algorithm for parametric models is sufficient and without loss of generality. Moreover, we follow the convention of assortment optimization literature to assume that the revenue of a selected item remains the same across distinct assortments, i.e., there exists $r_j$ for $j \in [N]$ such that for any $s \in \bbS$, $r(s,j) = r_j$ if $j \in s$ and $r(s,j)=0$ if $j \notin s$.

Specifically, let $\widehat{L}_n(\theta) = -(1/n)\sum^n_{i=1}\log p(a_i|s_i;\theta)$ and $\widehat{\theta}_n \in \argmin_{\theta \in \Theta}\widehat{L}_n(\theta)$ be the MLE. 
We aim to solve the following optimization problem
\begin{equation*}
    \begin{aligned}
        &\widehat s_{\pess,n} \in \argmax_{s \in \bbS} \min_{\theta \in \Omega_{n}(\alpha_n)} \mathcal{V}(s;\theta),
    \end{aligned}
\end{equation*}
where $\Omega_{n}(\alpha_n) := \bigg\{ \theta \in \Theta: \widehat{L}_{n}(\theta) - \widehat{L}_{n}(\widehat{\theta}_n) \leq \alpha_n\bigg\}$ and $\mathcal{V}(s;\theta) = \cV(s;p(\cdot|s;\theta))=\sum_{j\in s}r_jp(j|s;\theta)$.
The proposed algorithm PASTA is executed for a maximum of $T$ iterations. For initialization, set $\theta^{(0)} = \widehat{\theta}_{n}$ and randomly choose $s^{(0)} \in \bbS$. At the $t$-th iteration, given $s^{(t-1)}$ and $\theta^{(t-1)}$ from the previous iteration, we consecutively execute the following two steps:
\begin{itemize}
    \item[\mbox{}] $\mathcal{A}_1$. Compute the optimal assortment $s^{(t)}$ given $\theta^{(t-1)}$: $s^{(t)} \in \argmax_{s \in \bbS}\cV(s,\theta^{(t-1)})$;
    \item[\mbox{}]  $\mathcal{A}_2$. Compute the optimal $\theta^{(t)}$ using $s^{(t)}$: $ \theta^{(t)} \in \argmin_{\theta \in \Omega_n}\cV(s^{(t)},\theta)$.
\end{itemize}

In particular, the step $\mathcal{A}_1$ corresponds to finding an optimal assortment for a known choice model with parameter $\theta^{(t-1)}$. As discussed in Section~\ref{sec:related_work}, there exists a streamline of works that address this problem efficiently. 
In our experiments, we use two existing linear programming(LP)-based algorithms that respectively solve step $\mathcal{A}_1$ under the MNL and NL models~\citep{AO_LP,davis2014assortment}.
Due to the space constraint and their tangential connection to our contribution, 
we summarize them in Appendix~\ref{sec:app-lp}.

We next propose a procedure for the step $\mathcal{A}_2$. For a given optimized assortment $s^{(t)}$, we aim to search for the worst-case model parameter $\theta^{(t)}$ from the uncertainty set $\Omega_{n}$ that minimizes the expected revenue. In particular, we employ a gradient descent with line search (GDLS) method to compute $\theta^{(t)}$ by solving the following problem.
\begin{equation}
    \theta^{(t)} \in \argmin_{\theta \in \Omega_n} \mathcal{V}(s^{(t)};\theta).
\end{equation}
Given a feasible initial parameter $\theta^{(t,0)}\in \Omega_n$, we run at most $M$ 
gradient descent steps. Note that feasible \(\theta^{(t,0)}\) is easy to obtain as $\widehat{\theta}_{n}$ and $\theta^{(t-1)}$ are guaranteed to be feasible by definition. Suppose $\beta_m$ is the step size for gradient descent in the $m$-th 
step. At each step $m \in \{1,2,\cdots,M\}$, we perform a line search to maintain the feasibility. In 
particular, given $\theta^{(t,m-1)}\in\Omega_n$, we first evaluate the gradient as $\xi_m = \nabla_{\theta} \mathcal{V}(s^{(t)};\theta^{(t,m-1)})$.
Then we initiate $\beta_m$ with a pre-specified step size $\beta_m = \widetilde{\beta}$, 
and check whether $\theta^{(t,m)} = \theta^{(t,m-1)} - \beta_m \xi_m$ is feasible, i.e. 
$\theta^{(t,m-1)} \in\Omega_{n}$. If not, we set $\beta_m \leftarrow c\beta_m$ for some pre-specified $c \in (0,1)$, and recompute 
$\theta^{(t,m)} = \theta^{(t,m-1)} - \beta_m \xi_m$. Such a search is repeated until 
$\theta^{(t,m)}$ is feasible. After $M$ steps, we set \(\theta^{(t)}=\theta^{(t,M)}\) and the step $\cA_2$ finishes. We provide the pseudocode in Algorithm~\ref{algo:GDLS} for the overall process. In all of our numerical studies, we set $M=3$, $\widetilde{\beta}=0.01$ and $c=\frac{1}{5}$, which perform well.

\begin{algorithm}[h!]
   \caption{Gradient Descent with Line Search (GDLS)}
   \label{algo:GDLS}
\begin{algorithmic}[1]
\State {\bfseries Input:} assortment $s^{(t)}$; uncertainty set hyperparameter $\alpha_n$;
initial parameter $\theta^{(t,0)}$; 
initial step size $\tilde{\beta}$; 
step shrinkage constant $c$; 
number of descent steps $M$
\State {\bfseries Output:} the optimized parameter $\theta^{(t)}$
\State{$\widehat{L}_{n}(\theta)= -\frac{1}{n}\sum_{i=1}^{n}\log p(A_{i}|S_i;\theta)$}
\State{$\widehat\theta_{\ML,n} \gets \argmin_{\theta \in \Theta}\widehat{L}_{n}(\theta)$; $\widehat L_{\ML} \gets \widehat L_n(\widehat\theta_{\ML,n})$}
\For{$m=1$ {\bfseries to} $M$}
\State{$\xi_m = \nabla_{\theta} \mathcal{V}(s^{(t)};\theta^{(t,m-1)})$ \quad {\fontfamily{pcr}\selectfont/*compute the gradient*/} }
\State{$\beta_m \gets \widetilde{\beta}$}
\State{$\theta^{(t,m)}\gets\theta^{(t,m-1)}-\beta_m\xi_m$}
\While{$n\cdot\left( \widehat L_n(\theta^{(t,m)})  - \widehat L_{\ML}\right) > \alpha_n$}\quad{\fontfamily{pcr}\selectfont/* check model feasibility */}
\State{$\beta_m \gets c\beta_m$ \quad{\fontfamily{pcr}\selectfont/* decrease the step size */} }
\State{$\theta^{(t,m)}\gets\theta^{(t,m-1)}-\beta_m\xi_m$}
\EndWhile
\EndFor
\State{$\theta^{(t)}\gets\theta^{(t,M)}$}
\end{algorithmic}
\end{algorithm}

\subsection{Data Generation Process and Baselines}\label{sec:data-generation}
We compare PASTA to assortment optimization without pessimism, which in the sequel goes by the name baseline method. 
Our method and the baseline method are evaluated on synthetic data for which the optimal assortment $s^\star$ and true choice model parameter $\theta_{0}$ are known so that we can compute the regret. We study two popular choice models, the MNL and NL models. We describe the data generation process and the baseline method in details below.

\paragraph{The MNL Model} For the MNL model, we consider the assortment optimization scenarios described by $N$, $K$, $d$, $n$ and $p$, where $N$ is the total number of available products; $K$ is the cardinality constraint of the assortments, i.e., $\mathbb S = \{s: |s| \leq K\}$; $d$ is the dimension of choice model parameter $\theta_{0}$ and item features $\{x_j\}_{j=1}^N$; $n$ is the sample size of the offline dataset; $p$ is the probability of observing the optimal assortment $s^\star$. Similar to \cite{chen2020dynamic}, we first generate the true parameter $\theta_{0}$ as a uniformly random unit $d$-dim vector. Recall that $r_j$ and $x_j$ denote the reward and feature of product $j$, respectively.
For $j \in \{1,\dots,N\}$, we generate the $r_j$ uniformly from the range $[5, 8]$ and $x_j$ as uniformly random unit $d$-dim vector such that $\exp(x_j^{\intercal}\theta_{0})\leq \exp(-0.6)$ to avoid degenerate cases, where the optimal assortments include too few items. Given such information, the true optimal assortment $s^\star$ can be computed. Then, we generate an offline dataset $\mathcal{D}=\{(S_i,A_i)\}_{i=1}^{n}$ with $n$ samples. For $i \in \{1,\dots,n\}$, we generate $S_i$ following the distribution $\pi_S$ such that $\pi_S(s^\star)=p$ and $\pi_S(s)=  (1-p)/(|\mathbb{S}|-1)$ for $s \neq s^\star$, where $0<p<1$ is the probability of observing the optimal assortment $s^\star$. After the assortment $S_i$ is sampled, the customer choice $A_i$ is sampled according to the probability computed by the MNL model as in~\eqref{eqn:mnl_ass} with the true parameter $\theta_{0}$. 

\paragraph{The NL Model} For the NL model, we assume without loss of generality that items are partitioned into $K$ nests, each containing an equal number of $\kappa$ items~\citep{davis2014assortment}. Thus, the total number of items is $K\cdot\kappa$, and 
we characterize the assortment optimization problems by $K$, $\kappa$, $d$, $n$ and $p$, where $K$ is the total number of nests and $\kappa$ is the number of products in each nest; $n$ is the sample size of the offline dataset; $p$ is the probability for sampling the optimal assortment; $d$ is the dimension of item features. As presented in Section~\ref{sec:application}, the NL model has parameters $\theta_{0} = (\widetilde{\theta}_{0},\lambda_{0})$ with $\widetilde{\theta}_{0} \in \bbR^d$ and $\lambda_{0} \in (0,1)^M$ 
We consider unconstrained assortment optimization for the NL model, i.e., $\bbS = 2^{[N]}\setminus\{\emptyset\}$. The data generating process is similar to that of the MNL model case. We generate the true parameters $\widetilde{\theta}_{0}$ as a uniformly random unit $d$-dim vector and $\lambda_{0}$ uniformly within $(0,1)^M$. For notation simpilicity, let $r_{kj}$ and $x_{kj}$ respectively denote the reward and feature of the $j$-th item in the $k$-th nest. 
For $k \in \{1,\dots,K\}$ and $j \in \{1,\dots,\kappa\}$, we generate $r_{kj}$ uniformly from the range $[5, 8]$ and $x_{kj}$ as uniformly random unit $d$-dim vector such that $\exp(x_{kj}^{\intercal}\theta_{0})\leq \exp(-0.6)$. 
The generating process of the offline dataset $\cD$ is identical to that of the MNL model case.

\paragraph{Baseline Methods} For both the MNL and NL models, we use assortment optimization without pessimism, i.e., the as-if approach defined in~\eqref{eqn:as-if-opt}, as the baseline method. In our experiments, we use the gradient descent method to find the MLE $\widehat\theta_{\ML,n}$ that minimizes the empirical negative log-likelihood function $\widehat{L}_n$. Then given $\widehat\theta_{\ML,n}$, the baseline method solves the assortment optimization problem in~\eqref{eqn:as-if-opt}. 

\begin{figure}[H]
\begin{minipage}{\textwidth}
\centering
   \subcaptionbox*{}{\includegraphics[width=0.8\textwidth]{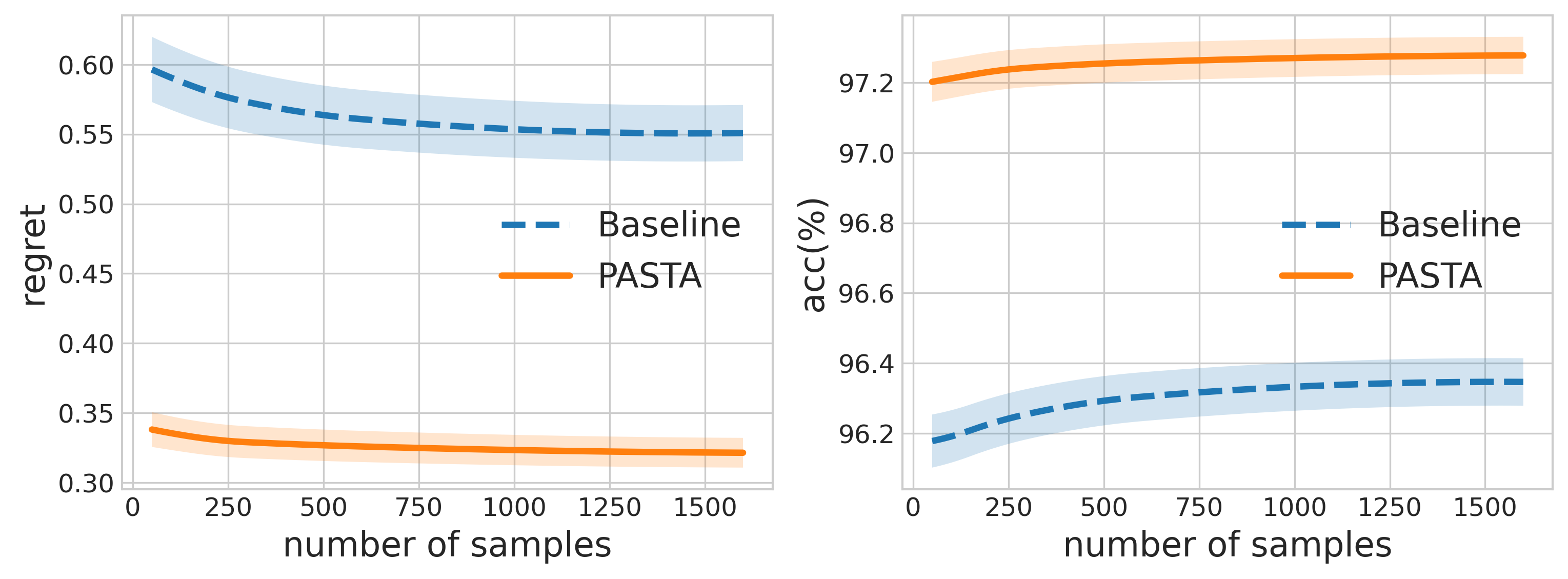}}\vspace{-1em}
\subcaptionbox*{}{\includegraphics[width=0.8\textwidth]{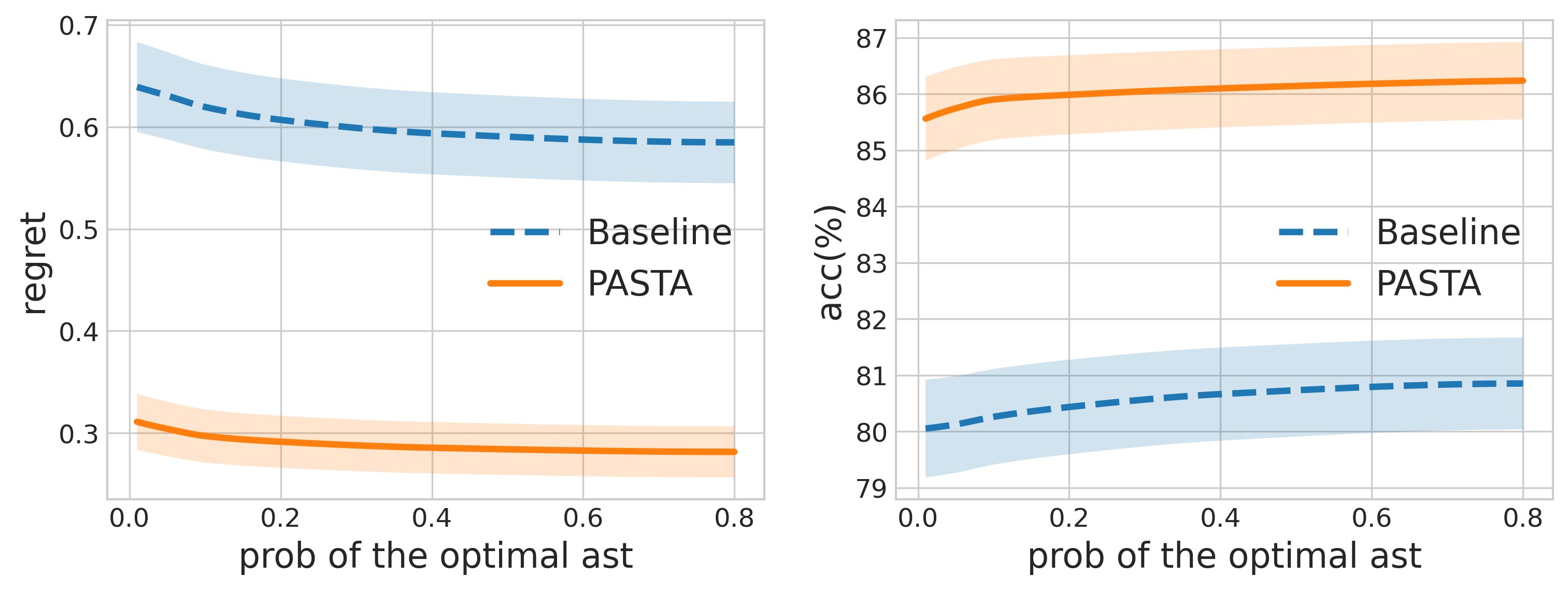}}\vspace{-1em}  
\subcaptionbox*{}{\includegraphics[width=0.8\textwidth]{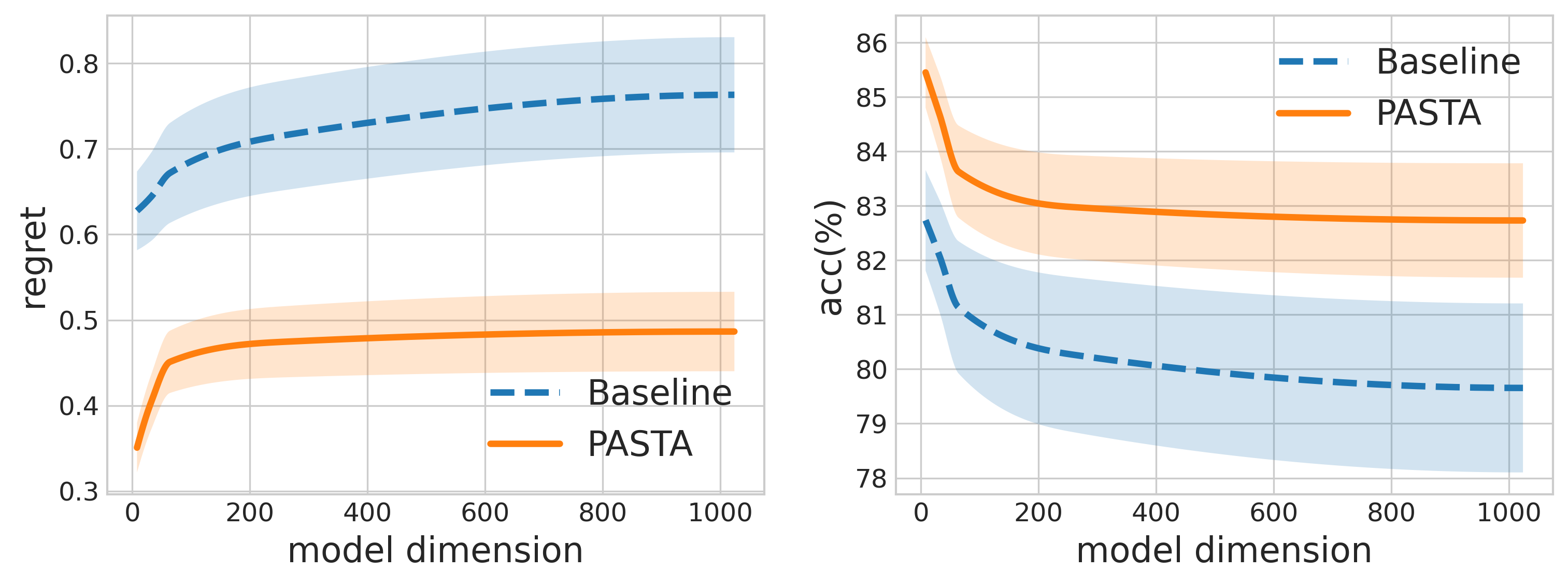}} 
\end{minipage}
\caption{Performance comparison between PASTA and the baseline method on the MNL model. The left panel shows average regret (lower is better), while the right panel presents assortment accuracy (higher is better).Our proposed method PASTA consistently outperforms the baseline method in both metrics. The performance gain remains consistent across varying probabilities of observing the optimal assortment (middle row) and varying model complexity, measured by the dimension of the model parameter (bottom row).}
\label{fig:mnl}
\end{figure}


\subsection{Performance Comparison}\label{sec:perf-comparison}
To measure the statistical variance of the results, we repeat the data generation process in Section~\ref{sec:data-generation} to randomly generate $50$ offline datasets. The solutions of PASTA and the baseline method are recorded in these experiments. 
We measure the performance with two metrics: (1) the \textit{regret} of the solutions which indicates how far the performance of the solutions is to that of the optimal performance, i.e., revenue of the optimal assortment $s^\star$; (2) the assortment \textit{accuracy} of the solutions with respect to the optimal assortment $s^\star$. The assortment accuracy of an assortment~$s$ is defined as the ratio of the number of correctly chosen products to the number of products in $s^\star$. 
For hyper-parameters, we set $\alpha_n = 2\widehat L_n(\widehat\theta_{\ML,n})$ and the maximum of iteration $T=30$.

\paragraph{The MNL Model} We first set $N=40$, $K=10$, $d=8$, $p=0.1$ and gradually increase the number of samples $n$. In Figure~\ref{fig:mnl}, we observe that PASTA consistently outperforms the baseline method in both regret and assortment accuracy. Other than the number of samples, another important parameter for the offline assortment optimization is the probability of observing the optimal assortment $\pi_S(s^\star)$. To this end,  Figure~\ref{fig:mnl} shows the performance of PASTA under varying $\pi_S(s^\star)$ from as low as $0.01$ to $0.8$. We observe that the gain of PASTA is consistent under a wide range of $\pi_S(s^\star)$.


\paragraph{The NL Model} We first set $K=3$, $\kappa=5$, $d=8$ and $p=0.1$, and gradually increase the number of samples $n$. In Figure~\ref{fig:nl}, we observe that PASTA consistently outperforms the baseline method. While the performance of the baseline method improves with increasing number of samples, PASTA maintains a regret that is less than $25\%$ of the baseline regret. Next, we study the robustness of PASTA against the probability of observing the optimal assortment $\pi_S(s^\star)$. In Figure~\ref{fig:nl}, we set $K=3$, $\kappa=5$, $d=8$, $n=200$ and increase $\pi_S(s^\star)$. As shown in Figure~\ref{fig:nl}, the gain of PASTA over the baseline method is consistent. Lastly, we study how the model complexity $d$ effects the performance. We set $K=3$, $\kappa=5$, $p=0.1$, $n=1000$ and increase $d$ from $8$ to $1024$. We present the results in Figure~\ref{fig:nl}. It can be observed that, while the regrets of PASTA and the baseline method both increase as $d$ increases, the gain of PASTA maintains.

\section{Conclusion}
We propose PASTA, a unified framework for offline assortment optimization under general choice models. Notably, the success of PASTA requires only that the offline data distribution includes an optimal assortment. We demonstrate the effectiveness of PASTA by applying it to several popular choice models, including the MNL, LCL and NL models, and establish the first finite-sample regret guarantees for the offline assortment optimization problem under these models. Furthermore, we derive the minimax regret lower bound under MNL, justifying the minimax optimality of PASTA in terms of the sample and model complexities. 
During the preparation of our manuscript, we notice a concurrent work~\citep{han2025learning} that explores assortment optimization for a specialized MNL model using the observational data.  
Their approach is ranking-based and requires only that the offline data contains the items in the optimal assortment, due to a stronger model assumption.
It remains an interesting research problem to investigate how to generalize the similar optimal item coverage for a broader class of choice models under a unified framework such as PASTA.

\begin{figure}[H]
\begin{minipage}{\textwidth}
\centering
   \subcaptionbox*{}{\includegraphics[width=0.40\textwidth]{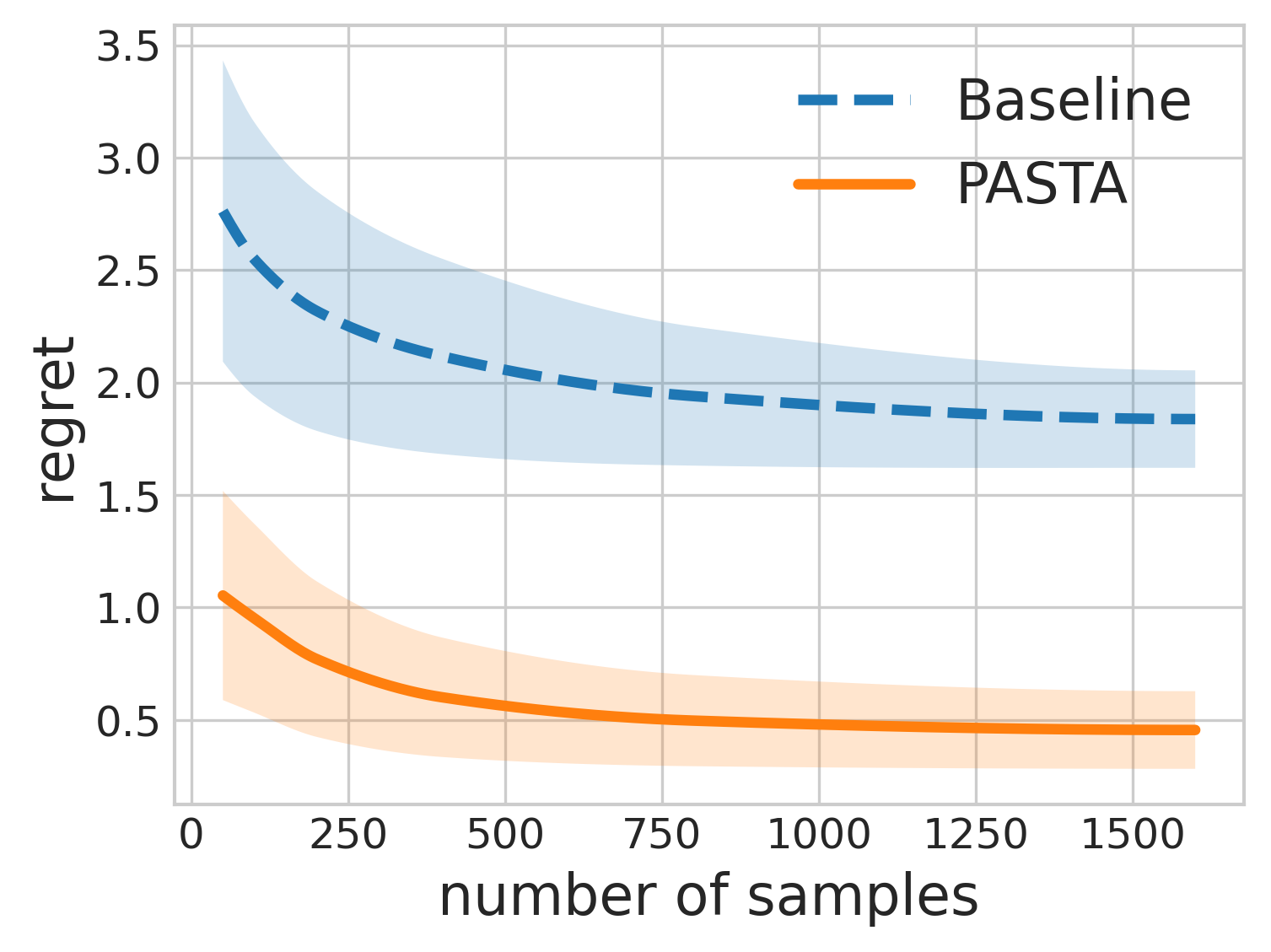}}
   \subcaptionbox*{}{\includegraphics[width=0.40\textwidth]{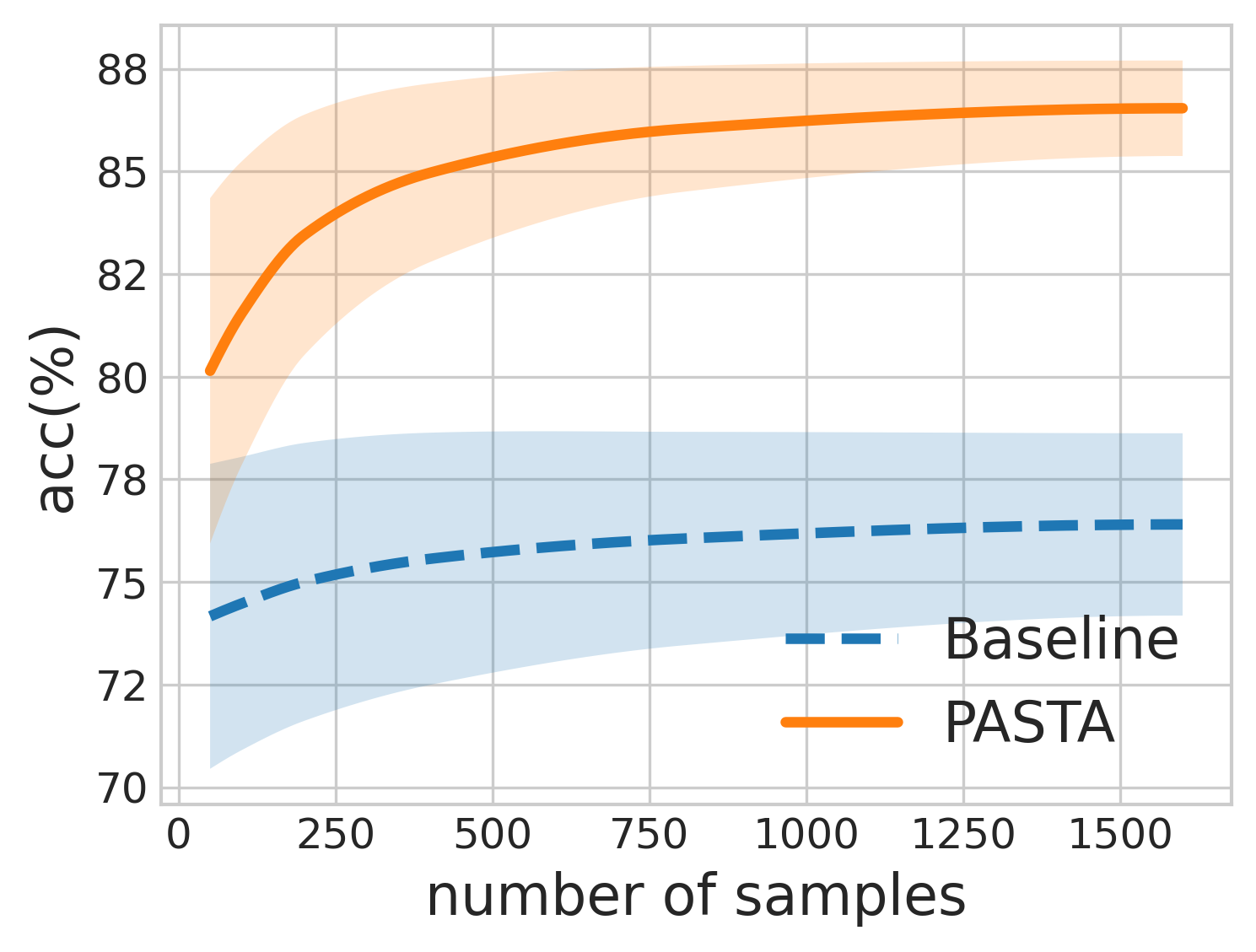}}\vspace{-1em}
\subcaptionbox*{}{\includegraphics[width=0.8\textwidth]{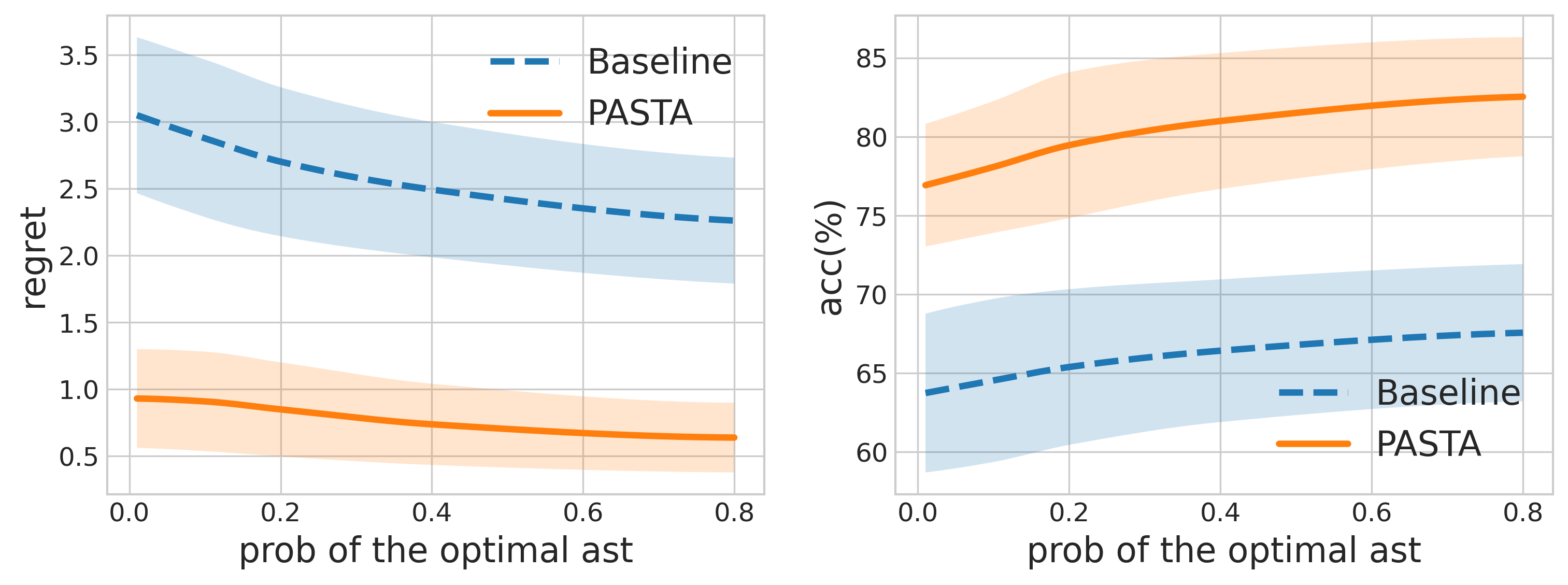}}\vspace{-1em}  
\subcaptionbox*{}{\includegraphics[width=0.40\textwidth]{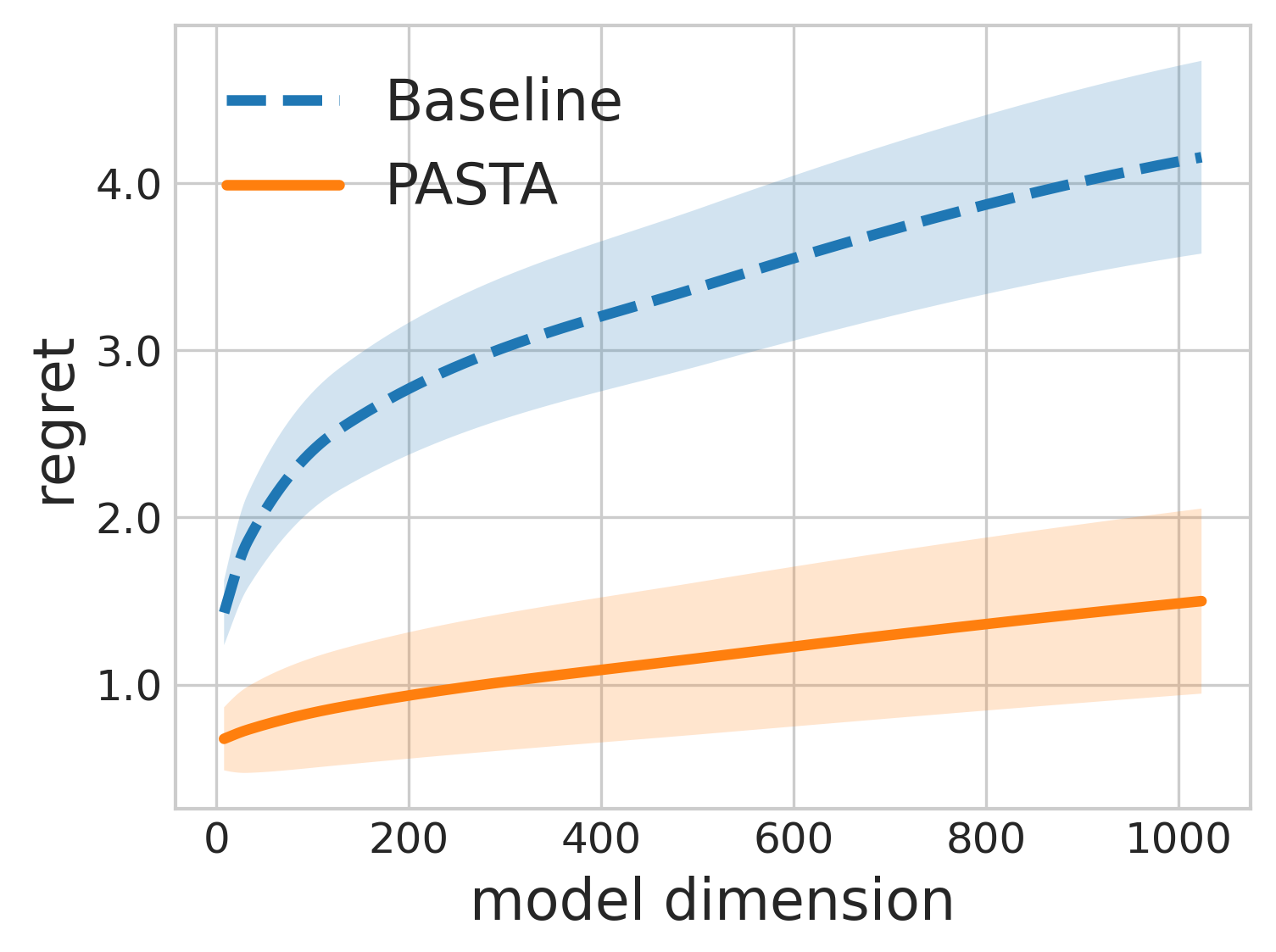}} 
\subcaptionbox*{}{\includegraphics[width=0.40\textwidth]{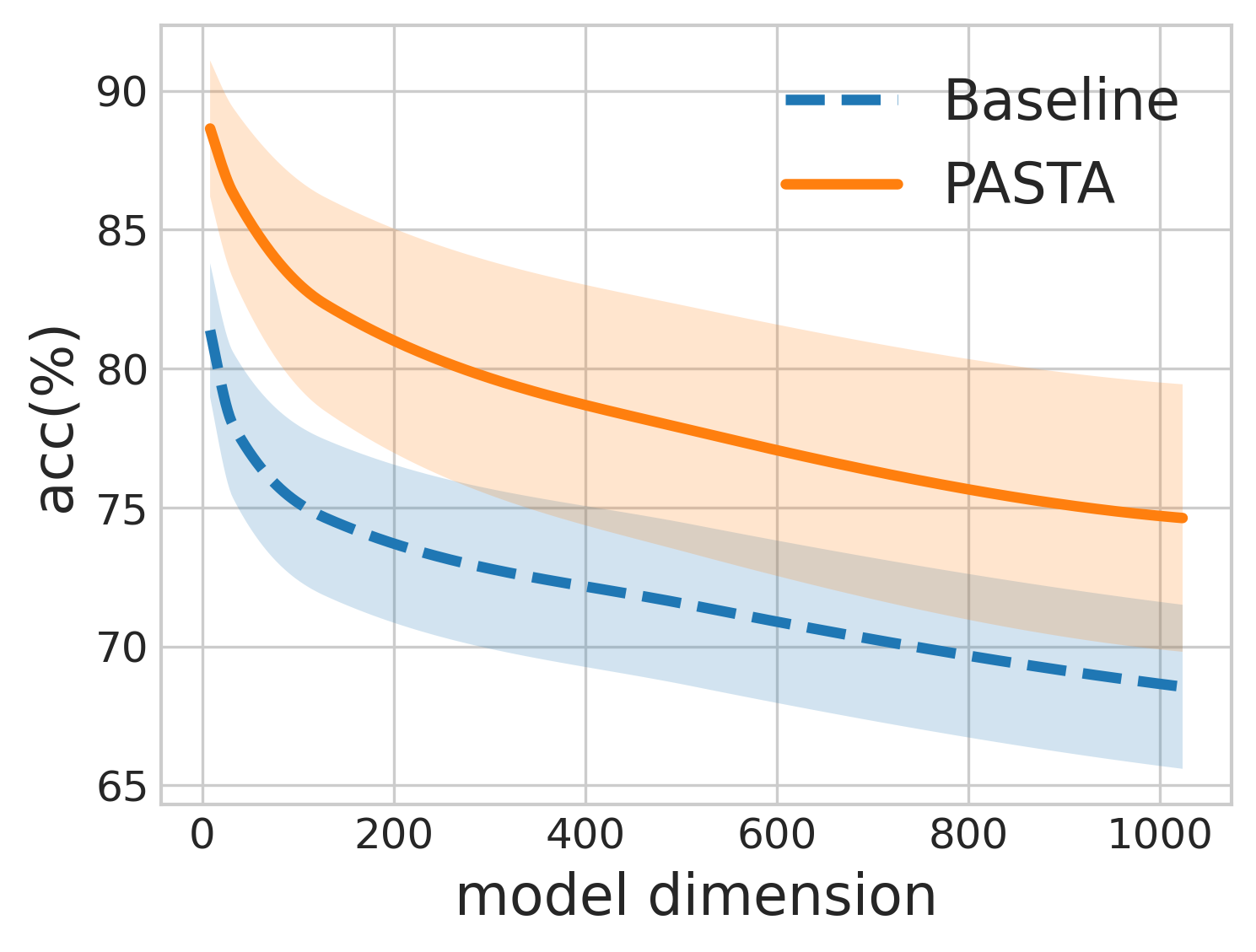}} 
\end{minipage}
\caption{PASTA also consistently outperforms the baseline method when the underlying choice model is an NL model. The performance gain in terms of the regret difference increases with increasing model complexity (bottom row).}
\label{fig:nl}
\end{figure}

\appendix

\section{Main Technical Theorem: Hellinger Concentration of MLE for Conditional Density}\label{sec:app-main}
We note that the results in this section are applicable to general scenarios that involve estimating conditional probability density function with MLE. Hence, in this section, we will use $Y$ and $X$ to respectively denote $A$ and $S$ to underscore that the results are generally applicable. 
Let $\cP$ be a family of conditional densities for $Y|X$. Assume a dataset $D$ consisting of $n$ independent and identically distributed samples of $(X,Y)$ following the probability density function (PDF) $(x,y)\mapsto \pi(x)p_0(y|x)$ where $p_0$ is the true data generating conditional PDF. The goal is to estimate $p_0$ with the observed data $D$. In particular, let the MLE $\widehat{p}_n$ of $p_0$ be defined as
\begin{equation}\label{eqn:mle}
    \widehat{p}_n \in \argmin_{p \in \cP}\left\{\widehat{L}_n(p) := -\frac{1}{n}\sum^n_{i=1}\log p(Y_i|X_i)\right\}.
\end{equation}
Next we present our main theorem.
\begin{theorem}\label{thm:main}
Define 
\[ \kappa(u) = \begin{cases}
			{2\left( e^{|u|/2} - 1 - |u|/2 \right) \over (1-e^{-u/2})^{2}}, & u \neq 0;\\
			1, & u = 0.
		\end{cases} \]
Let $\tau > 0$ be an arbitrary but fixed smoothing constant satisfying $\kappa(\tau) < 8(1-e^{-\tau})^{2}$, which is equivalent to  $0.5518 < \tau < 3.1926$.  
Let $n$ denote the number of samples. Suppose $\sigma_{n},\alpha_{n} > 0$ satisfy the following relationships: $ \alpha_{n} = {\sigma_{n}^{2}/8\kappa(\tau)}$, 
$\sigma_{n}^{2} \ge n^{-1}$, and
		\begin{align}\label{eqn:main-entropy}
\int_{\sigma_{n}^{2}/2^{10}}^{\sigma_{n}}\sqrt{\log\cN_{[]}\left({u \over 2\sqrt{2}e^{\tau/2}},\cP,H\right)}\rd u \le {1 \over 2^{27}}\sqrt{n}\sigma_{n}^{2}.
		\end{align}
		Let $c(\tau) = 2\big\{ {1-e^{-\tau} \over 4} - {\kappa(\tau) \over 32(1-e^{-\tau})} \big\} > 0$ and $c_{1} = 63/(257\times 2^{14})$.
        Consider the set estimate
		\[\Omega_{n}(\alpha) := \left\{ p \in \cP: \widehat{L}_{n}(p) - \widehat{L}_{n}(\widehat{p}_{n}) \le \alpha \right\}. \]
Then for every $\alpha \ge \alpha_{n}/4$, with probability at least $1 - c\cdot e^{-2c_{1}n\alpha}$, it holds that
		\[(i) p_0 \in\Omega_{n}(\alpha),\;\text{and}\;(ii)\sup_{p \in\Omega_{n}(\alpha)}H^2(p,p_0) \le {\alpha \over c(\tau)},\]
where $c$ is a universal constant. In particular, at $\tau = 1.7024$, $c(\tau)$ is maximized as $0.1802$.
\end{theorem}
\begin{remark}
    To connect with the statement in Theorem \ref{thm:main-simple}, we have
    $C_{1} = 2^{-7}\kappa(\tau)$, $C_{2} = 2\sqrt{2\kappa(\tau)}$, 
    $C_{3}=2\sqrt{2}e^{\tau/2}$,
    $C_{4} = 2^{-24}\kappa(\tau)$,
    $C_{5} = c$,
    $C_{6} = 2c_{1}$,
    $C_{7} = 1/c(\tau)$.
\end{remark}
\begin{proof}[Proof Sketch of Theorem~\ref{thm:main}]
We defer the detailed proof to the Appendix~\ref{sec:app-proofs} and only present a sketch of proof  below. All the intermediate results referred  below are also deferred to the Appendix due to space constraint.

The key difficulty is that likelihood–ratio functions are not uniformly bounded, which obstructs direct empirical-process control. We therefore introduce a smoothed class $\cP_\tau=\{(1-e^{-\tau})p+e^{-\tau}p_0:\,p\in\cP\}$ and the smoothed risks $L_\tau,\widehat L_{\tau,n}$; this guarantees a pointwise bound on the log-likelihood ratio (Lemma~\ref{lem:smooth}) while preserving correct specification ($p_0\in\cP_\tau$). We localize the analysis to the Hellinger ball $\{p:\,H^2(p,p_0)\le \alpha\}$ and consider the corresponding class of smoothed log-likelihood ratios $\cF^{\mathrm{loc}}_\tau(\alpha)$. By Lemma~\ref{lem:exp}, this localized class satisfies a Bernstein condition, yielding sharp concentration for $\|(\bbP_n-\bbP)\|_{\cF^{\mathrm{loc}}_\tau(\alpha)}$ (Theorem~\ref{thm:concentration}). 
Next, we identify that a special empirical ``margin'' event implies that any $p$ with $H^2(p,p_0)\ge \alpha$ must have strictly larger empirical risk than $p_0$; since $\widehat p_n$ minimizes $\widehat L_n$, it must lie in the local region $H^2(\widehat p_n,p_0)\le \alpha$. On this event, an empirical smoothed risk inequality together with $L_\tau(p_0)\le L_\tau(\widehat p_n)$ bounds the empirical excess $\widehat L_n(p_0)-\widehat L_n(\widehat p_n)$ by the localized empirical-process norm. A standard peeling argument then transfers this pointwise control to the whole uncertainty set $\Omega_n(\alpha)$, showing that all empirically near-optimal $p$ remain within Hellinger radius $\alpha/c(\tau)$ of $p_0$. Under the numeric condition $\kappa(\tau)<8(1-e^{-\tau})^2$ and for $\alpha\ge \alpha_n/4$, the concentration bounds deliver $\bbP\Big\{\,p_0\in\Omega_n(\alpha)\ \text{and}\ \sup_{p\in\Omega_n(\alpha)}H^2(p,p_0)\le \alpha/c(\tau)\Big\}
\;\ge\; 1-ce^{-c'n\alpha}$, with universal constants $c,c'>0$. 

\end{proof}

\section{Linear Programming (LP) for Assortment Optimization Problems with Known Choice Models}\label{sec:app-lp}
Here we summarize two LP-based approaches for solving the assortment optimization problem under known choice models. Specifically, we solve the following optimization problem
\[
\argmax_{s \in \bbS} \cV(s;\theta) = \sum_{j \in s} r_j p(a|s;\theta),
\]
where $p(a|s;\theta)$ is a choice model parameterized by $\theta$. Note that here $\theta$ is fixed and known. Different choice models require different optimization methods. This work considers two popular choice models -- the MNL and NL models. In both scenarios, the assortment optimization problem can be solved with LP-based methods~\citep{AO_LP,davis2014assortment}.

\subsection{MNL Model}

Let an assortment $s$ be represented by an $N$-dimensional binary vector $\gamma\in \{0,1\}^N$ where $\gamma_j = 1$ if and only if $j \in s$. Suppose that $s \in \bbS$ corresponds to the following feasible set for $\gamma$ with $M$ linear inequality constraints:
\[ \Gamma=\left\{\gamma \in \{0,1\}^N: \sum_{j \in N}a_{ij}\gamma_j\leq b_i ~ \text{for}~ i \in [M]\right\}, \]
where the matrix of constraint coefficients $[a_{ij}]_{i\in [M], j\in[N]}$ is a totally unimodular matrix. In other words, based on the one-to-one correspondence between $s$ and $\gamma$, we have $s \in \bbS$ if and only if $\gamma \in \Gamma$. 

Next, we denote $v_i = \exp(x_{i}^\intercal \theta)$ as the preference score for the $i$-th item, with item feature $x_i$ and revenue $r_i$. The customer choice probability under the MNL model \eqref{eqn:mnl_ass} becomes $p(i|s;\theta)=\frac{v_i}{1+\sum_{j\in s}v_j}$. The assortment optimization can be formulated as
\begin{align} \label{eqn:opt_reward}
\max_{\gamma \in \Gamma}\frac{\sum_{i\in [N]}r_iv_i\gamma_i}{1+\sum_{i\in[N]}v_i\gamma_i},
\end{align}
which is equivalent to the following linear programming problem~\citep{AO_LP}:
\begin{equation} \label{eqn:opt_reward_lp}
\begin{aligned}
\max_{w_j: j \in [N] \cup \{0\}} \ & \sum_{j\in[N]}r_j w_j\\
\text{subject to} \quad & \sum_{j\in[N]}w_j + w_0 = 1\\
& \sum_{j\in [N]}a_{ij}\frac{w_j}{v_j}\leq b_i w_0 \quad \forall i \in [M] \\
& 0 \leq \frac{w_j}{v_j} \leq w_0 \quad \forall j \in [N].
\end{aligned}
\end{equation}
In particular, we can recover the optimal solution to  Problem~\eqref{eqn:opt_reward}, denoted as $\gamma^\star$, using the optimal solution to Problem~\eqref{eqn:opt_reward_lp}, denoted by $w^\star$, via the following formula $\gamma_j^\star = \frac{w_j^\star}{v_j w_0^*} \quad \forall j \in [N]$.
Finally, an optimal assortment $s^\star(\theta)$ for a fixed choice $p(a|s;\theta)$ is obtained by the correspondence $j \in s^{\star}$ if and only if $\gamma_{j}^\star = 1$.

\subsection{NL Model}
Without loss of generality, we consider NL models with $K$ baskets where each basket contains $M$ items. Let an assortment $s$ for the NL model be represented by an array of $K$ binary vectors $\gamma = \{\gamma_k\}^K_{k=1}$ where $\gamma_k \in \{0,1\}^{M}$ represents the choice for the $k$-th basket. In particular, for $m \in [M]$, $\gamma_{k,m} = 1$ if we choose the $m$-th item in the $k$-th basket, with item feature $x_{k,m}$ and revenue $r_{k,m}$. For the purpose of optimization and without loss of generality, we assume that the items in each basket are sorted with decreasing revenues, i.e., $r_{k,m} \ge r_{k,m'}$ for $k \in [K]$, $m,m' \in [M]$, and $m \ge m'$. Let
\begin{align*}
&V_k(\gamma_k)=\sum^M_{m=1}\gamma_{k,m}\exp(\tilde{\theta}^{\intercal}x_{k,m}/\lambda_k);\\
& R_k(\gamma_k)=\frac{\sum^M_{m=1}\gamma_{k,m}\exp(\tilde{\theta}^{\intercal}x_{k,m}/\lambda_k)r_{k,m}}{V_k(\gamma_k)}.   
\end{align*}
The assortment optimization problem for fixed $\tilde{\theta},\lambda,x,r$ is therefore defined as
\begin{equation}\label{eqn:nl-opt-obj}
    \max_{\gamma=\{\gamma_k\}^K_{k=1}} \sum^K_{k=1}\frac{V_k(\gamma_k)^{\lambda_k}}{1+\sum^K_{k'=1}V_{k'}(\gamma_{k'})^{\lambda_{k'}}}R_k(\gamma_k).
\end{equation}
We can efficiently solve~\eqref{eqn:nl-opt-obj} through the following LP problem~\citep{davis2014assortment} 
\begin{equation}\label{eqn:nl-lp}
\begin{aligned}
    &\min_{\eta,y_1,\dots,y_K} \eta \\
    &\quad\mathrm{s.t.}\quad \eta \ge \sum^K_{k=1}y_k\\
     &k \in [K],\;y_k \ge V_k(\gamma_k)^{\lambda_k}(R_k(\gamma_k)-\eta) \qquad \forall \gamma_k \in \Gamma,
\end{aligned}
\end{equation}
where $\Gamma=\{\widetilde{M}_1,\widetilde{M}_2,\dots,\widetilde{M}_M\}$ has the subsets of baskets containing first $m$ items, i.e., for $m \in [M]$ $\widetilde{M}_m = \{1,\dots,1,0,\dots,0\} \in \{0,1\}^M$ with the first $m$ elements as $1$ and the last $M-m$ elements as $0$. After solving~\eqref{eqn:nl-lp} with solutions $\eta^\star, y^\star_k$, the optimal assortment $\gamma^\star_k$ can be constructed as $\gamma_k^\star = \argmax_{\gamma_k \in \Gamma}V_k(\gamma_k)^{\lambda_k}(R_k(\gamma_k)-\eta^\star)$.

\section{Proofs of Theoretical Results}\label{sec:app-proofs}

\subsection{Proof of Theorem~\ref{thm:regret}}\label{sec:proof-thm-regret}
\vspace{1em}
\begin{proof}[Proof of Theorem~\ref{thm:regret}]
By triangle inequality, for any $p \in \cP$, we have $\cV(s^{\star};p_0)-\cV(s^{\star};p) \leq \left| \cV(s^{\star};p)-\cV(s^{\star};\widehat{p}_{n})\right| +\left|\cV(s^{\star};\widehat{p}_{n})-\cV(s^{\star};p_0) \right|$.
If $p_0 \in \Omega_n$, we further have for any $p \in \Omega_n(\alpha_n)$
\begin{equation}\label{eqn:regret-triangle}
\begin{aligned}
 &\cV(s^{\star};p_0)-\cV(s^{\star};p) \leq 2\max_{p \in \Omega_n}\left| \cV(s^{\star};p)-\cV(s^{\star};\widehat{p}_{n})\right|.   
\end{aligned}
\end{equation}

Lemma~\ref{lemma:l1_distance_ub} implies that for any $p \in \cP$, $\left| \cV(s^{\star};p)-\cV(s^{\star};\widehat{p}_{n})\right| \leq r_{s^{\star}}\sqrt{C_{s^{\star}}\mathbb{E}_{S}\left[\left\|p(\cdot|S)-\widehat{p}_{n}(\cdot|S)\right\|_1^2\right]}$,
where $\|\cdot\|_1$ is the $L_1$ norm, $r_{s^{\ast}}=\max_{j \in s^{\star}}r_{j}$ is the largest possible revenue among all items and $C_{s^{\star}}=1/\pi_S(s^{\star})$. In Lemma~\ref{lemma:l1_H_ub}, we establish that \(\mathbb{E}_{S}\left[\left\|p(\cdot|S)-\widehat{p}_{n}(\cdot|S)\right\|_1^2\right] \leq 8H^2(p, \widehat{p}_{n})\),
where $H^2$ is the generalized squared Hellinger distance defined in~\eqref{eq:H}. Combining the above two inequalities, we have that for any $p \in \cP$, $\bigg| \cV(s^{\star};p)-\cV(s^{\star};\widehat{p}_{n})\bigg| \leq r_{s^{\ast}}\sqrt{8C_{s^{\star}}H^2(p,\widehat{p}_{n})}$.
From Assumption~\ref{ass: technical assumptions}~(III), 
we have that with probability at least $1-\delta$ where $0 < \delta < 1$,
for any $p \in \Omega_n(\alpha_n)$, $H^2(p,\widehat{p}_{n}) \le \alpha_n$. With this upper bound for $H^2(p,\widehat{p}_{n})$, we further have for probability at least $1-\delta$, 
\begin{equation}\label{eqn:V-ub}
    \sup_{p \in \Omega_n}\bigg| \cV(s^{\star};p)-\cV(s^{\star};\widehat{p}_{\ML,n})\bigg| \lesssim r_{s^{\ast}}\sqrt{\alpha_n(\delta)/\pi_S(s^\star)}.
\end{equation}
Lastly, under Assumption~\ref{ass: technical assumptions}~(III), with probability at least $1-\delta$ we have $p_0 \in \Omega_n$. Then application of Lemma~\ref{lemma:general_pasta_regret} along with the inequalities in~\eqref{eqn:regret-triangle} and~\eqref{eqn:V-ub} conclude the proof.
\end{proof}
\subsection{Proof of Theorem~\ref{thm:main-simple}}
\begin{proof}[Proof of Theorem~\ref{thm:main-simple}]
This result is a simplified result of Theorem~\ref{thm:main}. In particular, with $\kappa(\tau)$ and $c(\tau)$ defined as in Theorem~\ref{thm:main}, choosing the smoothing constant $\tau=2$ directly leads to Theorem~\ref{thm:main-simple}. 
\end{proof}
\subsection{Proof of Theorem~\ref{thm:main}}
Our proof builds upon the empirical process on likelihood ratio functions. In particular, for a family of $\bbR$-valued functions $\cF$, denote $\| \bbP_{n} - \bbP \|_{\cF} = \sup_{f \in \cF}|\bbP_{n}(f) - \bbP(f)|$. Note that under the assumption of correct specification, i.e., $p_0 \in \cP$, $p_0$  minimizes the population negative log-likelihood:
	\begin{align}
		p_0 \in \argmin_{p \in \cP}\Big\{ L(p) := \bbE[- \log p(Y|X)] \Big\}.
		\label{eq:truth}
	\end{align}
To establish empirical process results without the assumption of uniform boundedness on the function class, our proof relies on a technique known as smoothing. To this end, for an arbitrary but fixed smoothing constant $\tau > 0$, consider the following \emph{smoothed family of conditional PDFs} $\cP_{\tau} := \left\{ (1-e^{-\tau})p + e^{-\tau}p_0: p \in \cP \right\}$. In particular, for $p = p_0$, we have $p_0 \in \cP_{\tau}$.
That is, $\cP_{\tau}$ is also a correctly specified family.
Moreover, the smoothed log-likelihood ratio is upper bounded as $\log{p_0(Y|X) \over (1-e^{-\tau})p(Y|X) + e^{-\tau} p_0(Y|X)}\le \tau$.
Define the corresponding smoothed negative log-likelihoods as 
\begin{equation}\label{eq:smooth}
    \begin{aligned}
        &L_{\tau}(p) := \bbE\Big\{ - \log[(1-e^{-\tau})p(Y|X)+ e^{-\tau}p_0(Y|X)] \Big\}; \\
        &\widehat{L}_{\tau,n}(p ) := \bbE_{n}\Big\{ - \log[(1-e^{-\tau})p(Y|X) + e^{-\tau}p_0(Y|X)] \Big\};
    \end{aligned} 		
\end{equation}
To facilitate proof, for any $\tau>0$ and $\alpha \ge 0$, we also define the local family of smoothed log-likelihood ratios $\cF_{\tau}^{\rm loc}(\alpha)$: for $p \in \cP$, let $f_{p,\tau}(X,Y)=\log{p_0(Y|X)} - \log[(1-e^{-\tau})p(Y|X) + e^{-\tau} p_0(Y|X)]$, and define $\cF_{\tau}^{\rm loc}(\alpha)$ as
\begin{align}
		&\cF_{\tau}^{\rm loc}(\alpha) := \left\{ f_{p,\tau}: H^{2}(p,p_0) \le \alpha \right\}.
		\label{eq:smooth_loc}
	\end{align}
Notably, per Lemma~\ref{lem:exp}, it satisfies the Bernstein's condition. That is, for any $f_{p,\tau} \in \cF_{\tau}^{\rm loc}(\alpha)$ and $k \ge 2$, we have $\bbE |f_{p,\tau}|^{k} \le {k! \over 2}2^{k-2} \times 8\kappa(\tau)H^2(p,p_0)$.
In particular, $\bbE(f_{p,\tau}^{2}) \le 8\kappa(\tau)H^2(p,p_0)$. 

\begin{proof}{Proof of Theorem~\ref{thm:main}}
Fix $\alpha \ge 0$ and consider the following event $$\mathcal{E}_{n}(\alpha) :=  \left\{ \inf_{p \in \cP:H^2(p,p_0) \ge \alpha}\left\{ \widehat{L}_{n}(p) - \widehat{L}_{n}(p_0) \right\} \le c(\tau)\alpha \right\}.$$
On $\mathcal{E}_{n}(\alpha)^{\complement}$, we have $$\inf_{p \in \cP:H^2(p,p_0)\ge\alpha}\Big\{ \widehat{L}_{n}(p ) - \widehat{L}_{n}(p_0) \Big\} > c(\tau)\alpha \ge 0,$$
while $\widehat{L}_{n}(\widehat{p }_{n}) - \widehat{L}_{n}(p_0) \le 0$ because $\widehat{p}_n$ minimizes $\widehat{L}_{n}$. 
Therefore,  we have $H^{2}(\widehat{p }_{n},p_0) \le \alpha $,
and hence $f_{\widehat{p }_{n}} \in \cF_{\tau}^{\rm loc}(\alpha)$.
The following inequalities hold 
\begin{align}
    &\widehat{L}_{n}(p _0) - \widehat{L}_{n}(\widehat{p }_{n}) \le {1 \over 1-e^{-\tau}}\left\{ \widehat{L}_{\tau,n}(p _0)  - \widehat{L}_{\tau,n}(\widehat{p }_{n}) \right\}\\
    &\le {1 \over 1-e^{-\tau}}\left\{ L_{\tau}(p _0) - L_{\tau}(\widehat{p }_{n}) + \| \bbP_{n} - \bbP \|_{\cF_{\tau}^{\rm loc}(\alpha)} \right\} \le {1 \over 1-e^{-\tau}}\| \bbP_{n} - \bbP \|_{\cF_{\tau}^{\rm loc}(\alpha)},
\end{align}
where the first inequality is due to Lemma~\ref{lem:smooth}, the second inequality holds because $f_{\widehat{p }_{n}} \in \cF_{\tau}^{\rm loc}(\alpha)$, and the last inequality is due to $L_{\tau}(p _0) = \min_{p  \in p }L_{\tau}(p) \le L_{\tau}(\widehat{p}_n)$.
Let $\cA_{n}(\alpha) := \left\{ \| \bbP_{n} - \bbP \|_{\cF_{\tau}^{\rm loc}(\alpha)} \ge {\kappa(\tau) \over 32}\alpha \right\}$. On $\cA_{n}(\alpha)^{\complement}$ , we further have \(\| \bbP_{n} - \bbP \|_{\cF_{\tau}^{\rm loc}(\alpha)} \le {\kappa(\tau) \over 32}\alpha\).
That is, on $\cE_{n}(\alpha)^{\complement}\cap\cA_{n}(\alpha)^{\complement}$ , we have that \(\widehat{L}_{n}(p _0) - \widehat{L}_{n}(\widehat{p }_{n}) \le {\kappa(\tau) \over 32(1-e^{-\tau})} \alpha\), which by definition of $\Omega_{n}$ is equivalent to $p _0 \in\Omega_{n}\left( {\kappa(\tau) \over 32(1-e^{-\tau})} \alpha\right)$. In other words, $p _0 \in\Omega_{n}(\alpha)$ under the event $\cE_{n}\left( {32(1-e^{-\tau})\over \kappa(\tau)}\alpha \right)^{\complement}\cap\cA_{n}\left( {32(1-e^{-\tau})\over \kappa(\tau)}\alpha \right)^{\complement}$.
Recall that on $\cE_{n}(\alpha)^{\complement}$, we have
\begin{align}
    \inf_{p \in \cP:H^2(p,p_0) \ge \alpha}\Big\{ \widehat{L}_{n}(p ) - \widehat{L}_{n}(p_0) \Big\} > c(\tau)\alpha.
    \label{eq:bad_event}
\end{align}
For any $p  \in\Omega_{n}(c(\tau)\alpha)$, we also have $\widehat{L}_{n}(p ) - \widehat{L}_{n}(p_0) \le \widehat{L}_{n}(p ) - \min_{\widetilde{p } \in p }\widehat{L}_{n}(\widetilde{p }) \le c(\tau)\alpha$.
Then for any $p  \in\Omega_{n}$, $H^2(p,p_0)$ cannot exceed $\alpha_{n}$, or otherwise violating \eqref{eq:bad_event}. This implies that $\sup_{p  \in\Omega_{n}(c(\tau)\alpha)}H^2(p,p_0) \le \alpha$.
Equivalently, $\sup_{p  \in\Omega_{n}(\alpha)}H^2(p,p_0) \le {\alpha/c(\tau)}$ on the event $\cE_{n}\left( {\alpha/ c(\tau)} \right)^{\complement}$.
Thus, we have 
$$\bbP\left\{ p _0 \in\Omega_{n}(\alpha), ~ \sup_{p  \in\Omega_{n}(\alpha)}H^{2}(p , p _0) \le {\alpha \over c(\tau)} \right\}$$ 
lower bounded by 
$$\bbP\left\{ \cE_{n}\left( {32(1-e^{-\tau})\over \kappa(\tau)}\alpha \right)^{\complement}\cap\cA_{n}\left( {32(1-e^{-\tau})\over \kappa(\tau)}\alpha \right)^{\complement} \cap \cE_{n}\left( {\alpha \over c(\tau)} \right)^{\complement} \right\},$$ 
which by the same calculation in~\eqref{eqn:en-and-an} is further lower bounded by 
$$1- \bbP\left\{\bigcup_{j=0}^{+\infty}\cA_{n}\left( {32(1-e^{-\tau})\over \kappa(\tau)}2^{j}\alpha \right) \right\} -\bbP\left\{ \bigcup_{j=1}^{+\infty}\cA_{n}\left( {2^{j}\alpha \over c(\tau)} \right) \right\}.$$

To continue, note that with the condition $\kappa(\tau) < 8(1-e^{-\tau})^{2}$, we have  
$$\max\left\{ {\kappa(\tau) \over 32(1-e^{-\tau})}, {c(\tau) \over 2} \right\} \le {1 - e^{-\tau} \over 4} \le {1 \over 4}.$$
With this fact, for any $\alpha \ge \alpha_n/4 \gtrsim n^{-1}$, by Theorem \ref{thm:concentration}, we have $$\bbP\left\{\bigcup_{j=0}^{+\infty}\cA_{n}\left( {32(1-e^{-\tau})\over \kappa(\tau)}2^{j}\alpha \right) \right\} \le c\cdot\exp(-4c_2n\alpha),$$
where $c$ is a universal constant.
Similarly, we have $\bbP\left\{ \bigcup_{j=1}^{+\infty}\cA_{n}\left( {2^{j}\alpha \over c(\tau)} \right) \right\} \le c\cdot\exp(-2c_2n\alpha)$,
where $c$ is also a universal constants.  
Hence, $\bbP\left\{ p _0 \in\Omega_{n}(\alpha), ~ \sup_{p  \in\Omega_{n}(\alpha)}H^{2}(p , p _0) \le {\alpha \over c(\tau)} \right\} \ge 1- c\cdot\exp(-2c_2n\alpha)$.
Note that all $c$'s represent different universal constants. Notably, Theorem \ref{thm:concentration} requires that ${32(1-e^{-\tau})\over \kappa(\tau)}\alpha \ge \alpha_{n};\; {2\alpha \over c(\tau)} \ge \alpha_{n}$,
which are satisfied due to the condition that $\alpha \ge {\alpha_{n} / 4} \ge \max\left\{ {\kappa(\tau) / 32(1-e^{-\tau})}, {c(\tau) /2} \right\}\alpha_{n}$.
\end{proof}

\begin{theorem}[Rate Theorem]
		\label{thm:rate}
Let $\tau, \alpha_{n}, c(\tau), c_2$ be defined as in Theorem~\ref{thm:main}. Then for every $\alpha \ge \alpha_{n}/2$, consider the following event $$\mathcal{E}_{n}(\alpha) := \left\{ \sup_{p \in \cP:H^2(p,p_0) \ge \alpha}\left\{ \prod_{i=1}^{n}{p(Y_{i}|X_{i}) \over p_0(Y_{i}|X_{i})} \right\} \ge e^{-c(\tau)n\alpha} \right\},$$ 
which is equivalent to the event $\left\{ \inf_{p \in \cP:H^2(p,p_0) \ge \alpha}\left\{ \widehat{L}_{n}(p ) - \widehat{L}_{n}(p_0) \right\} \le c(\tau)\alpha \right\}$, 
we have $\bbP\{\mathcal{E}_{n}(\alpha)\} \le c\cdot e^{-c_{2}n\alpha}$ where $c$ is a universal constant.
\end{theorem}
\begin{proof}{Proof of Theorem \ref{thm:rate}}
Consider the local family of smoothed log-likelihood ratios $\cF_{\tau}^{\rm loc}(\alpha)$ in \eqref{eq:smooth_loc}.
		Let $J(\alpha) := \min\{ J \in \bbN: 2^{J}\alpha \ge 1 \}$.
		Then, with Lemma~\ref{lem:smooth}, we further have
$$\sup_{p: H^2(p,p_0) > \alpha}\left\{ \widehat{L}_{n}(p_0) - \widehat{L}_{n}(p ) \right\}\le {1 \over 1-e^{-\tau}}\sup_{p: H^2(p,p_0) > \alpha}\left\{ \widehat{L}_{\tau,n}(p_0) - \widehat{L}_{\tau,n}(p ) \right\}.$$
Let $\cP_j(\alpha) = \{p  \in \cP : 2^{j-1}\alpha < H^2(p,p_0) \le 2^{j}\alpha\}$, and we have $\sup_{p: H^2(p,p_0) > \alpha}\left\{ \widehat{L}_{\tau,n}(p_0) - \widehat{L}_{\tau,n}(p) \right\} \le \max_{1 \le j \le J(\alpha)}\bigg\{ \sup_{p  \in \cP_j(\alpha)}\left\{ L_{\tau}(p_0) - L_{\tau}(p ) \right\} + \| \bbP_{n} - \bbP \|_{\cF_{\tau}^{\rm loc}(2^{j}\alpha)} \bigg\}$.
Moreover, for each $j \in J(\alpha)$ and $p  \in \cP$ such that $2^{j-1}\alpha < H^2(p,p_0) \le 2^{j}\alpha$, we have $L_{\tau}(p_0) - L_{\tau}(p ) \le -2H^{2}\big( (1-e^{-\tau})p + e^{-\tau}p_0, p_0 \big) \le -{(1-e^{-\tau})^{2} \over 4}2^{j}\alpha$,
where the first inequality is by Lemma \ref{lem:H_le_KL}, the second inequality is by Lemma \ref{lem:H_smooth} and that $H^2(p,p_0) > 2^{j-1}\alpha$. Hence, we have
$$\sup_{p: H^2(p,p_0) > \alpha}\left\{ \widehat{L}_{n}(p_0) - \widehat{L}_{n}(p ) \right\} \le {1 \over 1-e^{-\tau}}\max_{1 \le j \le J(\alpha)}\bigg\{ -{(1-e^{-\tau})^{2} \over 4}2^{j}\alpha +~\| \bbP_{n} - \bbP \|_{\cF_{\tau}^{\rm loc}(2^{j}\alpha)} \bigg\}.$$
Note that the event ${1 \over 1-e^{-\tau}} \max_{1 \le j \le J(\alpha)}\left\{ -{(1-e^{-\tau})^{2} \over 4}2^{j}\alpha + \| \bbP_{n} - \bbP \|_{\cF_{\tau}^{\rm loc}(2^{j}\alpha)} \right\} \ge -c(\tau)\alpha$  implies the event $\exists j$ such that $1 \le j \le J(\alpha)$ and $\|\bbP_{n} - \bbP \|_{\cF_{\tau}^{\rm loc}(2^{j}\alpha)} \ge (1 - e^{-\tau})\left( {1-e^{-\tau} \over 4}2^{j} - c(\tau) \right)\alpha \ge {\kappa(\tau) \over 32}2^{j}\alpha$.
Here, the last inequality is based on the fact that for all $j \ge 1$, we have $c(\tau) \le 2^{j}\left\{ {1 - e^{-\tau} \over 4} - {\kappa(\tau) \over 32(1-e^{-\tau})}\right\}$
if and only if $(1 - e^{-\tau})\left( {1-e^{-\tau} \over 4}2^{j} - c(\tau) \right) \ge {\kappa(\tau) \over 32}2^{j}$.
For $\alpha \ge 0$, define $\cA_{n}(\alpha) := \left\{ \| \bbP_{n} - \bbP \|_{\cF_{\tau}^{\rm loc}(\alpha)} \ge {\kappa(\tau) \over 32}\alpha \right\}$. 
Here, for all $\alpha \ge \alpha_{n}/4$ and $1 \le j \le J(\alpha)$, we have $2^{j}\alpha \ge \alpha_{n}/2$, and hence Theorem \ref{thm:concentration} can be applied to $\cA_{n}(2^{j}\alpha)$. That is, $\bbP\left\{\cA_{n}(2^{j}\alpha)\right\} \leq c\cdot\exp(-2^jc_2 n\alpha)$ where $c$ is a universal constant and $c_2 = 63/(257\times 2^{14})$. With this we have
\begin{equation}\label{eqn:en-and-an}
\begin{aligned}
& \bbP\{\cE_{n}(\alpha)\} \le \bbP\left\{ \bigcup_{j=1}^{J(\alpha)}\cA_{n}(2^{j}\alpha) \right\}  \leq c\sum_{j=1}^{J(\alpha)}e^{-2^{j}c_{2}n\alpha}.  
\end{aligned}
\end{equation}
Lastly, because $\alpha \ge \alpha_n \gtrsim n^{-1}$, we have $\sum_{j=1}^{J(\alpha)}e^{-2^{j}c_{2}n\alpha}\le \sum_{j=1}^{\infty}(e^{-c_{2}n\alpha})^{2^j} \le c\cdot e^{-c_{2}n\alpha}$,
where $c$ is a universal constant that is different to the previous universal constants.
\end{proof}

\begin{theorem}\label{thm:concentration}
Let $\alpha_n$ and $\kappa(\tau)$ be defined as in Theorem~\ref{thm:main}. Consider the local family of smoothed log-likelihood ratios $\cF_{\tau}^{\rm loc}(\alpha)$ in \eqref{eq:smooth_loc}. 
Then for every $\alpha \ge \alpha_{n}$, 
$$\bbP\left( \| \bbP_{n} - \bbP \|_{\cF_{\tau}^{\rm loc}(\alpha)} \ge {\kappa(\tau) \over 32}\alpha \right) \le C e^{-c_1\kappa(\tau)n\alpha},$$ 
where $C$ is a universal constant and $c_1 = 63/(257\times 2^{14})$.
	\end{theorem}
	
	\begin{proof}{Proof of Theorem \ref{thm:concentration}} This result is a direct consequence of Theorem~\ref{thm:one-side}. By Lemma \ref{lem:exp}, the smoothed local family $\cF_{\tau}^{\rm loc}(\alpha)$ satisfies the Bernstein's condition in Theorem \ref{thm:one-side} with $\sigma^{2} = 8\kappa(\tau)\alpha$. Let $\sigma^2_n$ be defined as in Theorem~\ref{thm:main}. Due to $\alpha \ge \alpha_n$, we have $\sigma^{2} = 8\kappa(\tau)\alpha \ge 8\kappa(\tau)\alpha_n = \sigma_n^2 \ge 1/n$. 
		Moreover, by \eqref{eqn:main-entropy} and Lemmas \ref{lem:number}, \ref{lem:entropy_integral}, 
		it further satisfies the bracketing integral condition for $s = 6$ and every $\alpha \ge \alpha_{n}$.
Then Theorem \ref{thm:one-side} can be applied to establish the desired result.
	\end{proof}
 \begin{theorem}\label{thm:one-side}
Suppose $\cF$ is a family of $\bbR$-valued random variables. Assume that there exists some $\sigma^{2} < +\infty$, such that for any $f \in \cF$, it satisfy the Bernstein's condition: 
$$\bbE|f-\bbE[f]|^{k} \le {k! \over 2}2^{k-2} \sigma^{2};\;\forall k \ge 2.$$
Let $M  > 0$ and $\varepsilon \in (0,1)$. Suppose $n, M, \varepsilon, \sigma^2$ satisfy $M=\varepsilon\sigma^2/4$, $\sigma^2 \ge n^{-1}$, 
    \begin{equation}\label{eqn:integral_condition}
    \int_{\varepsilon M \over 32}^{\sigma}\sqrt{\log\cN_{[]}(u,\cF,\cL^{2}(\bbP))}\rd u \le {M\sqrt{n}\varepsilon^{3/2} \over 2^{10}},
    \end{equation}
    where $\cN_{[]}(u,\cF,\cL^{2}(\bbP))$ is the minimal number of $u$-brackets to cover $\cF$.
Then $\bbP\left\{ \| \bbP_{n} - \bbP \|_{\cF} \ge M \right\} \le C\exp\left\{ -(1-\varepsilon)\frac{nM^2}{4(\sigma^2+M)}\right\}$,
where $C$ is a universal constant. In particular, choose $\varepsilon=1/2^s$ for $s > 0$, the bracketing integral condition~\eqref{eqn:integral_condition} becomes
\begin{align}
\int_{\sigma^{2}\over2^{7+s/2}}^{\sigma}\sqrt{\log\cN_{[]}(u,\cF,\cL^{2}(\bbP))}\rd u \le {1 \over 2^{12+5s/2}}\sqrt{n}\sigma^{2},
\label{eq:entropy_integral}
\end{align}
and for any $\sigma^2$ such that~\eqref{eq:entropy_integral} is satisfied, $\bbP\left( \| \bbP_{n} - \bbP \|_{\cF} \ge {\sigma^{2} \over 2^{2+s}} \right) \le C\exp\left( -\Delta n\sigma^{2} \right);\;\Delta={2^{s} - 1 \over 2^{5+2s}(2^{2+s}+1)}$.
\end{theorem}
\begin{proof}{Proof of Theorem~\ref{thm:one-side}}
Consider a decreasing sequence of real numbers $\delta_0 > \delta_1 > \dots > \delta_N >0$. Because $\cF$ has a finite bracketing number, then by definition of $\log\cN_{[]}$, we have for every $\delta_j, j \in [N]$, there exists a finite set $\cF_j \subset \cF$ with $|\cF_j| = \cN_{[]}(\delta_j,\cF,\cL^{2}(\bbP))$ such that for every $f \in \cF$, there exists a pair of function $g_j^u(f), g_j^l(f) \in \cF_j$ with $g_j^l(f) \leq f \leq g_j^u(f)$ almost surely with respect to $\bbP$, and $\bbE[(g_j^u-g_j^l)^2]^{1/2} \le \delta_j$. Hence, given a fixed $f \in \cF$, we can construct a set of functions $\{g_j^l, g_j^u\}_{j=0}^N$ for every $\delta_0 > \delta_1 > \dots > \delta_N >0$.
For $k \in [N]$, let $u_k(f)$ and $l_k(f)$ be defined as $u_k(f) = \min_{j \le k} g_j^u(f)$ and $l_k(f) = \max_{j \le k} g_j^l(f)$.
We will write $u_k$ and $l_k$ in the following proof if their dependency on $f$ is clear in context.
It is easy to observe that $g_0^l = l_0 \le l_1 \le l_2 \le \dots \le l_N \le f \le u_N \le \dots \le u_0 = g_0^u$. Hence, with $\{g_j^l, g_j^u\}_{j=0}^N$, we have built a sequence of brackets with decreasing width. Note that by definition, $u_j \le g_j^u$ and $l_j \ge g_j^l$, we have $\bbE[(u_j-l_j)^2] \le \bbE[(g_j^u-g_j^l)^2] \le \delta_j^2$ for $j \in [N]$.
Consider another decreasing sequence $a_1 > a_2 > \dots > a_N$ which need not to be positive. We will use $\{a_i\}_{i=1}^N$ to construct truncation. Let $\Omega$ be the domain of $\cF$, and let $A_0 = \{\omega \in \Omega: u_0 - l_0 \ge a_1\}$, $A_k = \{\omega \in \Omega: u_k - l_k \ge a_{k+1} \quad\text{and}\quad u_j - l_j < a_{j+1} \quad \text{for} \quad j = 0, \dots, k-1 \},\quad k \in \{1,\dots,N-1\}$, and $A_N = \left(\bigcup_{j < N}A_j\right)^\complement$.
Note that $\{A_j\}_{j=0}^N$ forms a partition of $\Omega$ 
and thus $\sum_{j=0}^N\mathbf{1}_{A_j} = 1$. 
With this and some computation, we can write every $f \in \cF$ as 
$f = u_0 + \sum_{k=1}^N(u_k-u_{k-1})\mathbf{1}_{\bigcup_{j\ge k} A_j } + \sum_{k=0}^N (f-u_k)\mathbf{1}_{A_k}$.
Let $\nu_n(f) = \bbP_n f -Pf $, and we are interested in the quantity that for $M \in \bbR$, \(\bbP^* = \bbP(\sup_{f\in\cF} \nu_n(f) \ge M)\). 
Let $\varepsilon \in (0,1)$, and $\{\gamma_1,\dots,\gamma_N\}$ be a sequence of real numbers such that $\gamma_j > 0$ for all $j \in \{1,\dots,N\}$ and 
\begin{equation}\label{eqn:gamma_condition}
    \sum_{j=1}^N\gamma_j \le \frac{\varepsilon M}{8}.
\end{equation}
Then we have $ \bbP^* = \bbP(\sup_{f\in\cF} \nu_n(f) \ge M)\le \bbP_1 + \sum_{k=1}^N\bbP_{2,k} + \sum_{k=0}^{N-1}\bbP_{3,k} + \bbP_4 $,
where \(\bbP_1 = \bbP(\sup_{f\in\cF}\nu_n(u_0) > (1-\frac{\varepsilon}{4})M)\), \(\bbP_{2,k} = \bbP(\sup_{f\in\cF}\nu_n((u_k-u_{k-1})\mathbf{1}_{\substack{\cup\\ {j\ge k}} A_k }) > \gamma_k)\), \(\bbP_{3,k} = \bbP(\sup_{f\in\cF}\nu_n((f-u_k)\mathbf{1}_{A_k}) > \gamma_{k+1}), k \in \{0,\dots,N-1\}\), and \(\bbP_4 = \bbP(\sup_{f\in\cF}\nu_n((f-u_N)\mathbf{1}_{A_N}) > \frac{\varepsilon M}{8}+\gamma_N)\).
The following proof is dedicated to upper bound these quantities. Specifically, we consider a carefully selected set $\{\delta_j\}_{j=0}^N$. Let $\delta_0$ be the smallest $\delta$ such that $\log\cN_{[]}(\delta,\cF,\cL^{2}(\bbP)) \le \frac{\varepsilon}{4}\psi(M,\sigma^{2},\cF)$, i.e., $\delta_0 = \bigg(\log\cN_{[]}\bigg)^{-1}\bigg(\frac{\varepsilon}{4}\psi(M,\sigma^{2},\cF)\bigg)$.
Subsequently, let 
$s = {\varepsilon M}/{8}$
and for $j \in \{0,\dots,N-1\}$, $\delta_{j+1} = \max\left(s,\sup\{\delta\le{\delta_j}/{2}: \log\cN_{[]}(\delta,\cF,\cL^{2}(\bbP)) \ge  4\cdot\log\cN_{[]}(\delta_j,\cF,\cL^{2}(\bbP))\}\right)$.
Let $N = \min\{j: \delta_j \le s\}$. Suppose that $\log\cN_{[]}(\sigma,\cF,\cL^{2}(\bbP)) \le \frac{\varepsilon}{4}\psi(M,\sigma^{2},\cF)$,
then $\delta_0 \le \sigma$ by the monotonicity of the bracketing number. With $\{\delta_j\}_{j=0}^N$, we proceed to define $\{\gamma_j\}$ and $\{a_j\}$. For $j \in \{1,\dots,N\}$, let $\gamma_j = 4\frac{\delta_{j-1}}{{\sqrt{n}}}\left(\frac{\sum_{i\le j}\log\cN_{[]}(\delta_i,\cF,\cL^{2}(\bbP))}{\varepsilon}\right)^{1/2}$ and $a_j = \frac{8\delta_{j-1}^2}{\gamma_j}$. 

Now we are at the stage of presenting the upper bounds.
To bound $\bbP_4$, note that for any $f \in \cF$, because \((f-u_N)\mathbf{1}_{A_N} \le 0\) almost surely, \(\nu_n((f-u_N)\mathbf{1}_{A_N})\le \bbP(u_N-l_N) \le \bbE[(u_N-l_N)^2]^{1/2} \le \delta_N \le s = \frac{\varepsilon M}{8}\).
Hence, we have $  \bbP_4 = \bbP(\sup_{f\in\cF}\nu_n((f-u_N)\mathbf{1}_{A_N}) > \frac{\varepsilon M}{8}+\gamma_N) = 0$.
To bound $\bbP_3$, note that for any $f \in \cF$ and $k \in \{0,\dots,N-1\}$, \(\nu_n((f-u_k)\mathbf{1}_{A_k}) \le \bbE[(u_k - f)\mathbf{1}_{A_k}] \le \sqrt{\bbE[(u_k - l_k)^2]\bbE[\mathbf{1}_{A_k}^2]} \le \sqrt{\delta^2_k \bbP(A_k)}\). 
Moreover, due to construction of $\{A_k\}$, we have $u_k - l_k \ge a_{k+1}$ on $A_k$. Thus, due the chain of inequalities, $\bbE[(u_k-l_k)^2] \ge \bbE[(u_k-l_k)^2\mathbf{1}_{A_k}] \ge a^2_{k+1}\bbE[\mathbf{1}_{A_k}]$,
we have \(\bbP(A_k) \le \frac{ \bbE[(u_k-l_k)^2]}{a^2_{k+1}} \le \frac{\delta^2_k}{a^2_{k+1}}\).
Plugging this back, we have for all $f \in \cF$, $\nu_n((f-u_k)\mathbf{1}_{A_k}) \le \sqrt{\delta^2_k \bbP(A_k)} \le \sqrt{\delta^2_k\frac{\delta^2_k}{a^2_{k+1}}} \le \frac{\delta_k^2}{a_{k+1}}$.
By construction of $a_j$, i.e., 
$a_j = 8\delta^2_{j-1}/\gamma_j$, we have $\gamma_j \ge \delta^2_{j-1}/a_j$. 
Hence, $\nu_n((f-u_k)\mathbf{1}_{A_k}) \le \delta^2_{k}/a_{k+1} \le \gamma_{k+1}$ and $\bbP_{3,k} = \bbP(\sup_{f\in\cF}\nu_n((f-u_k)\mathbf{1}_{A_k}) > \gamma_{k+1}) = 0$.

To upper bound $\bbP_1$, we have \(\bbP_1\le \bbP(\max_{u_0 \in \cF_0} \nu_n(u_0) > (1-\frac{\varepsilon}{4})M)\le \sum_{u_0 \in \cF_0} \bbP(\nu_n(u_0)> (1-\frac{\varepsilon}{4})M)\).
By assumption for any $u_0 \in \cF_0 \subset \cF$ and $k \ge 2$, \(\bbE[|u_0-\bbE[u_0]|^k] \le \frac{k!}{2}2^{k-2}\sigma^2\).
Recall that by Bernstein's Inequality, for any random random $X$ with $\bbE[|X-\bbE[X]|^k] \le \frac{k!}{2}2^{k-2}\sigma^2$ and $X_1, \dots, X_n$ as random variables with the same distribution as $X$, we have $\bbP(\frac{1}{n}\sum_{i=1}^n (X_i-\bbE[X_i]) > t) \le \exp\left(-\frac{nt^2}{2(\sigma^2+2t)}\right) \le \exp\left(-\frac{nt^2}{4(\sigma^2+t)}\right)$.
Hence, let $\psi(M,\sigma^2,\cF) = nM^2/4(\sigma^2+M)$, we thus have \(\bbP_1 \le \cN_{[]}(\delta_0,\cF,\cL^{2}(\bbP))\exp(-\psi((1-\frac{\varepsilon}{4})M, \sigma^2, \cF))\). Moreover, by definition of $\delta_0$, we have \(\cN_{[]}(\delta_0,\cF,\cL^{2}(\bbP))\le \exp(\frac{\varepsilon}{4}\psi(M, \sigma^2, \cF))\). After straightforward calculation. we have $\bbP_1 \le \exp(-(1-\varepsilon)\psi(M, \sigma^2, \cF))$.

To upper bound $\bbP_2$, note that for any $k \in \{1,\dots,N\}$, \( \text{Var}[(u_k-u_{k-1})\mathbf{1}_{\cup_{j\ge k} A_j} ]\le \bbE[((u_k-u_{k-1})\mathbf{1}_{\cup_{j\ge k} A_j})^2] \le \bbE[(u_k-u_{k-1})^2] \le \bbE[(u_{k-1}-l_{k-1})^2] \le \delta^2_{k-1}\).
By construction of $\{A_k\}$, we have $0 \ge u_k - u_{k-1} \ge l_{k-1}-u_{k-1} \ge -a_k$ on $\bigcup_{j\ge k} A_j$ for $k=1,\dots,N$. Then by Bernstein's Inequality for random variable with bounded support and finite variance, it holds that $\bbP(\nu_n((u_k-u_{k-1})\mathbf{1}_{\bigcup_{j\ge k} A_j} ) > \gamma_k) \le \exp(-\frac{n\gamma^2_k}{2(\delta^2_{k-1}+a_k \gamma_k/3)})$. 
Note that $u_k$ is defined as $\min_{j\le k}g^u_j$ where $g^u_j \in \cF_j$. This implies that the possible number of $u_k$ is upper bounded by $\prod_{j=0}^k|\cF_j|$. Then we have $\prod_{j=0}^k|\cF_j|\prod_{j=0}^{k-1}|\cF_j| \le \exp\left(2\sum_{j=1}^k\log\cN_{[]}(\delta_j,\cF,\cL^{2}(\bbP))\right)$. This leads to
\(\bbP_{2,k}\le \exp\left(2\sum_{j=1}^k\log\cN_{[]}(\delta_j,\cF,\cL^{2}(\bbP))-\frac{n\gamma^2_k}{2(\delta^2_{k-1}+a_k \gamma_k/3)}\right) \).
By definition of $\gamma_k$ and $a_k$, we also have $\frac{n\gamma^2_k}{2(\delta^2_{k-1}+a_k \gamma_k/3)} \ge \frac{2\sum_{j \le k}\log\cN_{[]}(\delta_j,\cF,\cL^2(\bbP))}{\varepsilon}$.
Plugging this back in, we have \(\bbP_{2,k} \le \exp\left(-2\frac{1-\varepsilon}{\varepsilon}\sum_{j\le k}\log\cN_{[]}(\delta_j,\cF,\cL^2(\bbP))\right)\).
Recall that by construction $\log\cN_{[]}(\delta_0, \cF, \cL^2(\bbP)) = \varepsilon\psi(M, \sigma^2, \cF))/4$ and $\log\cN_{[]}(\delta_k, \cF, \cL^2(\bbP)) \ge 4\log\cN_{[]}(\delta_{k-1}, \cF, \cL^2(\bbP))$. Hence, $\log\cN_{[]}(\delta_k, \cF, \cL^2(\bbP)) \ge 4^{k-1}\varepsilon\psi(M, \sigma^2, \cF))$ for $k=1,\dots,N$. Then, \(\sum_{j=1}^k \log\cN_{[]}(\delta_j, \cF, \cL^2(\bbP)) \ge 4^{k-1}\varepsilon\psi(M, \sigma^2, \cF)\).
This leads to $\bbP_{2,k} \le \exp\left(-2(1-\varepsilon)4^{k-1}\psi(M, \sigma^2, \cF)\right)$, and \(\sum_{k=1}^N\bbP_{2,k} \le \sum_{k=1}^{\infty} \exp(-(1-\varepsilon)\psi(M, \sigma^2, \cF))^{2^k}\).
To continue let $\alpha = \exp(-(1-\varepsilon)\psi(M, \sigma^2, \cF)) < 1$ where $\psi(M,\sigma^2,\cF) = \frac{nM^2}{4\sigma^2+4M}$.
By the assumption that $\sigma^2 \ge 1/n$ and $M=\varepsilon \sigma/4$, we have $\psi(M, \sigma^2, \cF) = n\varepsilon^2\sigma^2/16(4+\varepsilon)\ge \varepsilon^2/16(4+\varepsilon) > 0$. Then we have \(\sum_{k=1}^N\bbP_{2,k} \le \sum_{k=1}^\infty\alpha^{2^k} \le c\cdot \exp(-(1-\varepsilon)\psi(M, \sigma^2, \cF))\),
where $c$ is a universal constant. 
With all the results above, we finally have \(\bbP^* \le \bbP_1 + \sum_{k=1}^N\bbP_{2,k} + \sum_{k=0}^{N-1}\bbP_{3,k} + \bbP_4 \le c\cdot \exp(-(1-\varepsilon)\psi(M, \sigma^2, \cF))\),
where $c$ is a different universal constant. The last step is to verify that~\eqref{eqn:gamma_condition} holds. With Lemma 3.1 in~\cite{alex1984} and $s=\varepsilon M/8$, this can be verified with straightforward calculation.
\end{proof}

\subsection{Proof of Corollary~\ref{corol:pasta-regret}}
\begin{proof}[Proof of Corollary~\ref{corol:pasta-regret}]
By Theorem~\ref{thm:main-simple}, consider any $\sigma_n$ that  satisfies~\eqref{eqn:main-simple-entropy}. For any $\alpha \ge \sigma_n^2/96$, with probability at least $1-C_1\exp(-2C_2n\alpha)$ where $C_1$ and $C_2$ are universal constants that $p_0 \in \Omega_n(\alpha)$, and \(\sup_{p \in \Omega_n(\alpha)}H^2(p,p_0) \le 4\alpha/25\). Then by Theorem~\ref{thm:regret}, we have \(\cR(\widehat{s}_{\pess,n}) \lesssim  r_{s^{\star}}\sqrt{{\alpha}/{\pi_S(s^\star)}}\). Here we need \(C_1\exp(-2C_2n\alpha) \le \delta\), 
which for sufficiently small $\delta$ is equivalent to $\alpha \gtrsim \frac{1}{n}\log\frac{1}{\delta}$.
Note that we also need $\alpha \ge \sigma_n^2/96$. Choosing $\alpha = \sigma^2_n +  \frac{1}{n}\log\frac{1}{\delta}$ satisfies these two constraints. Plugging this back into \(\cR(\widehat{s}_{\pess,n}) \lesssim  r_{s^{\star}}\sqrt{{\alpha}/{\pi_S(s^\star)}}\) concludes the proof.
\end{proof}

\subsection{Proof of Lemma \ref{lem:monotonicty}}
\begin{proof}[Proof of Lemma \ref{lem:monotonicty}]
Define $J(t) = \int_{C_{1}\alpha_{n}}^{C_{2}t}\sqrt{\log\cN_{[]}\left(u/C_{3},\cP,H\right)}\rd u;\;t \ge \alpha_{n}$. 
 By the fact that $u \mapsto \sqrt{\log\cN_{[]}\left(u/C_{3},\cP,H\right)}$ is nonnegative and nonincreasing,
    $J'(t) \ge 0$ and is nonincreasing in $t$. This suggests that $J(t)$ is concave in $t$,
    and hence, $\alpha \mapsto J(\sqrt{\alpha})/\sqrt{\alpha}$ is nonincreasing.
    For any $\alpha \ge \alpha_{n}$, we have 
    \begin{align*}
        &{1 \over \sqrt{\alpha}}\int_{C_1 \alpha}^{C_{2}\sqrt{\alpha}}\sqrt{\log\cN_{[]}\left(u/C_{3},\cP,H\right)}\rd u \le 
        {1 \over \sqrt{\alpha}}\int_{C_1 \alpha_{n}}^{C_{2}\sqrt{\alpha}}\sqrt{\log\cN_{[]}\left(u/C_{3},\cP,H\right)}\rd u \\
        &= {J(\sqrt{\alpha}) \over \sqrt{\alpha}} 
        \le {J(\sqrt{\alpha_{n}}) \over \sqrt{\alpha_{n}}} = {1 \over \sqrt{\alpha_{n}}}\int_{C_1 \alpha_{n}}^{C_{2}\sqrt{\alpha_{n}}}\sqrt{\log\cN_{[]}\left(u/C_{3},\cP,H\right)}\rd u \le C_{4}\sqrt{n\alpha_{n}} \le C_{4}\sqrt{n\alpha},
    \end{align*}
where the second inequality is due to that $\alpha \mapsto {J(\sqrt{\alpha})\over\sqrt{\alpha}}$ is nonincreasing.
 This suggests that $\alpha$ also satisfies \eqref{eqn:main-simple-entropy}.
\end{proof}

\section{Proofs of Results in Section~\ref{sec:application}}\label{sec:app-application-proofs}
\subsection{Proof of Theorem~\ref{thm:mnl-regret}}
\begin{proof}[Proof of Theorem~\ref{thm:mnl-regret}]
The MNL model has parameters $\theta \in \bbR^d$ and $x_i \in \bbR^d$ for $i \in [N]$ with the following conditional probability function 
$$p(a\mid s;\theta)=\frac{\exp(\theta^\top x_a)\,\mathbf{1}_{\{a\in s\}}+\mathbf{1}_{\{a=0\}}}{1+\sum_{j\in s}\exp(\theta^\top x_j)}.$$
We first establish that $p(a|s;\theta)$ is a Lipchitz continuous function of $\theta$ for all $s$ and $a \in s\cup\{0\}$. To this end,
\begin{equation}\label{eqn:grad-mnl}
\begin{aligned}
 \nabla_\theta p(a|s;\theta) &= \frac{\exp(\theta^{\top}x_a)x_a \cdot(1+\sum_{j \in s}\exp(\theta^{\top}x_j))- \exp(\theta^{\top}x_a)(\sum_{j \in s}\exp(\theta^{\top}x_j)x_j)}{\left(1+\sum_{j \in s}\exp(\theta^{\top}x_j)\right)^2} \\
& = \frac{\exp(\theta^{\top}x_a)}{1+\sum_{j \in s}\exp(\theta^{\top}x_j)}\cdot\frac{x_a+\sum_{j \in s}\exp(\theta^{\top}x_j)x_a-\sum_{j \in s}\exp(\theta^{\top}x_j)x_j}{1+\sum_{j \in s}\exp(\theta^{\top}x_j)} \\
& = p(a|s;\theta)\cdot\left(\frac{x_a}{1+\sum_{j \in s}\exp(\theta^{\top}x_j)}+ \sum_{j\in s}p(j|s;\theta)x_a + \sum_{j\in s}p(j|s;\theta)x_j\right).   
\end{aligned}
\end{equation}
With this we have $||\nabla_\theta \sqrt{p(a|s;\theta)}||_2   \le \frac{1}{2}\sqrt{p(a|s;\theta)}\left(||x_a||_2 + \sum_{j\in s}p(j|s;\theta)||x_a||_2 + \sum_{j\in s}p(j|s;\theta)||x_j||_2\right) \le \frac{3}{2}\sqrt{p(a|s;\theta)}C_x$.
Hence, we have proved that the function $\sqrt{p(a|s;\theta)}$ is $3C_x/2$-Lipschitz continuous with respect to $\theta$. Then for any $\theta_1, \theta_2$ and fixed $s$, their squared Hellinger distance can then be bounded as 
\begin{equation}\label{eqn:mnl-lipschitz}
\begin{aligned}
      h^2(p(\cdot|s;\theta_1), p(\cdot|s;\theta_2)) &= \frac{1}{2}\sum_{a\in s}\left(\sqrt{p(a|s;\theta_1)}-\sqrt{p(a|s;\theta_2)}\right)^2 \\
    & \le \frac{1}{2}\sum_{a\in s}\left(\frac{3C_x\sqrt{p(a|s;\theta)}}{2}||\theta_1-\theta_2||_2\right)^2 \le \Delta\cdot||\theta_1-\theta_2||^2_2,
\end{aligned}
\end{equation}
where $\Delta = 9C^2_x/8$. This implies that $H^2(\theta_1,\theta_2) \le \Delta||\theta_1 -\theta_2||_2^2$.
Then if $\theta \in \Theta$ where $\Theta$ is bounded, we have $\cN_{[]}(u, \cP, H) \le \cN_{[]}(u/\Delta,\Theta,||\cdot||) \simeq \left(\frac{\Delta}{u}\right)^d$.
With this, the left hand side of the entropy integral in~\eqref{eqn:main-simple-entropy} can be upper bounded by
\begin{align*} \int_{C_1\alpha_n}^{\sqrt{\alpha_n}}\sqrt{\log\cN_{[]}\left({u \over 2\sqrt{2}e},\cP,H\right)}\rd u 
&\le \sqrt{d}\int_{0}^{\sqrt{\alpha_n}}\sqrt{\log\left(\frac{\Delta\cdot2\sqrt{2}e}{u}\right)}\rd u \lesssim \sqrt{\alpha_n d\log\frac{\Delta\cdot2\sqrt{2}e}{\sigma_n}},
\end{align*}
where the last inequality follows from Lemma~\ref{lem:int} provided that $\sqrt{\alpha_n} \le \Delta \cdot 2\sqrt{2}$.
Hence, to solve~\eqref{eqn:main-simple-entropy}, it is sufficient to solve the equation $\sqrt{\alpha_n d\log\frac{\Delta\cdot2\sqrt{2}e}{\sqrt{\alpha_n}}} \lesssim C_2\sqrt{n}\alpha_n$, 
which is equivalent to $\alpha_n \gtrsim \frac{d}{n}\log\frac{C_x}{\alpha_n}$.
Choosing $\alpha_n = \frac{d}{n}\log\frac{n C_x}{d}$, and under the assumption that $n \gtrsim d\log(n/d)$, the right hand side of the above inequality reads as 
\begin{equation}\label{eqn:choose-sigma_n}
   \frac{d}{n}\log\left(\frac{n}{d}\frac{C_x}{\log(n/d)}\right) \lesssim \frac{d}{n}\log\frac{n C_x}{d} = \alpha_n. 
\end{equation}
Thus, the above choice of $\alpha_n$ satisfies~\eqref{eqn:main-simple-entropy}. Moreover, $\sqrt{\alpha_n} \lesssim C_{x}^{2}$ is satisfied provided that $n \gtrsim d\log(n/d)$.
Then applying Corollary~\ref{corol:pasta-regret} concludes the proof. 
\end{proof}

\subsection{Proof of Theorem~\ref{thm:lcl-regret}}
\begin{proof}[Proof of Theorem~\ref{thm:lcl-regret}]
To avoid notation overload, only in this proof we let $g(s,a;\theta_k)$ for $\theta_k \in \bbR^d$ and $k \in [K]$ denote the conditional probability function of the MNL model $g(s,a;\theta_k) = \frac{\exp(\theta_k^{\top}x_a)}{1+\sum_{j \in s}\exp(\theta_k^{\top}x_j)}$,
and let $p(a|s;\theta)$ be the conditional probability function of the LCL model as defined in~\eqref{eqn:lcl_ass} where $\theta = [\theta_1,\dots,\theta_K;\lambda_1,\dots,\lambda_K]$. Thus we have $p(a|s;\theta) = \sum_{k=1}^K \lambda_k\cdot g(a,s;\theta_k)$.
First note that, under Assumption~\ref{asm:bdd}, $g(a,s;\theta)$ is uniformly lower bounded for all $s$ and $a \in s\cup\{0\}$:
\begin{align}\label{eqn:lb_mnl}
    g(a,s;\theta) = \frac{\exp(\theta^{\top}x_a)}{1+\sum_{j \in s}\exp(\theta^{\top}x_j)} \ge \frac{\exp(\theta^{\top}x_a)}{1+|s|\exp(||\theta||_2 C_x)} \ge \frac{\exp(-C_\theta C_x)}{1+N\exp(C_\theta C_x)} := \Delta > 0
\end{align}
With the same calculations in~\eqref{eqn:grad-mnl}, we have 
\begin{align*}
    &\nabla_{\theta_k}p(a|s;\theta) = \lambda_k \cdot g(a,s;\theta_k)\cdot\left(\frac{x_a}{1+\sum_{j \in s}\exp(\theta_k^{\top}x_j)}+ \sum_{j\in s}g(j,s;\theta_i)x_a + \sum_{j\in s}g(j,s;\theta_k)x_j\right);\\
    &\nabla_{\lambda_k}p(a|s;\theta) = g(a,s;\theta_k).
\end{align*}
Then we have 
\begin{align*}
    &||\nabla_{\theta}\sqrt{p(a|s;\theta)}||_2 = \frac{1}{2}\sqrt{p(a|s;\theta)}||(\nabla_{\theta}p(a|s;\theta))/ p(a|s;\theta)||_2\\
    &\le \frac{1}{2}\sqrt{p(a|s;\theta)}\left(\sum_{i=1}^K ||\nabla_{\theta_i}p(a|s;\theta)/ p(a|s;\theta)||_2 + ||\nabla_{\lambda}p(a|s;\theta)/ p(a|s;\theta) ||_2\right)\\
    &\le \frac{1}{2}\sqrt{p(a|s;\theta)}\left(3KC_x + \frac{\sum_{i=1}^K g(a,s;\theta_i)}{\sum_{i=1}^K \lambda_i\cdot g(a,s;\theta_i)}\right).
\end{align*}
To proceed, with~\eqref{eqn:lb_mnl}, we have $\frac{\sum_{i=1}^K g(a,s;\theta_i)}{\sum_{i=1}^K \lambda_i\cdot g(a,s;\theta_i)} \le \frac{\sum_{i=1}^K g(a,s;\theta_i)}{\Delta\sum_{i=1}^K\lambda_i} \le K/\Delta$. 
Thus, the density function of LCL model is $(3KC_x + K/\Delta)$-Lipschitz continuous. With the same computation in~\eqref{eqn:mnl-lipschitz}, let $C_{LCL} = \frac{K^2(3C_x+1/\Delta)^2}{2}$, it holds that $\cN_{[]}(u, \cP, H) \simeq \left(u/C_{LCL}\right)^{K(d+1)-1}$. Here we have $K(d+1)-1 = Kd + K-1$ as the effective dimension as there are only $K-1$ free parameters for $\lambda \in \Delta^{K-1}$ in the $K-1$ simplex.
With this, the left hand side of the entropy integral in~\eqref{eqn:main-simple-entropy} can be upper bounded by $\int_{C_1\alpha_n}^{\sqrt{\alpha_n}}\sqrt{\log\cN_{[]}\left({u \over 2\sqrt{2}e},\cP,H\right)}\rd u \lesssim \sqrt{K(d+1)-1}\int_{0}^{\sqrt{\alpha_n}}\sqrt{\log\left(\frac{C_{LCL}\cdot2\sqrt{2}e}{u}\right)}\rd u \lesssim \sqrt{\alpha_n(Kd+K-1)\log\frac{C_{LCL}\cdot2\sqrt{2}e}{\sigma_n}}$. 
Hence, to solve~\eqref{eqn:main-simple-entropy}, it is sufficient to solve the equation 
$$\sqrt{\alpha_n(Kd+K-1)\log\frac{C_{LCL}\cdot2\sqrt{2}e}{\sigma_n}} \lesssim \sqrt{n}\alpha_n,$$
which is equivalent to $\alpha_n \gtrsim \frac{(Kd+K-1)}{n}\log\frac{C_{LCL}}{\alpha_n}$. 
Using the same technique in \eqref{eqn:choose-sigma_n}, we have $\alpha_n = \frac{(Kd+K-1)}{n}\log\frac{C_{\cP}n}{(Kd+K-1)}$ satisfies the above inequality. Then applying Corollary~\ref{corol:pasta-regret} concludes the proof. 
\end{proof}

\subsection{Proof of Theorem~\ref{thm:nl-regret}}
\begin{proof}[Proof of Theorem~\ref{thm:nl-regret}]
Let $p(a\in s_j |s;\theta)$ be the conditional choice probability function as defined in~\eqref{eqn:nl_ass} where $\theta = [\tilde{\theta},\lambda]$. Note that $p(a\in s_j |s;\theta)$  can be written as: for $a \in s_j$, 
$$p(a\in s_j |s;\theta) = g_A(a,s_j;\tilde{\theta},\lambda) \cdot g_B(s_j,s;\tilde{\theta},\lambda),$$
where $g_A(a,s_j;\tilde{\theta},\lambda)=\frac{\exp(\tilde{\theta}^{\top}x_a/\lambda_j)}{\sum_{i \in S_j}\exp(\tilde{\theta}^\top x_i/\lambda_j)}$ and $g_B(s_j,s;\tilde{\theta},\lambda) = \frac{\left(\sum_{i \in S_j}\exp(\tilde{\theta}^\top x_i/\lambda_j)\right)^{\lambda_j}}{1+\sum_{l=1}^K\left(\sum_{i \in S_l}\exp(\tilde{\theta}^\top x_i/\lambda_l)\right)^{\lambda_l}}$. 

For $j\in[K]$, let $V_j = \sum_{i \in S_j}\exp(\tilde{\theta}^{\top}x_i/\lambda_j)$ and $\kappa_{i,j}=\exp(\tilde{\theta}^{\top}x_i/\lambda_j)$. 
With straightforward calculation, it can be shown that 
\begin{align*}
    &\nabla_{\tilde{\theta}} g_A(a,s_j;\tilde{\theta},\lambda) = g_A(a,s_j;\tilde{\theta},\lambda)\frac{x_a-\sum_{i\in S_j}\kappa_{i,j}x_i}{\lambda_j v_j};\\
    &\nabla_{\tilde{\theta}} g_B(s_j,s;\tilde{\theta},\lambda) = g_B(s_j,s;\tilde{\theta},\lambda)
\frac{(\sum_{k\in[K]}v_k^{\lambda_k})\frac{\sum_{i \in S_j}\kappa_{i,j}x_i}{v_j} - \sum_{k\in[K]}v_k^{\lambda_k}\frac{\sum_{i \in S_k}\kappa_{i,k}x_i}{v_k}}{\sum_{k\in[K]}v_k^{\lambda_k}}.
\end{align*}
Moreover, with the above inequalities and the given assumption that $1/\lambda_j \le C_\lambda$ for $j \in [K]$, we have $||\nabla_{\tilde{\theta}} g_A(a,s_j;\tilde{\theta},\lambda)|| \le g_A(a,s_j;\tilde{\theta},\lambda)\frac{2C_x}{\lambda_j}$ and $||\nabla_{\tilde{\theta}} g_B(s_j,s;\tilde{\theta},\lambda)|| \le 2 g_B(s_j,s;\tilde{\theta},\lambda)C_x$.
Hence, through triangular inequalities of norms, we have 
\begin{align*}
    &||\nabla_{\tilde{\theta}} p(a\in s_j |s;\theta)|| = ||\nabla_{\tilde{\theta}} g_A(a,s_j;\tilde{\theta},\lambda) g_B(j;\tilde{\theta},\lambda) +g_A(a,s_j;\tilde{\theta},\lambda) \nabla_{\tilde{\theta}} g_B(s_j,s;\tilde{\theta},\lambda)||;\\
    &\le 2g_A(a,s_j;\tilde{\theta},\lambda) g_B(s_j,s;\tilde{\theta},\lambda)\max(2C_x/\lambda_j,\lambda) \le 2p(a\in s_j |s;\theta)C_x\max(2C_\lambda, 1).
\end{align*}

Regarding $\lambda$, it is useful to use the fact that $\nabla_{\lambda_i}\kappa_{a,j} = -\tilde{\theta}^{\top}x_a\cdot\exp(\tilde{\theta}^{\top}x_a/\lambda_j)/\lambda^2_j$
for $i = j, a \in S_j$ and $\nabla_{\lambda_i}\kappa_{a,j}=0$ otherwise. With these, we have 
\begin{align*}
    &\nabla_{\lambda_i} g_A(a,s_j;\tilde{\theta},\lambda)=g_A(a,s_j;\tilde{\theta},\lambda)\,\frac{\sum_{t\in S_j}\kappa_{t,j}\,\tilde{\theta}^{\top}x_t-\tilde{\theta}^{\top}x_a}{\lambda_j^{2}v_j}\,\mathbf{1}_{\{i=j\}}; \\
    &\nabla_{\lambda_i} g_B(s_j,s;\tilde{\theta},\lambda)=g_B(s_j,s;\tilde{\theta},\lambda)\,\frac{-\nabla_{\lambda_i}v_i^{\lambda_i}}{\sum_{k\in[K]}v_k^{\lambda_k}}\,\mathbf{1}_{\{i\neq j\}}+\frac{\sum_{k\in[K]\setminus\{j\}}v_k^{\lambda_k}}{\sum_{k\in[K]}v_k^{\lambda_k}}\frac{-\nabla_{\lambda_j}v_j^{\lambda_j}}{\sum_{k\in[K]}v_k^{\lambda_k}}\,\mathbf{1}_{\{i=j\}}.
\end{align*}

To continue, with straightforward calculation, for $j \in [K]$, $\nabla_{\lambda_j}v_j^{\lambda_j} = v_j^{\lambda_j}\left(\frac{\sum_{i \in S_j} -\tilde{\theta}^{\top}x_i\kappa_{i,j}}{\lambda_j v_j} + \ln v_j\right)$.
Then we have $\sum_{i \in [K]}||\nabla_{\lambda_i} g_A(a,s_j;\tilde{\theta},\lambda)|| \le g_A(a,s_j;\tilde{\theta},\lambda)\frac{2C_xC_{\theta}}{\lambda^2_j}$ where $||\nabla_{\lambda_i} g_B(s_j,s;\tilde{\theta},\lambda)|| \le  g_B(s_j,s;\tilde{\theta},\lambda)g_B(s_i,s;\tilde{\theta},\lambda)\frac{C_xC_\theta}{\lambda_j}$ if $i \neq j$ and $||\nabla_{\lambda_i} g_B(s_j,s;\tilde{\theta},\lambda)|| \le  g_B(s_j,s;\tilde{\theta},\lambda)\frac{C_xC_{\theta}}{\lambda_j}$ if $i = j$.
Hence, $||\nabla_{\lambda_i} g_B(s_j,s;\tilde{\theta},\lambda)|| \le  g_B(s_j,s;\tilde{\theta},\lambda)C_xC_\theta C_\lambda$ for all $i,j \in [K]$. With these, we have 
\begin{align*}
    &||\nabla_{\lambda} p(a\in s_j |s;\theta)||_2 = ||\nabla_\lambda g_A(a,s_j;\tilde{\theta},\lambda) g_B(s_j,s;\tilde{\theta},\lambda) +g_A(a,s_j;\tilde{\theta},\lambda) \nabla_\lambda g_B(s_j,s;\tilde{\theta},\lambda)||\\ 
    &\le 2p(a\in s_j |s;\theta)C_xC_{\theta}\max(2C^2_\lambda, C_\lambda).
\end{align*}
The above results lead to 
\begin{align*}
    &||\nabla_{\tilde{\theta},\lambda} \sqrt{p(a\in s_j |s;\theta)}|| \le\frac{\sqrt{p(a\in s_j |s;\theta)}}{2} \frac{(||\nabla_{\theta} \sqrt{p(a\in s_j |s;\theta)}||+||\nabla_{\lambda} \sqrt{p(a\in s_j |s;\theta)}||)}{p(a\in s_j |s;\theta)}\\ 
    &\le 2\sqrt{p(a\in s_j |s;\theta)}C_x(1+C_\theta C_\lambda)\max(2C_\lambda,1).
\end{align*}
Let \(C_{NL}=(2C_x(1+C_\theta C_\lambda)\max(2C_\lambda,1))^2/2\) and
we have for any $s=\{s_j\}^K_{j=1}$, and any pairs of $\theta_1 = (\tilde{\theta}_1,\lambda_1), \theta_2=(\tilde{\theta}_2,\lambda_2)$ that $h^2(p(\cdot |s;\theta_1),p(\cdot |s;\theta_2)) = \frac{1}{2}\sum_{a\in s_j, s_j \in s}\left(\sqrt{p(a\in s_j |s;\theta_1)}-\sqrt{p(a\in s_j |s;\theta_2)}\right)^2 \le C_{NL}||\theta_1-\theta_2||^2_2$,
where $\theta_1, \theta_2 \in \bbR^{K+d}$. This implies that $H^2(\theta_1,\theta_2) \le C_{NL}||\theta_1-\theta_2||^2_2$.
Then with the same steps in the proofs for the MNL and LCL models, we can choose $\alpha_n = \frac{K+d}{n}\log\frac{C_{\cP}n}{K+1}$ so that it satisfies the entropy condition~\eqref{eqn:main-simple-entropy}. Applying Corollary~\ref{corol:pasta-regret} concludes the proof. 
\end{proof}

\subsection{Proof of Theorem~\ref{thm:minimax_d}}
\begin{proof}[Proof of Theorem~\ref{thm:minimax_d}]\label{sec:proof-minimax-theorem}
    We present the proof for the case where $d \leq N$. The proof for the case where $d > N$ is identical. Let the four-element tuple $[\mathbf{x}=\{x_j\}_{j \in [N]}, r, \theta_0, \pi_S]$ define an instance of the offline assortment optimization problem under the MNL model where $\theta_0$ is the true model parameter. Moreover, let $I=[\mathbf{x},r]$ and $P=[\theta_0,\pi_S]$. Let $\cI \ni I$ and $\cP \ni P$ be the set of possible offline assortment optimization problems under consideration.
    The goal is to lower bound $\mathcal{R} = \inf_{\widehat s} \sup_{I\in \cI,P \in \cP}\bbE\left[|\cV(s^\star)-\cV(\widehat s)|\right]$,
    where $s^\star = s^\star(I,P)$ is the optimal assortment for instance $[I,P]$, and $\cV(s)=\cV(s;I,P)=\sum_{j \in s}r(s,j)p(j|s;\theta_0)$ is the true expected revenue function of the instance $[I,P]$. 
    To establish a lower bound for $\mathcal{R}$, we first consider $d$ instances of $I$. For the case where $d > N$, we consider $N$ instances. All the other proof steps are identical.  Specifically, let the $i$-th instances $I_i$ be defined as follows: (i) $r(s,i)=r_i=1$ and $r(s,j)=r_j=0$ for $j\neq i$; (ii) For $i \in \{1,\dots, d\}$, we set $x_i \in \bbR^d$ to have all zero elements except that the $i$-th element is set to $1/\varepsilon$ where $\varepsilon \in \bbR$ is to be chosen later , e.g., $x_1=[1/\varepsilon, 0, \dots, 0]$, $x_2=[0,1/\varepsilon, 0, \dots, 0]$, etc. For $i \in \{d+1,\dots,N\}$, $x_i$ can be any vector in $\bbR^d$ as long as $x_i$ is distinct to $x_j$ for $j < i$.

Clearly, the optimal assortment for $I_i$ is always the singleton set $\{i\}$, regardless of the value of $\theta_0$. We also have \(\cR \ge \cR_1 = \inf_{\widehat s} \sup_{P \in \cP} \bbE\left[\sum_{i=1}^d\frac{1}{d}|\cV(s^\star;I_i,P)-\cV(\widehat s;I_i,P)|\right]\)
because the regret average over $d$ instances is always not greater than the worst instance.  
Next we continue to lower bound $\cR_1$ by specifying a $\pi_S$ and selecting a subset of MNL models $\widetilde{\Theta} \subset \Theta$ for the worst case regret. Consider the assortment set $\bbS:=\{\{1\},\dots,\{d\}\}$. We define $\pi_S$ as $\pi_S(s)=1/d$ if $s \in \bbS$ otherwise $\pi_S(s)=0$. In other words, the probability mass is uniformly distributed on all assortments in $\bbS$. Let $\widetilde{\Theta} := \{\delta\cdot v: v \in \{-1,1\}^d \}$ where $\delta \in \bbR$ is to be chosen later. Then \(\cR \ge \cR_1 \ge \cR_2 = \inf_{\widehat s} \sup_{\theta \in \widetilde{\Theta}} \bbE\left[\sum_{i=1}^d\frac{1}{d}|\cV(s^\star;\theta)-\cV(\widehat s;\theta)|\right]\).

In the subsequent proof, for $v\in\{-1,1\}^d$, we use $P_{v}$ to denote the distribution $P_{v}(s,a)=\pi_S(s)p(a|s;\delta\cdot v)$ and $\theta_v$ for $\theta_v = \delta\cdot v$. We write $v \sim_i v'$ if $v$ and $v'$ differ in only the $i$-th coordinate. we use $v^+ (\theta^+)$ and $v^-(\theta^-)$ to denote a pair of $v$ where only $1$ element is different. Note that if $\delta$ is small enough, then $|\theta^T x_i| \leq 1$ for all $i$. To satisfy this requirement, one sufficient condition is $|\theta^T x_i| \leq ||\theta||\cdot ||x_i|| \leq \frac{\delta}{\varepsilon} \leq 1 => \delta \leq \varepsilon$.

For any $v$ and $v'$ such that $v \sim_i v'$,  we have 
\begin{equation}\label{eqn:lb-by-e}
\begin{aligned}
    &|\cV(s^\star;I_i,P_v)-\cV(s^\star;I_i,P_{v'})| = \frac{|\exp(x_i^T\delta v)-\exp(x_i^T\delta v')|}{\big(1+\exp(x_i^T\delta v)\big)\big(1+\exp(x_i^T\delta v')\big)}\\
    &\ge  \frac{\exp(\delta/\varepsilon)-\exp(-\delta/\varepsilon)}{2+\exp(\delta/\varepsilon)+\exp(-\delta/\varepsilon)} \geq \frac{\exp(\delta/\varepsilon)-1}{3+e} \geq \frac{\delta}{(3+e)\varepsilon},
\end{aligned} 
\end{equation}
where for the second to last inequality in~\eqref{eqn:lb-by-e} we use the fact that $|\theta^T x_i| \leq 1$ hence $\exp(-\delta/\varepsilon) \leq 1$ and $\exp(\delta/\varepsilon) \leq e$; for the last inequality in~\eqref{eqn:lb-by-e} we use the fact that $\exp(x)-1 \geq x$ for $0\leq x \leq 1$. 

Thus, we have proved that, for any $v$ and $v'$ such that $v \sim_i v'$, $|\cV(s^\star;I_i,P_v)-\cV(s^\star;I_i,P_{v'})| \ge \delta/(3+e)\varepsilon$.  
Let $n$ denote the number of samples and $P^n$ denote the product distribution of $P$. Then, from Assouad's Lemma~\citep{yu1997assouad}, we have
\begin{equation}\label{eqn:r2-lb}
\cR_2 \ge \frac{\delta}{(3+e)\varepsilon} \min_{(v,v')}\{||P^n_v\wedge P^n_{v'}||\} ,
\end{equation}
where $v,v' \in \{-1,1\}^d$ and only differ in 1 coordinate. We further have \(\min_{v,v'}\{\|P^n_v\wedge P^n_{v'}\|\} \ge \min_{v,v'}\{\frac{1}{2}\exp(-\KL(P^n_v \| P^n_{v'})) =\frac{1}{2}\min_{v,v'}\{\exp(-n\bbE_{s\sim \pi_s}[\KL(p(\cdot|s,\theta_v)\|p(\cdot|s,\theta_{v'}))])\}\}\) where the first inequality is due to fact that $\|p\wedge p'\| \geq 1/2\exp(-\KL(p||p'))$ and the equality follows from the chain rule of KL divergence. With Lemma~\ref{lemma:KL_singleton}, for any $v \sim_j v'$, when $s=\{j\}$, we have \(\KL(p(\cdot|s;\theta_v)||p(\cdot|s;\theta_{v'})) \leq \frac{4\delta^2}{\varepsilon^2}\); 
when $s=\{i\}$ and $i\neq j$, we have \(\KL(p(\cdot|s;\theta_v)||p(\cdot|s;\theta_{v'})) \leq (0\cdot\delta-0\cdot\delta)^2 = 0\).
Then for any pair of $\theta_v$ and $\theta_{v'}$ such that $v$ and $v'$ only differ in one coordinate, we have \(\bbE_{s\sim \pi_s}[\KL(p(\cdot|s,\theta_v)||p(\cdot|s,\theta_{v'}))]\leq \frac{1}{d}\frac{4\delta^2}{\varepsilon^2}\).
With these results and~\eqref{eqn:r2-lb}, we have $\cR_2 \ge \frac{\delta}{44\varepsilon}\exp(-\frac{4n\delta^2}{d\varepsilon^2})$.
Choose $\varepsilon = \sqrt{n}$, $\delta=\sqrt{d}$, we finally have \(\cR \ge \cR_2 \ge c\sqrt{d/n}\)
where $c$ is an absolute constant. Note that we need $\delta \leq \varepsilon$ which is equivalent to $n\geq d$. 
In the case where $d > N$, with the same proof, we can have $\mathcal{R}(\cP, \mathcal{I})\geq c\cdot \sqrt{{N}/{n}}$. Hence, we have established that $\cR\geq c\cdot \sqrt{\min(N,d)/{n}}$, when $n \geq \min(N,d)$. This concludes the proof.
\end{proof}

\section{Technical Lemmas}
\begin{lemma}
    \label{lem:int}
    Suppose $A > 0$ and $C > 1$. Then for $0 \le a \le e^{-{C \over 2(C-1)}}A$, we have $\int_{0}^{a}\sqrt{\log{A \over u}}\rd u \le Ca\sqrt{\log{A \over a}}$.
    As a special case, if $C = 2$, then for $0 \le a \le e^{-1}A$, we have $\int_{0}^{a}\sqrt{\log{A \over u}}\rd u \le 2a\sqrt{\log{A \over a}}$.
\end{lemma}
\begin{proof}[Proof of Lemma \ref{lem:int}]
    Define 
$$f(a) := \begin{cases}
    Ca\sqrt{\log{A \over a}} - \int_{0}^{a}\sqrt{\log{A \over u}}\rd u, &a > 0;\\
    0; &a = 0.
\end{cases}$$
Then $f$ is continuous at $0$. Moreover, for $a > 0$, we have 
$$f'(a) = (C-1)\sqrt{\log{A \over a}} - {C \over 2\sqrt{\log{A/a}}},$$ 
which is nonnegative if and only if 
$$\log(A/a) \ge {C \over 2(C-1)} \Leftrightarrow a \le e^{-{C \over 2(C-1)}}.$$ 
As $a \to 0^+$, we further have 
\begin{align*}
    &{f(a) \over a} = C\sqrt{\log{A \over a}} - {1 \over a}\int_{0}^a\sqrt{\log{A \over u}}\rd u = (C-1)\sqrt{\log{A \over a}} - {1 \over 2a}\int_{0}^a{1 \over \sqrt{\log{A \over u}}}\rd u\\
    &\ge (C-1)\sqrt{\log{A \over a}} - {1 \over 2\sqrt{\log{A/a}}} \ge (C-1)\sqrt{\log{A \over a}} - {1 \over 2\sqrt{\log{A \over a}}}.
\end{align*}
This implies that $\liminf_{a \to 0^+}{f(a) \over a} \ge +\infty$. Therefore, for any $0 \le a \le e^{-{C \over 2(C-1)}}A$, we have $f(a) \ge  0$, which concludes the proof.
\end{proof}

\begin{lemma}\label{lemma:KL_singleton}
Let $p(a|s;\theta)$ be an MNL model as defined in~\eqref{eqn:mnl_ass}. For any singleton assortment $s=\{i\}$, it holds that \(\KL(p(\cdot|s;\theta_v)||p(\cdot|s;\theta_{v'})) \leq (x_i^T\theta_v - x_i^T\theta_{v'})^2\) where $v,v'$ are defined in Appendix~\ref{sec:proof-minimax-theorem}.
\end{lemma}
\begin{proof}[Proof of Lemma~\ref{lemma:KL_singleton}]
For any $v,v' \in \{1,-1\}^d$ that differ only in 1 coordinate, let $p_v = \frac{1}{1+\exp(-x_i^T\theta_v)}$ and $p_{v'} = \frac{1}{1+\exp(-x_i^T\theta_{v'})}$ where $\theta_v = \delta\cdot v$ for some $\delta > 0$. Then we have 
\[\KL(p(\cdot|s,\theta_v)\|p(\cdot|s,\theta_{v'}))= (p_v-p_{v'})\log(\frac{p_v}{1-p_v}\frac{1-p_{v'}}{p_{v'}}).\]

With the fact that $p_v/(1-p_v)=\exp(x_i^T\theta_v)$ and $(1-p_{v'})/p_{v'}=\exp(-x_i^T\theta_{v'})$, we have $$\KL(p(\cdot|s,\theta_v)||p(\cdot|s,\theta_{v'})) \le \left(\frac{\exp(-x_i^T\theta_{v'})-\exp(-x_i^T\theta_v)}{\exp(-x_i^T\theta_{v'})}\right) (x_i^T\theta_v-x_i^T\theta_{v'}).$$

To continue, assume without loss of generality assume that $x_i^T\theta_v \geq x_i^T\theta_{v'}$, then 
\begin{align*}
    &\KL(p(\cdot|s,\theta_v)||p(\cdot|s,\theta_{v'})) \leq \frac{\exp(-x_i^T\theta_{v'})-\exp(-x_i^T\theta_v)}{\exp(-x_i^T\theta_{v'})}(x_i^T\theta_v-x_i^T\theta_{v'})\\ 
    &= (1-\exp(-x_i^T\theta_v + x_i^T\theta_{v'}))(x_i^T\theta_v-x_i^T\theta_{v'})\\
    &\leq (1 - (1+(-x_i^T\theta_v + x_i^T\theta_{v'})))(x_i^T\theta_v-x_i^T\theta_{v'}) = (x_i^T\theta_v-x_i^T\theta_{v'})^2,
\end{align*}
where in the last inequality we use the fact that $\exp(x) \geq 1+x$.  
\end{proof}

\begin{lemma}\label{lemma:l1_distance_ub}
Let $C_{s^\star} := 1/\pi_S(s^\star)$ and $r_{s^{\star}} := \max_{j \in s^\star}r(s^\star,j)$, then the following inequality holds: for any $p_1, p_2 \in \cP$, $\left|\cV(s^\star;p_1)-\cV(s^\star;p_2)\right|\le r_{s^\star}\sqrt{C_{s^\star}\bbE_S[\|p_1(\cdot|S)-p_2(\cdot|S)\|^2_1]}$. 
\end{lemma}
\begin{proof}[Proof of Lemma~\ref{lemma:l1_distance_ub}] With change of measure and by definition of $\cV(s;p)$, we have $\cV(s^\star;p_1)-\cV(s^\star;p_2) = \bbE_S\left[f(S)g(S,p_1,p_2)\right]$,
where $f(S)=\mathbb{I}(S=s^\star)/\pi_S(S)$ and $g(S,p_1,p_2)=\mathbb{I}(S=s^\star)\sum_{j\in S\cup\{0\}}r(S,j)\left(p_1(j|S)-p_2(j|S)\right)$. Thus, we have 
\begin{equation}\label{eqn:V-abs-ub}
    \begin{aligned}
        &\left|\cV(s^\star;p_1)-\cV(s^\star;p_2)\right| \le r_{s^\star}\bbE_S\left[f(S)\|p_1(\cdot|S)-p_2(\cdot|S)\|_1\right],
    \end{aligned}
\end{equation}
where $\|\cdot\|_1$ is the $L_1$ norm. By H\"older's inequality, we also have $\bbE_S[f(S)\|p_1(\cdot|S)-p_2(\cdot|S)\|_1] \le \sqrt{\bbE_S[f^2(S)]\bbE_S[\|p_1(\cdot|S)-p_2(\cdot|S)\|^2_1]}$.
Moreover, we have 
$$\bbE_S[f^2(S)] = \bbE_S[\mathbb{I}(S=s^\star)/\pi^2_S(S)] = 1/\pi_S(s^\star),$$ 
as $\frac{\mathbb{I}(S=s^\star)}{\pi_S(S)}$ has the value zero everywhere except at $s=s^\star$. 
Combining this with the upper bound for $\left|\cV(s^\star;p_1)-\cV(s^\star;p_2)\right|$ in~\eqref{eqn:V-abs-ub}, we have 
$\left|\cV(s^\star;p_1)-\cV(s^\star;p_2)\right|\le r_{s^\star}\sqrt{C_{s^\star}\bbE_S[\|p_1(\cdot|S)-p_2(\cdot|S)\|^2_1]}$.
\end{proof}

\begin{lemma}\label{lemma:l1_H_ub}
For any $p_1,p_2 \in \cP$, $\mathbb{E}_{S}\left[\left(\sum_{j\in S\cup\{0\}}\big|p_1(j|S)-p_2(j|S)\big|\right)^2\right] \leq 8H^2(p_1, p_2)$.
\end{lemma}

\begin{proof}[Proof of Lemma~\ref{lemma:l1_H_ub}]
We use the fact that the total variation distance is less than a constant multiple of the Hellinger distance~\citep{tsybakov_nonparametric}, i.e., for any $s \in \bbS$, $\frac{1}{2}\sum_{j\in s\cup\{0\}}\big|p_1(j|s)-p_2(j|s)\big| \leq \sqrt{2} h(p_1(\cdot|s),p_2(\cdot|s))$,
where $h$ is the Hellinger distance. This implies that for any $s \in \bbS$, $\left(\sum_{j\in s\cup\{0\}}\big|p_1(j|s)-p_2(j|s)\big|\right)^2 \leq 8 h^2(p_1(\cdot|s),p_2(\cdot|s))$.
Taking expectation with respect to $S$ for both sides of this inequality concludes the proof. 
\end{proof}

\begin{lemma}[Bernstein's Exponential Condition] \label{lem:exp}
		Consider a correctly specified family $ \cP = \{ p_{\theta}(y|x): \theta \in \Theta \}$, and the corresponding family of log-likelihood ratios $\cF := \big\{ \log{p_{\theta^{*}}(Y|X) \over p_{\theta}(Y|X)}: \theta \in \Theta \big\}$.
		Define 
		\[ \kappa(u) := \begin{cases}
			{2\left( e^{|u|/2} - 1 - |u|/2 \right) \over (1-e^{-u/2})^{2}}, & u \neq 0;\\
			1, & u = 0.
		\end{cases} \]
		Assume that for some $\tau > 0$, we have $\log{p_{\theta^{*}}(Y|X) \over p_{\theta}(Y|X)} \le \tau$. Then for any $f_{\theta}\in \cF$, we have $bbE\left( e^{|f_{\theta}|/2} - 1 - {|f_{\theta}| \over 2}  \right) \le \kappa(\tau)H^{2}(\theta,\theta^{*})$. 
        By Taylor's expansion $e^{u} - 1 - u = \sum_{k=2}^{\infty}{u^{k} \over k!}$, it follows that, for any $f_{\theta} \in \cF$ and $k \ge 2$, we have $ \bbE |f_{\theta}|^{k} \le {k! \over 2}2^{k-2} \times 8\kappa(\tau)H^{2}(\theta, \theta^{*})$. In particular, $\bbE(f_{\theta}^{2}) \le 8\kappa(\tau)H^{2}(\theta,\theta^{*})$.
\end{lemma}

\begin{proof}[Proof of Lemma \ref{lem:exp}]
		Consider $f_{p } = f_{p }(Y|X) = \log{p_0(Y|X) \over p(Y|X)} \in \cF$. In particular, $ f_{p } \le \tau $. 
		On the event $\{p_0(Y|X) > 0\}$, $f_{p }$ is finite.
		By Lemma \ref{lem:kappa_exp}, $\kappa(f_{p }) \le \kappa(\tau)$, that is, 
        \begin{align*}
            &e^{|f_{p }|/2} - 1 - {|f_{p }| \over 2} \le \kappa(\tau) \cdot {1 \over 2}(1-e^{-f_{p }/2})^{2}\\ 
            &= \kappa(\tau) \cdot {1 \over 2}\left\{ 1 - \exp\left[ -{1 \over 2}\log{p_0(Y|X) \over p(Y|X)} \right] \right\}^{2}\\
            &= \kappa(\tau) \cdot {1 \over 2p_0(Y|X)}\left( \sqrt{p_0(Y|X)} - \sqrt{p(Y|X)} \right)^{2}.
        \end{align*}
		Therefore, $\bbE\left( e^{|f_{p }|/2} - 1 - {|f_{p }| \over 2}  \right)\le \kappa(\tau) \bbE\left\{{1 \over 2}\int_{\cY}\left( \sqrt{p(y|X)} - \sqrt{p_0(y|X)} \right)^{2} \rd y\right\} = \kappa(\tau)H^2(p,p_0)$.
\end{proof}

\begin{lemma}\label{lem:kappa_exp}
    Consider $\kappa(u)$ in Lemma~\ref{lem:exp}. Then $\kappa(\cdot)$ is a non-decreasing and continuous function.
\end{lemma}

\begin{proof}[Proof of Lemma \ref{lem:kappa_exp}]
For $x > -1$, consider the change of variable $u = 2\log{1 \over 1+x}$ such that $1-e^{-u/2} = -x$. Then $u \ge 0$ if and only if $-1 < x < 0$. We further have $e^{|u|/2}-1-\tfrac{|u|}{2}
=\bigl(\log(1+x)-\tfrac{x}{1+x}\bigr)\mathbf{1}_{\{-1<x<0\}}
+\bigl(x-\log(1+x)\bigr)\mathbf{1}_{\{x>0\}}$. 
Based on the above transformation, we define 
\[ f(x) := \begin{cases}
    -{2 \over (1+x)x} + {2 \over x^{2}}\log(1+x),& -1 < x < 0;\\
    1, & x = 0;\\
    {2 \over x} - {2 \over x^{2}}\log(1+x), & x > 0.
\end{cases} \]
		Then for $u \neq 0$, we have $\kappa(u) = f(e^{-u/2}-1)$. It suffices to show that $f$ is non-increasing and continuous. Consider $g(x) := -x + (1+x)\log(1+x);\;x > -1$.
		Then $g(0) = g'(0) = 0$ while $g''(0) = 1$.
		As $x \to 0^{-}$, we have 
        $\lim_{x \to 0^{-}}\left\{ f(x) = {g(x) \over (1+x)x^{2}/2} \right\} = g''(0) = 1$.
		As $x \to 0^{+}$, we have $\lim_{x \to 0^{+}}\left\{ f(x) = {x - \log(1+x) \over x^{2}/2} \right\} = 1$.
        This proves that $f$ is a continuous function on $(-1,+\infty)$.
        For $-1 < x < 0$, we have $f'(x) = 2{\rd \over \rd x}\left\{ {g(x) \over (1+x)x^{2}} \right\} =  {2 \over (1+x)^{2}x^{4}}\left\{ g'(x)(1+x)x^{2} - g(x)(2x + 3x^{2}) \right\} = {2 \over (1+x)^{2}x^{3}}\left\{ 2x+3x^{2}-2(1+x)^{2}\log(1+x) \right\}$.
		For $-1 < x \le 0$, consider $h(x) := -2x-3x^{2} + 2(1+x)^{2}\log(1+x)$.
		Then $h'(x) = 4g(x) \ge 0;\; -1 < x \le 0$.
		We have $h(x) \le h(0) = 0$ for $-1 < x \le 0$, and hence $f'(x) = -{2h(x) \over (1+x)^{2}x^{3}} \le 0;\; -1 < x < 0$.
		This proves that $f$ is non-increasing on $(-1,0)$.
		Moreover, $ \lim_{x\to 0^{-}}f'(x) = -2 \lim_{x\to 0^{-}}\left\{ {1 \over (1+x)^{2}} \cdot {h(x) \over x^{3}} \right\}$.
		By $h(0) = 0$, $h'(0) = 4g(0) = 0$, $h''(0) = 4g'(0) = 0$, and $h'''(0) = 4g''(0) = 4$, we have $\lim_{x \to 0}h(x)/x^{3} = h'''(0)/6 = 2/3 $. Then $\lim_{x \to 0^{-}}f'(x) = -4/3 < 0$.
		For $x > 0$, we have 
        $f'(x) = 2{\rd  \over \rd x}\left\{ {x - \log(1+x) \over x^{2}} \right\} = {2 \over x^{4}}\left\{ \left( 1 - {1 \over 1+x} \right)x^{2} - 2x[x - \log(1+x)] \right\} =  {2 \over (1+x)x^{3}}\left\{ -x^{2} - 2x + 2(1+x)\log(1+x) \right\}$.
		For $x \ge 0$, consider $l(x) := x^{2} + 2x - 2(1+x)\log(1+x)$.
		Then $l'(x) = 2[x - \log(1+x)] \ge 0;\;\forall x \ge 0$.
		We have $l(x) \ge l(0) = 0$ for $x \ge 0$, and hence $f'(x) = -{2l(x) \over (1+x)x^{3}} \le 0;\;x > 0$.
		This proves that $f$ is non-increasing on $(0,+\infty)$. Moreover, $\lim_{x \to 0^{+}}f'(x) = -2 \lim_{x \to 0^{+}}\left\{ {1 \over 1+x} \cdot {l(x) \over x^{3}} \right\}$.
		Note that for $x \ge 0$, we have $l''(x) = 2\left( 1 - {1 \over 1+x} \right);\;l'''(x) = {2 \over (1+x)^{2}}$.
		Thus we have $l(0) = l'(0) = l''(0) = 0$ and $l'''(0) = 2$. Then $\lim_{x \to 0}l(x)/x^{3} = l'''(0)/6 = 1/3$, and $\lim_{x \to 0^{+}}f'(x) = -2/3 < 0$. Combining the above, $f(\cdot)$ is non-increasing and continuous on $(-1,+\infty)$.
\end{proof}

\begin{lemma}\label{lem:number}
    For fixed $\tau > 0$, $\alpha \in [0,+\infty]$ and any $u \ge 0$, we have $\cN_{[]}(u,\cF_{\tau}^{\rm loc}(\alpha),\cL^{2}(\bbP)) \le \cN_{[]}\left({u \over 2\sqrt{2}e^{\tau/2}},\cP_{\tau}, H\right)$.
    Moreover, $\cN_{[]}\left({u \over 2\sqrt{2}e^{\tau/2}},\cP_{\tau}, H\right) \le \cN_{[]}\left( {u \over 2\sqrt{2}e^{\tau/2}}, \cP, H \right)$.
\end{lemma}

\begin{proof}[Proof of Lemma \ref{lem:number}]
		Let $\cF_{\tau} := \cF_{\tau}^{\rm loc}(+\infty)$. Then $\cN_{[]}(u,\cF_{\tau}^{\rm loc}(\alpha),\cL^{2}(\bbP)) \le \cN_{[]}(u,\cF_{\tau},\cL^{2}(\bbP))$.
		For $i = 1,2$, consider $p_{i}(y|x) \in \cP$ and let $\widetilde{p}_{i} := (1-e^{-\tau})p_{i} + e^{-\tau}p_0 \in \cP_{\tau}$. In particular, $\cF_{\tau} \ni \log(p_0/\widetilde{p}_{i}) \le \tau$, that is, $\widetilde{p}_{i} \ge e^{-\tau}p_0 $.
		Then, for some $\lambda \in [0,1]$, we have $\left| \widetilde{p}_{1}^{1/2} - \widetilde{p}_{2}^{1/2} \right| = \left| \exp \left( {1 \over 2} \log \widetilde{p}_{1} \right) - \exp\left( {1 \over 2}\log \widetilde{p}_{2} \right) \right| = {\widetilde{p}_{1}^{\lambda/2}\widetilde{p}_{2}^{(1-\lambda)/2} \over 2}\left| \log \widetilde{p}_{1} - \log \widetilde{p}_{2} \right| \ge {e^{-\tau/2} \over 2}p_0^{1/2}\left| \log \widetilde{p}_{1} - \log \widetilde{p}_{2} \right|$,
        where the second equality is due to mean value theorem. Therefore, $H^{2}(\widetilde{p}_{1},\widetilde{p}_{2}) \ge {e^{-\tau} \over 8}\bbE\left\{ \int_{\cY} p_0(y|X)\left( \log {\widetilde{p}_{1}(y|X) \over \widetilde{p}_{2}(y|X)} \right)^{2} \rd y \right\} = {1 \over 8e^{\tau}}\bbE_{p_0}\left( \log {\widetilde{p}_{1}(Y|X) \over \widetilde{p}_{2}(Y|X)} \right)^{2}$.
		This proves that $cN_{[]}(u,\cF_{\tau},\cL^{2}(\bbP)) \le \cN_{[]}\left({u \over 2\sqrt{2}e^{\tau/2}},\cP_{\tau}, H\right)$.
        Also note that $\left| \widetilde{p}_{p _{1}}^{1/2} - \widetilde{p}_{p _{2}}^{1/2} \right| = {\left| \widetilde{p}_{p _{1}} - \widetilde{p}_{p _{2}} \right| \over \widetilde{p}_{p _{1}}^{1/2} + \widetilde{p}_{p _{2}}^{1/2}} = {|p_{p _{1}} - p_{p _{2}}| \over \left( p_{p _{1}} + {p_0 \over e^{\tau} - 1} \right)^{1/2} + \left( p_{p _{2}} + {p_0 \over e^{\tau} - 1} \right)^{1/2}} \le {|p_{p _{1}} - p_{p _{2}}| \over p_{p _{1}}^{1/2} + p_{p _{2}}^{1/2}} = \left| p_{p _{1}}^{1/2} - p_{p _{2}}^{1/2} \right|$.
		This further proves that $H^{2}(\widetilde{p}_{p _{1}}, \widetilde{p}_{p _{2}}) \le H^{2}(p_{p _{1}}, p_{p _{2}})$, and hence $\cN_{[]}\left({u \over 2\sqrt{2}e^{\tau/2}},\cP_{\tau}, H\right) \le \cN_{[]}\left( {u \over 2\sqrt{2}e^{\tau/2}}, \cP, H \right)$.
	\end{proof}

\begin{lemma}\label{lem:entropy_integral}
Consider two constants $s_0, \sigma_0 \in \bbR$. If $s = s_{0}$ and $\sigma = \sigma_{0}$ satisfy the entropy integral condition \eqref{eq:entropy_integral}, then for $j \ge 0$, $s \le s_{0} + j/5$ and $\sigma = 2^{j/2}\sigma_{0}$ also satisfy the entropy integral condition.
\end{lemma}

 \begin{proof}[Proof of Lemma \ref{lem:entropy_integral}]
		Note that $s = s_{0}$ and $\sigma = \sigma_{0}$ satisfying \eqref{eq:entropy_integral} implies that
        ${1 \over \sigma_{0}}\int_{0}^{\sigma_{0}}\sqrt{\log\cN_{[]}(u,\cF,\cL^{2}(\bbP))}\rd u \le {1 \over 2^{12+5s_{0}/2}}\sqrt{n}\sigma_{0} + {1 \over \sigma_{0}}\int_{0}^{\sigma_{0}^{2}/2^{7+s_{0}/2}}\sqrt{\log\cN_{[]}(u,\cF,\cL^{2}(\bbP))}\rd u$.
		By $u\mapsto \sqrt{\log\cN_{[]}(u,\cF,\cL^{2}(\bbP))}$ is non-increasing, this implies that for any $j \ge 0$, $s = s_{0} + j/5$ and $\sigma = 2^{j/2}\sigma_{0}$, we have 
        ${1 \over \sigma}\int_{0}^{\sigma}\sqrt{\log\cN_{[]}(u,\cF,\cL^{2}(\bbP))}\rd u
			\le {1 \over \sigma_{0}}\int_{0}^{\sigma_{0}}\sqrt{\log\cN_{[]}(u,\cF,\cL^{2}(\bbP))}\rd u$. 
            Moreover, it is straightforward to show that ${1 \over \sigma_{0}}\int_{0}^{\sigma_{0}}\sqrt{\log\cN_{[]}(u,\cF,\cL^{2}(\bbP))}\rd u \le {1 \over \sigma}\int_{0}^{\sigma^{2}/2^{7+s/2}}\sqrt{\log\cN_{[]}(u,\cF,\cL^{2}(\bbP))}\rd u$. Combining these gives 
            \(
            {1 \over \sigma}\int_{0}^{\sigma}\sqrt{\log\cN_{[]}(u,\cF,\cL^{2}(\bbP))}\rd u \le {1 \over 2^{12+5s_{0}/2}}\sqrt{n}\sigma_{0} + {1 \over \sigma}\int_{0}^{\sigma^{2}/2^{7+s/2}}\sqrt{\log\cN_{[]}(u,\cF,\cL^{2}(\bbP))}\rd u
            \). 
		That is, $\int_{\sigma^{2}/2^{7+s/2}}^{\sigma}\sqrt{\log\cN_{[]}(u,\cF,\cL^{2}(\bbP))}\rd u \le {1 \over 2^{12+5s/2}}\sqrt{n}\sigma^{2}$.
		This proves that $s = s_{0} + j/5$ and $\sigma = 2^{j/2}\sigma_{0}$ also satisfy the entropy condition \eqref{eq:entropy_integral}. Finally, the left hand side of \eqref{eq:entropy_integral} is non-decreasing in $s$ while the right hand side is non-increasing in $s$. Therefore, it also holds for all $s \le s_{0} + j/5$.
	\end{proof}



\begin{lemma}\label{lem:smooth}
		Consider $L(p ),\widehat{L}_{n}(p ),L_{\tau}(p ),\widehat{L}_{\tau,n}(p )$ in \eqref{eq:truth}, \eqref{eqn:mle} and \eqref{eq:smooth}. Then  we have $L(p_0) = L_{\tau}(p_0) = \min_{p  \in p }L_{\tau}(p )$ and $\widehat{L}_{\tau,n}(p_0) = 	\widehat{L}_{n}(p_0)$.
		Moreover, for any $p  \in p $, we have
		\[ \widehat{L}_{n}(p_0) - \widehat{L}_{n}(p ) \le {1 \over 1 - e^{-\tau}}\left\{ \widehat{L}_{\tau,n}(p_0) - \widehat{L}_{\tau,n}(p ) \right\}. \]
\end{lemma}

\begin{proof}[Proof of Lemma \ref{lem:smooth}]
		It can be established that $L(p_0) = L_{\tau}(p_0)$ and $\widehat{L}_{\tau,n}(p_0) = \widehat{L}_{n}(p_0)$ from direct derivation. For ease of notation, denote $\lambda = e^{-\tau} \in (0,1)$. First, we aim to establish that $ \min_{p  \in p } L_{\tau}(p ) = L_{\tau}(p_0)$. For any $p  \in \cP $, denote $p_{\lambda,p } := (1-\lambda)p + \lambda p_0$, and we have
		\[ \begin{aligned}
			& L_{\tau}(p ) - L_{\tau}(p_0)
			= \bbE\log{p_0(Y|X) \over p_{\lambda,p }(Y|X)} = \bbE\left\{ \int_{\cY} p_0(y|X)\log{p_0(y|X) \over p_{\lambda,p }(y|X)}\rd y \right\} \ge 0,
		\end{aligned} \]
		with equality if and only if $p_0(\cdot|X) = p_{\lambda,p }(\cdot|X)$ almost everywhere on $\cY$, almost surely in $X$.
		This concludes that $\min_{p  \in p }L_{\tau}(p ) = L_{\tau}(p_0)$.
		Next, for any $p  \in \cP $, we have by convexity of $u \mapsto -\log(u)$ that $\widehat{L}_{\tau,n}(p ) = \bbE_{n}\Big\{ - \log[(1-\lambda)p(Y|X) + \lambda p_0(Y|X)] \Big\} \le \bbE_{n}\Big\{ (1-\lambda)[-\log p(Y|X)] + \lambda[-\log p_0(Y|X)] \Big\} = (1-\lambda)\widehat{L}_{n}(p ) + \lambda\widehat{L}_{n}(p_0)$.
		By $\widehat{L}_{n}(p_0) = \widehat{L}_{\tau,n}(p_0)$, we have $\widehat{L}_{\tau,n}(p ) - \widehat{L}_{\tau,n}(p_0) \le(1-\lambda)\widehat{L}_{n}(p ) + \lambda\widehat{L}_{n}(p_0) - \widehat{L}_{\tau,n}(p_0) \le (1-\lambda)[\widehat{L}_{n}(p ) - \widehat{L}_{n}(p_0)]$. 
		Rearranging terms on both sides, we also have $\widehat{L}_{n}(p_0) - \widehat{L}_{n}(p ) \le {1 \over 1-\lambda}[\widehat{L}_{\tau,n}(p_0) - \widehat{L}_{\tau,n}(p )]$.
	\end{proof}

\begin{lemma}
    \label{lem:H_smooth}
    For $\lambda \in [0,1]$, and two conditional PDFs $p_{1},p_{2}$ of $Y|X$, we have ${(1-\lambda)^{2} \over 4}H^{2}(p_{1},p_{2}) \le H^{2}\big( (1-\lambda)p_{1} + \lambda p_{2}, \lambda p_{2} \big) \le (1-\lambda)H^{2}(p_{1},p_{2})$.
\end{lemma}

\begin{proof}[Proof of Lemma \ref{lem:H_smooth}]
		Suppose $ \lambda < 1$, $p_{1},p_{2} \ge 0$ and $p_{1},p_{2}$ are not zero simultaneously. 
		The second inequality follows from 
        $$\left| \sqrt{(1-\lambda)p_{1} + \lambda p_{2}} - \sqrt{p_{2}} \right| \le \sqrt{1-\lambda}{|p_{1} - p_{2}| \over \sqrt{p_{1}} + \sqrt{p_{2}}} = \sqrt{1-\lambda}|\sqrt{p_{1}} - \sqrt{p_{2}}|.$$
		The first inequality follows from 
        $$\left| \sqrt{p_{1}} - \sqrt{p_{2}} \right| = {\sqrt{(1-\lambda)p_{1} + \lambda p_{2}} + \sqrt{p_{2}} \over (1-\lambda)\left(\sqrt{p_{1}} + \sqrt{p_{2}}\right)}\left| \sqrt{(1-\lambda)p_{1} + \lambda p_{2}} - \sqrt{p_{2}} \right|.$$
		Note that \(\sqrt{(1-\lambda)p_{1} + \lambda p_{2}} \le \max\left\{ \sqrt{p_{1}}, \sqrt{p_{2}} \right\} \le \sqrt{p_{1}} + \sqrt{p_{2}}\),
		and hence \(\sqrt{(1-\lambda)p_{1} + \lambda p_{2}} + \sqrt{p_{2}} \le 2\left( \sqrt{p_{1}} + \sqrt{p_{2}} \right)\).
		Then we have $\left| \sqrt{p_{1}} - \sqrt{p_{2}} \right| \le {2 \over 1-\lambda}\left| \sqrt{(1-\lambda)p_{1} + \lambda p_{2}} - \sqrt{p_{2}} \right|$.
		Finally, if $\lambda = 1$ or $p_{1} = p_{2} = 0$, then the equalities hold trivially.
	\end{proof}

\begin{lemma}\label{lem:H_le_KL}
	Let $p_0(y|x)$ be the true conditional density. Then for any conditional density function $p(y|x)$, $2H^2(p,p_0) \le L(p ) - L(p_0)$.
\end{lemma}
\begin{proof}[Proof of Lemma \ref{lem:H_le_KL}]
    \[ \begin{aligned}
        & L(p ) - L(p_0) 
        = \bbE\left\{ \int_{\cY} p_0(y|X)\log{p_0(y|X) \over p(y|X)}\rd y \right\} \ge -2 \bbE\left\{ \int_{\cY} p_0(y|X)\left[ \sqrt{p(y|X) \over p_0(y|X)} - 1 \right]\rd y \right\}\\ 
        &= \bbE\left\{ 2 - \int_{\cY} \sqrt{p_0(y|X) p(y|X)}\rd y \right\} = \bbE\left\{ \int_{\cY}\left( \sqrt{p(y|X)} - \sqrt{p_0(y|X)} \right)^{2}\rd y \right\} = 2H^2(p,p_0),
    \end{aligned} \]
    where the inequality is due to $\log(1+u) \le u$.
\end{proof}
	
\begin{lemma}[{\citet[Lemma 4]{yang1998asymptotic}}]\label{lem:KL_le_H}
    Let $p_0(y|x)$ be the true conditional density. If for some $+\infty > \tau > 0$, $\log{p_0(Y|X) \over p(Y|X)} \le \tau$,  then for any $p$, $L(p ) - L(p_0) \le 2(2+\tau)H^2(p,p_0)$.
\end{lemma}

\bibliographystyle{chicago}
\bibliography{refs}
\end{document}